\documentclass[11pt,letterpaper]{article}
\usepackage{setspace}
\usepackage[english]{babel}
\usepackage[nottoc]{tocbibind}
\usepackage{amsmath,amssymb}
\usepackage{amsthm}
\usepackage{forest}
\usepackage[authoryear,round]{natbib}
\usepackage{algorithm}
\usepackage{algpseudocode}
\usepackage{booktabs}
\usepackage{xcolor}
\usepackage{physics}
\usepackage[small,compact]{titlesec} 
\usepackage{multirow}
\usepackage{bm}

\newcommand{\Mohammad}[1]{\textcolor{red}{#1}}

\usetikzlibrary{fit,positioning}

\usepackage[right=1in,left=1in,top=1in,bottom=1in]{geometry}

\usepackage{subcaption}
\usepackage{mathptmx}
\usepackage{amsmath}
\usepackage{graphicx}
\usepackage[colorlinks=true, allcolors=blue]{hyperref}
\usepackage{pifont}
\usepackage[toc,page]{appendix}

\def\hy#1{{\bf \color{blue}#1}}

\newtheorem{proposition}{Proposition}
\newtheorem{remark}{Remark}

\newtheorem{defn}{Definition}

\newcommand{\squishlist}{
   \begin{list}{$\bullet$}
    { \setlength{\itemsep}{0pt} \setlength{\parsep}{1pt}
      \setlength{\topsep}{1pt} \setlength{\partopsep}{1pt}
      \setlength{\leftmargin}{1.5em} \setlength{\labelwidth}{1em}
      \setlength{\labelsep}{0.5em} } }

\newcommand{\squishlisttwo}{
   \begin{list}{$\bullet$}
    { \setlength{\itemsep}{0pt} \setlength{\parsep}{0pt}
      \setlength{\topsep}{0pt} \setlength{\partopsep}{0pt}
      \setlength{\leftmargin}{1em} \setlength{\labelwidth}{1.5em}
      \setlength{\labelsep}{0.5em} } }

\newcommand{\squishend}{
    \end{list}  }
\newcommand{\argmax}{\arg\!\max}

\begin{document}
\title{Visual Polarization Measurement Using Counterfactual Image Generation}
\author{
      Mohammad Mosaffa\thanks{We would like to thank the participants of the 2024 PhD student workshop at Cornell University and the 2024 MarkTech Conference for their feedback. We also thank Sachin Gupta, Simha Mummalaneni, Mor Naaman, Hoori Rafieian, Jesse Shapiro, Zikun Ye, and Dennis Zhang for their detailed comments, which have significantly improved the paper. Please address all correspondence to: mm3322@cornell.edu, or83@cornell.edu, and hemay@uw.edu.}\\
      Cornell University \\
      \and 
      Omid Rafieian\footnotemark[1]\\
        Cornell University\\
        \and
        Hema Yoganarasimhan\footnotemark[1]\\
        University of Washington\\
 }
\date{}
\maketitle

\begin{abstract}
Political polarization is a significant issue in American politics, influencing public discourse, policy, and consumer behavior. While studies on polarization in news media have extensively focused on verbal content, non-verbal elements, particularly visual content, have received less attention due to the complexity and high dimensionality of image data. Traditional descriptive approaches often rely on feature extraction from images, leading to biased polarization estimates due to information loss. In this paper, we introduce the Polarization Measurement using Counterfactual Image Generation (\text{PMCIG}) method, which combines economic theory with generative models and multi-modal deep learning to fully utilize the richness of image data and provide a theoretically grounded measure of polarization in visual content. Applying this framework to a decade-long dataset featuring 30 prominent politicians across 20 major news outlets, we identify significant polarization in visual content, with notable variations across outlets and politicians. At the news outlet level, we observe significant heterogeneity in visual slant. Outlets such as \textit{Daily Mail}, \textit{Fox News}, and \textit{Newsmax} tend to favor Republican politicians in their visual content, while \textit{The Washington Post}, \textit{USA Today}, and \textit{The New York Times} exhibit a slant in favor of Democratic politicians. At the politician level, our results reveal substantial variation in polarized coverage, with Donald Trump and Barack Obama among the most polarizing figures, while Joe Manchin and Susan Collins are among the least. Finally, we conduct a series of validation tests demonstrating the consistency of our proposed measures with external measures of media slant that rely on non-image-based sources.



\end{abstract}

\thispagestyle{empty}

\newpage

\section{Introduction}

Political polarization has emerged as a central issue in American politics over the past few decades. Three out of ten Americans now consider polarization one of the most significant challenges facing the country \citep{skelley_fuong_2022}. When asked to describe the current political climate, the term ``divisive'' was the most frequent response, and similar terms like ``polarized'' and ``partisan'' were among the most common responses by Americans \citep{pew2023}. As individuals become more deeply rooted in their political identities, the potential for cross-party dialogue, compromise, and effective governance diminishes. Indeed, the impact of polarization extends beyond politics; it affects public policy, social cohesion, and the overall functioning of democratic institutions. Thus, it is important to have precise and systematic measures to study polarization. 

Several studies have focused on the verbal content in news articles as a rich source of information to propose measures for media bias and polarization and examine the demand-driven motives for news outlets to create content that matches the partisan preferences of their readers \citep{gentzkow2006media, gentzkow2010drives}. The research in this domain suggests that the choice of words and framing of issues can reveal accurate information about the speaker or the writer \citep{gentzkow2019measuring}. The existence of ideological slanting in the choice of words and framing poses important questions about the non-verbal aspect of news content. The non-verbal content conveys meaning through channels other than language, such as facial expression and body language. In news articles, the non-verbal content is often communicated through visual content such as images. Editors often pay great attention to the choice of visual content, as visuals are more memorable, processed more rapidly, and elicit stronger emotional responses than text \citep{sullivan1988happy, TownsendKahn2014, blanchard2023extraction}.\footnote{A large body of work in visual marketing and eye-tracking data demonstrates the value of visual information \citep{wedel_pieters_2007, chandon2009does, wedel2023modeling}.} Visuals not only shape engagement but also play a key role in spreading misinformation \citep{matatov2022stop} and influencing voter perceptions of competence and trustworthiness in politics \citep{hoegg2011impact}. Moreover, younger generations of consumers are showing a growing preference for news content that relies less on verbal and more on visual elements. Nevertheless, only a few studies have focused on the visual content to study polarization \citep{peng2018same, boxell2021slanted, ash2021visual, caprini2023visual}, and no prior work has utilized the richness and dimensionality of visual information. 

In this paper, we bridge this gap, and develop a framework to measure polarization and slant in visual content. Figure \ref{fig:motivation} shows a small sample of images used by {\it CNN} and {\it Fox News} to portray the 2024 presidential nominees, Donald Trump and Kamala Harris. Visually, we can see a more positive portrayal of Donald Trump (Kamala Harris) by {\it Fox News} ({\it CNN}). Our goal in this paper is to propose a method to systematically quantify this form of visual slanting. In particular, we seek to answer the following questions:
\begin{enumerate}
    \item How can we quantify political polarization in visual content, particularly in the images used in news articles?
    \item What is the extent of visual slant and polarization in mainstream media outlets? 
    \item How does the polarization in visual content vary across politicians and news outlets? 
\end{enumerate}

\begin{figure}[t]
    \centering
    \includegraphics[width=0.7\linewidth]{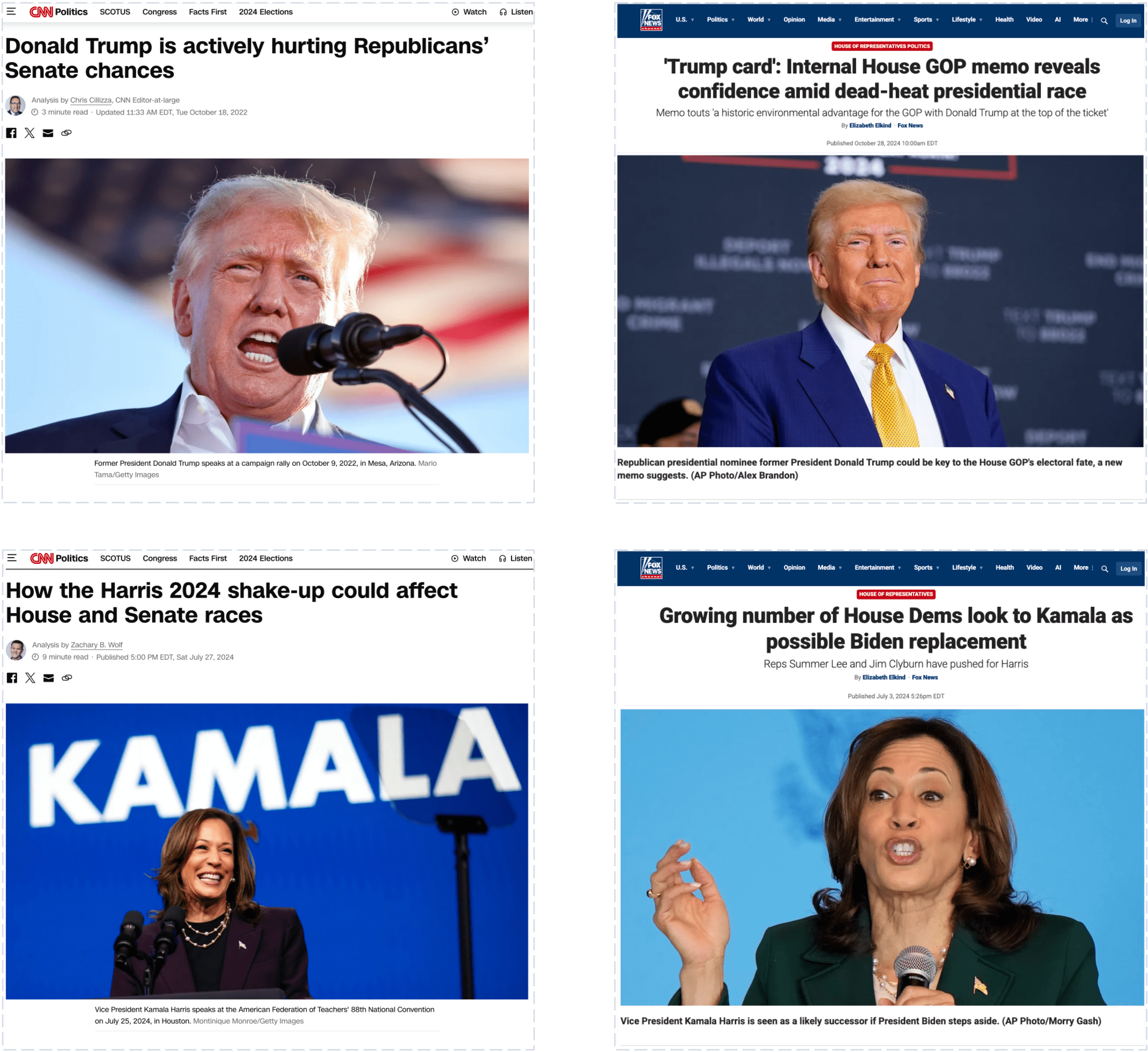}
    \caption{ \small Examples of 2024 presidential candidates' portrayal by {\it CNN} and {\it Fox News}.}
    \label{fig:motivation}
\end{figure}

There are three key challenges that we need to address to answer these questions satisfactorily. First, we need a formal definition of a metric or parameter of interest that captures polarization in visual content. To address this challenge, we turn to the very structure of the problem: the editor's choice of visual content. An ideologically slanted choice of image intuitively means that a news outlet prefers a more positive (negative) portrayal of a politician from the same (opposite) ideological side, compared to a neutral news outlet. We construct a flexible utility framework that models the editor's choice of images from a set of available options and focus on the smile as a focal feature we use to measure visual polarization given the consistent finding that a smile leads to a more positive portrayal of a politician \citep{sulflow2019power}. Using a utility framework allows us to compare the utility derived from two images that are on all aspects except the focal feature (e.g., smile) on which the visual polarization is measured. This comparison enables us to build a polarization measure centered around the focal feature of interest. For instance, consider two images of Donald Trump that differ only in whether he is smiling. For each outlet, we can define the difference in the utility from using the image with and without a smile. Intuitively, the greater the variability in this utility difference across outlets, the higher the level of political polarization in visual content. To capture this notion of variability for a given pair of news outlets, we define the difference in this utility difference as the \textit{visual polarization} measure between any pair of outlets. When one news outlet in the pair is a neutral outlet, the visual polarization measure reflects how slanted the image choice is by the other outlet, allowing us to characterize a \textit{visual slant} measure. Together, we develop a utility framework that allows us characterize both visual slant and visual polarization. 



Our second challenge stems from the richness and high dimensionality of the visual content. The common approach in the literature is to use an off-the-shelf machine learning model that extracts a certain feature (e.g., smile in an image) from the image \citep{peng2018same, boxell2021slanted}. However, this feature extraction approach relies solely on the selected feature and disregards other information in the image, which can introduce both \textit{omitted variable bias} and \textit{extraction bias} in the analysis \citep{wei2022unstructured}. We address this challenge by employing a generative approach that creates counterfactual versions of the same image with and without the feature of interest. This method enables us to retain all the information in the image and minimizes the degree to which other factors can confound our polarization measure. 

The third challenge lies in identifying our polarization parameter from the observed data, which is defined based on the utility functions of news outlets. This challenge arises because these utility functions are not directly identified from the data as we only observe the image selected by each outlet, but not the full choice set available to them. To address this issue, we leverage the variation in image choices across outlets for similar events (e.g., a specific press hearing), under the assumption that the choice set is nearly identical for all outlets covering the same event. For instance, for a similar event, both {\it Fox News} and {\it CNN} likely have access to almost identical sets of images sourced from common providers such as Getty Images or Associated Press, which allows us to assess whether one outlet derives greater utility from a positive portrayal of a politician compared to the other. We then theoretically link the identification of the polarization parameter to a news outlet prediction problem and develop a multi-modal deep learning model for this task. The model is designed to capture both the clustering structure in similar events and the subtle facial features that may reflect outlets' ideological preferences (if any). Finally, the estimates derived from this prediction task allow us to measure polarization at any desired level of granularity.

Together, we build a unified framework that combines the economic structure of the problem with generative models to generate comparable counterfactual images and measure polarization in visual content. The key advantage of our framework in comparison to the traditional regression-based approaches lies in its ability to overcome the potential confounding and mis-estimation of polarization due to information loss from feature extraction. Additionally, our generative strategy extends beyond the study of polarization and can be applied to other contexts involving visual data. Another key benefit of our framework is its ability to provide individual-level measures, enabling us to quantify the heterogeneity in visual polarization across politicians and news outlets.

We apply our framework to a comprehensive dataset comprising over 60,000 images of 30 prominent politicians from both the Republican and Democratic parties across 20 major news outlets over a 10-year span from 2011 to 2021. First, we employ a multi-modal deep learning model to predict the outlet of an image based on its visual information, the textual content of the article, and contextual data about the politician, year, and other relevant factors. We then employ Generative Adversarial Networks (GANs) to create sets of counterfactual images for all politicians with and without a smile. Using the model estimated in the initial step, we measure how the predicted probability of the image belonging to a particular outlet changes between these counterfactual images. Finally, we connect these differences to our visual slant measure and quantify the extent of polarization at both the aggregate and individual levels.

Our results suggest that both Democratic- and Republican-leaning outlets ideologically slant their visual content. Using Reuters as the base neutral news outlet, we find that compared to a neutral image of a Republican politician, a smiling image increases utility for a Republican-leaning news outlet and decreases utility for a Democratic-leaning outlet. Conversely, for Democratic politicians, a smiling image generates lower utility for Republican-leaning outlets and higher utility for Democratic-leaning outlets compared to the neutral image. These results suggest that news outlets exhibit a positive visual slant when covering politicians who share their ideological leanings and a negative visual slant when covering those with opposing views. We perform formal statistical tests and show that the distributions of visual slant are significantly different, highlighting the divide in their portrayal of politicians from either side. 

We then examine the heterogeneity of visual slant across news outlets, documenting substantial variation in how politicians are portrayed. For Democratic politicians, \textit{USA Today}, \textit{The New York Times}, and \textit{The Washington Post} exhibit the highest positive visual slant in their favor, while \textit{Daily Mail} and \textit{Fox News} display the strongest negative slant. For Republican politicians, the most positive visual slant scores appear in \textit{Daily Mail} and \textit{Newsmax}, whereas \textit{The Washington Post} and \textit{CNN} show the most negative slant. To quantify overall outlet-specific visual slant, we introduce a new measure, \textit{Conservative Visual Slant (CVS)}, which captures the degree to which an outlet's visual content favors Republican politicians while disadvantaging Democratic ones. {\it CVS} measure calculates the difference between an outlet's visual slant for Republican and Democratic politicians. Based on this measure, we identify \textit{Daily Mail}, \textit{Fox News}, and \textit{Newsmax} as outlets with the highest \textit{CVS} values, while \textit{The Washington Post}, \textit{USA Today}, and \textit{The New York Times} as the outlets the most negative \textit{CVS} values.

We also document significant variation in the degree of polarization across individual politicians. To capture this, we introduce a metric called {\it Overall Visual Polarization (OVP)}, which quantifies the standard deviation of a politician’s visual slant measures across media outlets. A higher {\it OVP} indicates greater dispersion in how different outlets visually present a politician, suggesting a more polarizing figure. On the Republican side, Donald Trump emerges as the most polarizing figure. That is, the gap in the extent of ideological slanting is remarkably large, with the Republican-leaning outlets receiving higher utility from using a positive and smiley portrayal of him compared to Democratic-leaning outlets. On the Democratic side, we find Barack Obama, Bernie Sanders, and Kamala Harris to be among the most polarizing figures who are portrayed very differently in Democratic-leaning and Republican-leaning outlets. Interestingly, Joe Manchin (D) and Susan Collins (R) rank among the least polarizing politicians, reflecting their reputations as moderates within their respective parties. We also observe low polarization levels for Liz Cheney, whose fallout with Donald Trump resulted in increased favorability among liberal outlets and decreased favorability among conservative outlets, ultimately contributing to reduced levels of visual polarization.

Lastly, we conduct a series of tests to validate the measures derived from our algorithm. First, we compare our {\it Conservative Visual Slant (CVS)} measure with existing measures of media slant from the prior literature \citep{flaxman2016filter, faris2017partisanship} using a series of correlation tests. Second, we compare the performance of our {\it CVS} measure against the visual slant metric proposed by \cite{boxell2021slanted} and demonstrate that our measure is more effective at capturing variations in external media slant indicators. Finally, we validate our {\it Overall Visual Polarization (OVP)} measure by showing a strong correlation between {\it OVP} scores and the ideological alignment of a politician’s constituency. Together, these validation tests underscore our algorithm’s ability to capture media polarization using images used by news outlets. 

In summary, our paper makes several contributions to the literature. Methodologically, we propose a framework that combines economic theory with generative models to provide robust measures of media bias and polarization in visual content. A key innovation of our framework is the use of Generative Adversarial Networks (GANs) in a way consistent with experimentation to generate counterfactual images based on a feature of interest, addressing the bias due to information loss present in social science studies that utilize image data. As such, our framework is general and applicable to all settings where researchers seek to quantify polarization on a given feature across a given set of images and outlets. Substantively, our work demonstrates the existence of ideological slanting in the visual content used by media outlets, with substantial heterogeneity observed across news outlets and politicians. 

\section{Related Literature}
\label{sec:related_lit}
Our paper relates to the study of political polarization in news media, a phenomenon well documented using text data. \citet{groseclose2005measure} quantifies media bias by examining the frequency with which different media outlets cite various think tanks, uncovering a persistent liberal bias. Similarly, \citet{gentzkow2010drives} investigate the factors driving media slant, highlighting the significant roles of consumer preferences and political affiliations by analyzing the alignment of newspaper language with political parties. \citet{jensen2012political} study the polarization of political discourse by analyzing records of Congressional speech and the Google Ngrams corpus, discovering a notable increase in discourse polarization since the late 1990s. Furthermore, \citet{gentzkow2019measuring} focus on the linguistic divide in Congressional speeches, showing how Democrats and Republicans increasingly use distinct vocabularies. While these studies predominantly focus on textual data, our work shifts the focus to polarization in visual content by examining how images in news articles contribute to this phenomenon.

More recently, research in this area has expanded from textual to visual content due to advances in computer vision techniques. \cite{peng2018same} use a two-stage approach with Azure Microsoft to extract facial expressions from 13,000 images of the 2016 election in the first stage and then analyze them through a regression model in the second stage, revealing bias toward politicians aligned with a news outlet's stance. \cite{boxell2021slanted} expand this by analyzing 70,000 images across more politicians and outlets. \cite{caprini2023visual} further extend this by generating and analyzing textual descriptions of images alongside news articles using Azure, applying \cite{gentzkow2019measuring} to show how the alignment of visual and textual bias amplifies polarization. Our study advances previous research in three significant ways. First, we enhance the traditional two-stage approach by addressing its inherent biases and issues with omitted variables, introducing a more robust methodology called Polarization Measurement Using Counterfactual Image Generation (PMCIG) that uses the rich information in the image content to measure the polarization in visual content. Second, our method delivers results that are both precise and available at a lower level of granularity, enabling us to quantify polarization at both the news outlet and politician levels and to track their evolution over the past decade. Third, we leverage a long term dataset spanning from 2011 to 2021, which allows us to capture long-term trends in political coverage that earlier studies may have missed.

From a methodological perspective, our work aligns with the growing trend in social science research that leverages computer vision techniques to analyze unstructured image data. One stream of work uses image data to generate descriptive insights on consumers and firms using novel measurement techniques \citep{dew2022letting, liu2020visual}. Another stream of work employs a two-stage approach and tries to connect image features to economic/marketing outcomes of interest by first, extracting features from images to create structured data and then applying statistical analysis to these features \citep{davenport2017analytics}. However, as discussed earlier, this approach has limitations, including potential information loss, endogeneity issues, and oversimplification due to parametric assumptions. Recently, researchers have started paying attention to this problem and proposing potential solutions. \citet{wei2022unstructured} identify biases in econometric models using machine-learned variables from unstructured data and propose solutions to improve accuracy. \citet{singh2023causal} introduce the RieszIV estimator, which incorporates high-dimensional unstructured data directly into causal analysis to manage endogeneity. \citet{xu2024unstructured} propose a debiased embedding framework that integrates representation learning with causal inference, addressing biases inherent in traditional embedding-then-inference frameworks. Additionally, \citet{luo2024using} employ GANs for Controllable Stimuli Generation (CSG), enabling precise manipulation of image attributes to isolate causal effects. Similarly, \citet{li2024product} advances AI-driven product design by integrating consumer preferences from internal data and external user-generated content into a generative framework, addressing the limitations of traditional GAN-based approaches. Building on these advancements, our PMCIG method combines GAN-based image manipulation with non-parametric deep learning models to more accurately quantify polarization in an image feature, holding all other features constant. 

\section{Setting and Data}
\label{sec:Data}
We collect publicly available data on images of politicians from media websites spanning a 10-year period for the study. Below, we describe our data collection, cleaning, and labeling strategy.

\subsection{Data Collection Strategy}
\label{ssec:data_collection}
We collect data on 30 politicians and 20 news outlets over a span of ten years, from 2011 to 2021. The set of politicians consists of those who ran for important public offices and/or held important national roles during the ten-year span of 2011--2021 and were commonly searched on Google (based on their popularity on Google searches using Google Trends data). The news outlets used in our study consist of the list of popular outlets that have been used in earlier studies on media polarization; see \citet{flaxman2016filter} for details. We refer readers to Web Appendix $\S$\ref{appssec:polnews} for a full list of politicians and outlets.


The data collection is done using the SerpAPI application \citep{serpapi}. SerpAPI facilitates efficient large-scale image scraping by leveraging Google's image search to extract relevant images and metadata. The process involves generating search queries based on specific criteria, such as targeted news outlets and date ranges, to extract images, article links, titles, and dates.\footnote{For each politician-outlet combination, we use a two-year rolling window that advances by one year at a time, resulting in 10 queries spanning a total of 10 years.} This approach was chosen because it provides a structured, efficient, and reliable interface for large-scale data collection, automating the process while ensuring compliance with Google's data access policies.

\begin{figure}[htp!]
    \centering
    \includegraphics[width=0.8\linewidth]{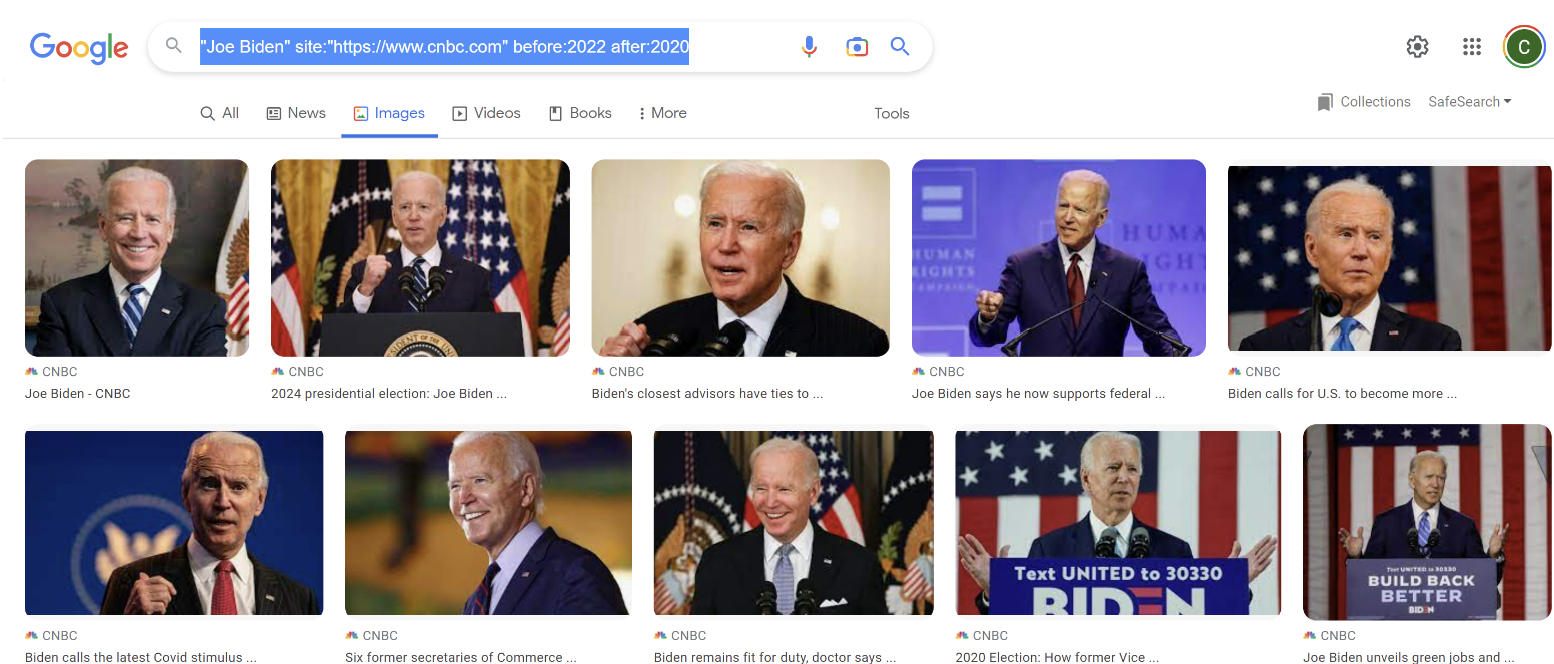}
    \caption{ \small Search results for ``Joe Biden" site: ``https://www.cnbc.com" before:2022 after:2020.}
    \label{fig:data_gathering}
\end{figure}

Figure \ref{fig:data_gathering} presents an example of a query we use for data collection, where we use ``site:" to restrict searches to specific news websites, ``before:" and ``after:" to filter results by publication date, and exact phrase searches to capture relevant content precisely. For each query, we aim to collect 80 news articles for each politician from each news outlet for the specified period.\footnote{Some queries retrieve fewer articles/images for certain politicians-outlet combinations because the outlet may not have published 80 images for that specific politician.} This information for each image in the query is compiled into a structured data frame with the following fields:
\squishlist
    \item Image: A URL pointing to the image.
    \item Alt: A short description or alternate text corresponding to the image.
    \item Href: A hyperlink where the related news article can be found.
    \item Title: The title of the news article or caption associated with the image.
    \item Query Parameter: The search query string utilized to retrieve the image and associated news, emphasizing the political figure and the source website along with a defined temporal span. 
\squishend
Overall, our data collection strategy gives us a comprehensive data set of 287,275 images for 30 politicians from 20 news outlets.

\subsection{Data Cleaning and Identifying Politicians}
\label{ssec:cleaning}
A crucial step is cleaning the data to ensure that each image contains the face of the politician referred to in the query. We face the following challenges in the data-cleaning step:
\squishlist
    \item Some images do not feature the intended politician but instead capture relevant scenes or contexts of the news without the individual's presence.
    \item Certain images, although retrieved under a specific individual's search query (e.g., Joe Biden), may inadvertently include another politician (e.g., Donald Trump), introducing cross-representation.
    \item Several images include multiple politicians, complicating the analysis of each politician's records.\footnote{We focus on single-face images to ensure that each news outlet’s selection reflects its portrayal of the intended politician, avoiding confounding effects from multiple individuals in the same image.}
\squishend
To address these challenges, we design a two-phase computer vision framework. We provide a brief overview of this framework here and refer readers to Web Appendix $\S$\ref{appssec:cleaning} for the technical details. In the first phase, we use a series of computer vision models to keep only images with one face presented in them. In the second phase, we build and deploy a face-verification tool to ensure that the face in an image belongs to the intended politician. First, we manually select 20 high-quality, single-face images for each of the 30 politicians. These selected images serve as true labels identifying the correct politician, ensuring a reliable foundation for training the face verification model. Next, we use the trained verification model for each instance in our one-face sample obtained from the first phase to verify that the predicted label for the face is the same as the intended politician. After applying the above data cleaning procedure, we are left with a set of 63,188 images, where each image shows a single face that belongs to the intended politician. Web Appendix $\S$\ref{appssec:cleaning} provides further details on our two-step model and includes a comprehensive table (Figure \ref{fig:2}) listing the number of images for each politician-outlet combination after the cleaning.

\section{Problem Definition}
\label{sec:prob_defn}

 \begin{figure}[htp!]
    \centering
    \includegraphics[width=0.75\linewidth]{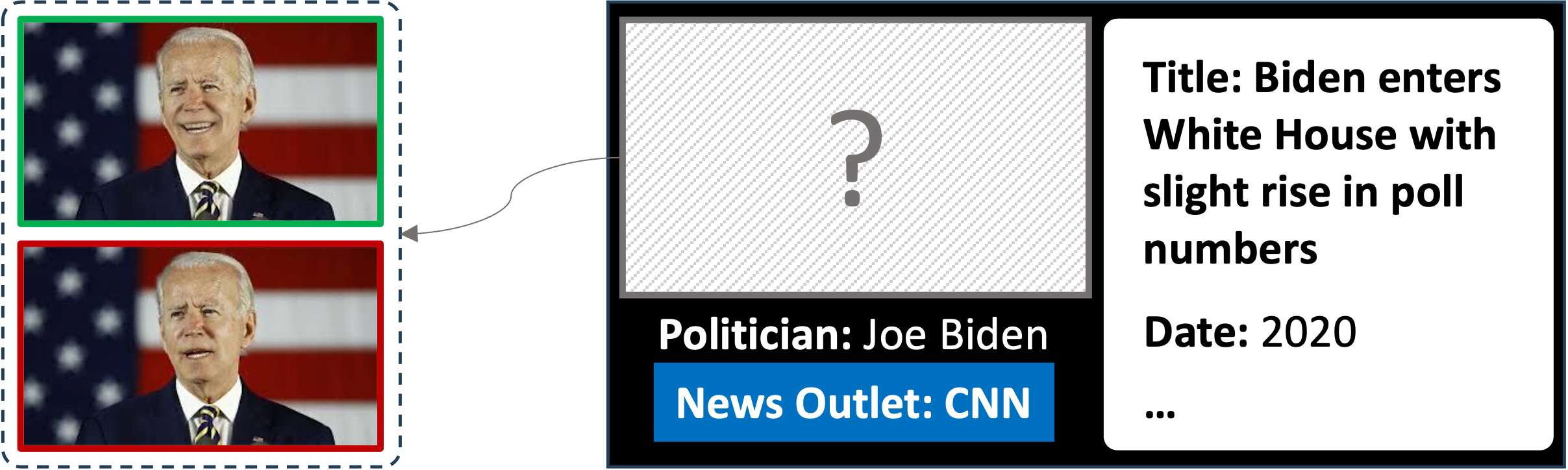}
    \caption{\small Example of an editor ({\it CNN}) selecting an image of Joe Biden from a set of two images, for a given news article.}
    \label{fig:News_Sample}
\end{figure}

Recall that our objective is to see if and how images of politicians in media outlets can be used to measure political polarization. As such, our goal here is to develop measures of {\it visual slant} and {\it visual polarization} that can be estimated from data on news articles (with images of politicians). 

Consider a dataset of news articles, $\mathcal{D}$. Each news article $i$ in this dataset is characterized by a tuple $(X_i, P_i, Y_i, Z_i)$, where $X_i$ represents the features associated with the key aspects of the article, such as its title text, topic, and publication date. $P_i$ refers to the politician who is the focal subject of the article, $Y_i$ is the news outlet that produces the article, and $Z_i$ is the image. $Z_i$ can be interpreted as detailed pixel-level information capturing all the relevant aspects of the image, such as its background, brightness, other objects present, and the facial expression of the politician $P_i$. 

Next, we define the image choice problem for a given article $i$ from the perspective of a news editor. Prior research has shown that images can have a significant impact on the extent to which readers engage with and click on news content \citep{MatiasEtAl2021}. As such, editors must carefully select images for each article, which involves choosing an image that appeals to the target audience and aligns with the editorial stance while ensuring that the visual elements complement the textual content \citep{jakesch2022belief}. Figure \ref{fig:News_Sample} shows an example news article, highlighting the decision stage where an editor picks an accompanying image for the given article.

Formally, let the editor or news outlet receive utility $U_i$ from producing article $i$ with characteristics $(X_i,P_i,Y_i,Z_i)$, which is defined as follows:
\begin{equation}\label{eq:util}
    U_i = u(Z_i,X_i,P_i,Y_i) + \xi_i,
\end{equation}
where $u(Z_i, X_i, P_i, Y_i)$ is the deterministic component and represents the editor's expected utility from producing the article, and $\xi_i$ denotes the idiosyncratic error term, which follows an i.i.d. Type 1 Extreme Value distribution.\footnote{We model the news outlet's decision-making using a utility function rather than a conventional profit function, as ideological preferences may drive them to deviate from profit-maximizing choices. For other theoretical models that micro-found agents' decision-making to examine polarization in equilibrium, see \cite{gentzkow2010drives}, \cite{iyer_yoganarasimhan_2021}, \cite{amaldoss2021media} and \cite{bondi2023privacy}.} This utility function captures all the main considerations of the editor when making image choices, such as how well the image supports the outlet's ideological stance (alignment with editorial policy, i.e., alignment of features of $Z_i$ with $Y_i$), the image's ability to attract and retain readers (reader-engagement), and the aesthetic quality and/or emotional impact of the image (visual appeal).

Intuitively, ideological slanting in images means that the outlet receives a higher (lower) utility from a more positive visual portrayal of a politician from the same (opposite) ideological side compared to a neutral outlet. The key challenge in developing a measure of ideological slant in visual content stems from the ambiguity around the definition of a ``positive'' portrayal: an image is an unstructured and high-dimensional object, and there are presumably numerous ways for the outlet to choose a more or less positive image. As such, we need to define ideological slanting for a certain image feature. In our analysis, we focus on the presence of \textit{smile} as the feature of interest because there is consensus that smile is a feature that contributes to a more positive portrayal \citep{sulflow2019power}. However, our framework is not restricted to this particular feature and can easily extend to other features or sets of features.

Let $T_i$ denote whether the subject in image $Z_i$ is smiling or not. To define the ideological slant measure for feature $T$, we start with a utility difference measure that compares the utility from two articles that are identical in all respects, except for the smile feature. Similar to the causal inference literature, we therefore consider two versions of an image $Z$: $Z(T=1)$ and $Z(T=0)$, where $Z(T=1)$ is the image with a smile and $Z(T=0)$ is the same image without a smile (with a neutral expression). For ease of exposition, let $Z^{(-T)}$ denote all the information in the image $Z$, except the smile feature $T$. This implies that $Z(T=1) \equiv (T=1, Z^{(-T)})$ and $Z(T=0) \equiv (T=0, Z^{(-T)})$, i.e., the two versions of the image are identical in all features other than smile \footnote{The notation assumes that $Z^{(-T)}$ is independent of $T$, meaning all non-smile features remain identical across conditions. This holds if the images are constructed such that only the smile varies, ensuring $Z^{(-T)} \mid T=1 \sim Z^{(-T)} \mid T=0$.}. Using this set of two images, we can define a measure of utility difference for a given news outlet $y$ and politician $p$ with respect to visual feature $T$ as follows:
\begin{equation}
\label{eq:deltau}
    \Delta^T u (p,y) = \mathbb{E}_{X}\left[u \left( Z(T=1),X,P,Y \right) - u \left(Z(T=0),X,P,Y \right) \mid P = p, Y = y\right].
\end{equation}
This utility difference helps isolate the utility increase or decrease for a news outlet that solely comes from the presence of a smile in the image of any given politician. However, this is still not a measure of ideological slant in visual content because we cannot fully attribute this utility difference to ideological preferences. For instance, there can be vertical preferences for the smile feature where all news outlets prefer a smiling image of a politician. As such, the difference in the utility difference $\Delta^T u (p,y)$ across two news outlets helps cancel out the vertical preference for the feature and the remaining difference can be linked to ideological preferences. For example,  we expect a news outlet with a higher conservative audience share (e.g., {\it Fox News}) to have a higher $\Delta^T u (p,y)$ for a conservative politician (e.g., Donald Trump) than a news outlet with lower conservative audience share (e.g., {\it CNN}) even if they both have a vertical preference for smiling images. Since this difference is defined for any two news outlets, it measures \textit{visual polarization} of the two outlets with respect to the feature of interest. However, it is important to note that this is not a measure of \textit{visual slant}, because both outlets may have slanted preferences. In what follows, we present formal definitions for both \textit{visual polarization} and \textit{visual slant}. We first present the formal definition for \textit{visual polarization} as follows:

\begin{defn}
\label{defn:visual_pol_def}
    For a given politician $p$ and any pair of news outlets $y_1$ and $y_2$, {\bf visual polarization} with respect to feature $T$ is defined as follows:
    \begin{equation}
        \rho^T (p,y_1,y_2) = \Delta^T u (p,y_1) - \Delta^T u (p,y_2).
        \label{eq:ParameterOfInterest}
    \end{equation}
\end{defn}
This definition of {\it visual polarization}, $\rho^T (p,y_1,y_2)$, is flexible and allows for detailed analysis at the level of individual politicians and news outlets. For example, by examining the above measure for a specific politician like Donald Trump across specific outlets such as {\it CNN} and {\it Fox News}, we can measure the extent to which these outlets are polarized or differentiated in their visual portrayal of Trump. Naturally, for the measure of \textit{visual polarization} to capture \textit{visual slant}, we need the news outlet $y_2$ to be neutral. Let $y_n$ denote the neutral news outlet. We can formally define \textit{visual slant} as follows:
\begin{defn}
\label{defn:visual_slant_def}
    For a given politician $p$ and any news outlet $y_1$, {\bf visual slant} with respect to feature $T$ is denoted by $\rho_s^T$ and defined as follows:
    \begin{equation}
        \rho^T_s (p,y_1) = \Delta^T u (p,y_1) - \Delta^T u (p,y_n),
        \label{eq:ParameterOfInterest_slant}
    \end{equation}
    where $y_n$ is the neutral news outlet.
\end{defn}
To further clarify the difference between \textit{visual polarization} and \textit{visual slant}, we return to the example with the portrayal of Donald Trump in {\it CNN} and {\it Fox News}, but add a neutral news outlet like Reuters. Suppose that the \textit{visual polarization} between {\it Fox News} and {\it CNN} for Donald Trump is equal to one, i.e., $\rho^T (\text{Donald Trump},\text{Fox News},\text{CNN})=1$. This measure is a composite of both Fox News' preference for a positive portrayal of Trump and {\it CNN}'s preference for a negative portrayal of him. Our \textit{visual slant} measure helps decompose the \textit{visual polarization} measure. For example, we may find that there is a positive \textit{visual slant} for Donald Trump at {\it Fox News} such that $\rho^T_s (\text{Donald Trump},\text{Fox News})=0.4$, but a negative \textit{visual slant} for him at {\it CNN} such that $\rho^T_s (\text{Donald Trump},\text{CNN})=-0.6$. It is easy to verify the following relationship between the two definitions:
\begin{equation}
    \rho^T (p,y_1,y_2) = \rho^T_s (p,y_1) - \rho^T_s (p,y_2).
\end{equation}
Lastly, a notable feature of our measures is that they are defined at a high degree of granularity, which allows us to capture the varying degrees of ideological slant in visual content specific to each outlet or each politician. This specification allows us to consider different types of aggregation:
\squishlist
\item First, by aggregating the {\it visual slant} measure over Democratic (Republican) politicians for a given outlet, we can obtain insights into how a given outlet portrays liberal (conservative) politicians. This analysis can tell us how conservative outlets like {\it Fox News} and liberal outlets like {\it CNN} portray the two groups of politicians differently. 
\item Second, by aggregating a given politician's {\it visual slant} measure across all the outlets and then comparing these measures across politicians, we can derive insights into how different politicians are portrayed in the media. For instance, this can help us understand questions such as whether the portrayal of Donald Trump in the media is more (or less) polarizing than that of Joe Biden.
\squishend


\section{Standard Reduced-Form Approach and its Limitations}
\label{sec:reduced_form}
In this section, we discuss the standard reduced-form approach used to measure polarization in visual content. In $\S$\ref{ssec:Two_Step_Model}, we describe the two-step reduced-form approach and connect the estimated parameter under this approach to the measure of polarization in visual content as defined in Equation \eqref{eq:ParameterOfInterest}). Next, in $\S$\ref{ssec:drawbacks}, we provide a theoretical and conceptual discussion of the drawbacks of this approach. 

\subsection{Two-step Model}
\label{ssec:Two_Step_Model}

Recall that our goal is to measure {\it visual polarization}, \( \rho^T (p,y_1,y_2) \), with respect to a focal feature $T$, such as whether the politician in the image is smiling or not. However, this is inherently challenging when dealing with unstructured image data, because unlike the standard causal inference literature, where treatment is observed \citep{imbens_rubin_2015}, here we do not directly observe $T_i$, i.e., the presence or absence of the treatment (smile in this case) in a given article $i$. As a result, a common approach in social science settings is to use a {\it Two-step Approach}. This approach addresses the high dimensionality and complexity of unstructured data by first extracting meaningful features and then using these features in a structured econometric model. See \citet{davenport2017analytics} for a general discussion of this approach in the broader social sciences literature, and \citet{boxell2021slanted} and \citet{peng2018same} for applications of this approach to the context of image polarization in media.\footnote{Beyond the political polarization context, this two-step approach is now commonly used in marketing research involving images in other settings as well. For instance, unstructured data such as video and audio streams (\( Z \)) are analyzed to extract features like facial expressions (\( T \)), which are then used to study their impact on user engagement (\( Y \)) \citep{lu2021larger}. In the online labor market, profile pictures (\( Z \)) are examined to identify features such as perceived race or attire (\( T \)) to assess job offer likelihood (\( Y \)), highlighting visual biases in employment opportunities \citep{troncoso2022look}. Additionally, research on social media posts about e-cigarettes (\( Z \)) extracts demographic features (\( T \)) to study tax policy compliance (\( Y \)), revealing demographic responses to legislative changes \citep{anand2024frontiers}.} Methodologically, the two-step approach can be outlined as follows:
\begin{enumerate}
    \item \textit{Feature extraction using a machine learning model:} The first step involves using an off-the-shelf machine learning model, denoted as \( f_1 \), to extract relevant feature(s) \( T \) from the unstructured data \( Z \). This model \( f_{1}: Z \rightarrow T \) is used to derive the feature(s) \( T \):
    \begin{equation}
        \hat{T} = f_{1}(Z)
        \label{eq:step1_twostep}
    \end{equation}
    \item \textit{Econometric analysis:} After extracting features \( \hat{T} \), the second step involves analyzing these features within a structured econometric model. There are two potential ways to relate \( \hat{T} \) and \( Y \):
    \begin{subequations}
    \begin{align}
        Y &= f_{2}^{\text{T-independent}}(\hat{T}, \boldsymbol{X}) + \epsilon, \;\;\; \textrm{or,} \\
        \hat{T} &= f_{2}^{\text{T-dependent}}(Y, \boldsymbol{X}) + \epsilon,
    \end{align}
    \end{subequations} 
    where $f_{2}^{\text{T-independent}}$ is the second-stage econometric model where extracted feature $\hat{T}$ is used as an independent variable, and $f_{2}^{\text{T-dependent}}$ is the second-stage model where $\hat{T}$ is used as dependent variable. Social scientists typically prefer the second form because \( \hat{T} \) is an estimated variable that could include measurement error, and therefore using it as the independent variable is not appropriate \citep{bollen2009causal}. 
\end{enumerate}
We now describe how this approach can be used to quantify polarization in visual content with respect to feature \(T\), defined as \(\rho^T (p,y_1,y_2)\) in Equation \eqref{eq:ParameterOfInterest}. Consider a simple example where news outlets from Democratic-leaning outlets like the {\it The New York Times}, {\it CNN}, and {\it BBC} (\(y_1\)) and Republican-leaning outlets like {\it Fox News}, {\it Newsmax}, and {\it Daily Mail} (\(y_2\)) choose images for a politician, say Hillary Clinton (\(p\)). In this context, the first step involves estimating the binary indicator \(\hat{T}\), which denotes whether the focal politician (Hillary Clinton) is smiling in the image (\(\hat{T} = 1\)) or not (\(\hat{T} = 0\)). This can be obtained from standard off-the-shelf emotion recognition models such as Face++ \citep{faceplusplus_emotion}, Google Vision \citep{google_ml_kit}\footnote{\cite{microsoft_azure_vision}’s Face API, a popular model previously used in studies such as \cite{boxell2021slanted, caprini2023visual}, has been unavailable for use since June 2022 due to updated responsible AI policies \citep{microsoft2023responsible}.}, or through in-house model training where we first label a sample of the images using human subjects and then train a model to predict the label given the image.

For the second step, we can use a logistic regression where we regress \(\hat{T}_i\) on characteristics \(X_i\) and \(Y_i\). For notational simplicity, consider the case where \(Y\) is a binary variable (e.g., \(Y=1\) for Democratic-leaning outlets, and \(Y=0\) for Republican-leaning outlets). We can write the log odds ratio for the resulting logistic regression as follows:

\begin{equation}\label{eq:reducedform}
\log \left( \frac{\Pr(\hat{T} = 1 \mid Y, \boldsymbol{X})}{\Pr(\hat{T} = 0 \mid Y,  \boldsymbol{X})} \right) = \alpha_{0} + \alpha_{1} \boldsymbol{X} + \beta Y
\end{equation}

If the model in the equation above is well-specified, the log odds ratio characterizes the utility difference \(\Delta^T u(p, y)\) between choosing an image with a smile versus one without a smile, as defined in Equation \eqref{eq:deltau}. In that event, we can connect the parameter \(\beta\) to the polarization measure defined in Equation \eqref{eq:ParameterOfInterest} as follows:
\begin{equation}
\begin{split}
    \rho^T (p,Y=1,Y=0) & = \Delta^T u (p,Y=1) - \Delta^T u (p,Y=0) \\
    & = (\alpha_{0} + \alpha_{1} \boldsymbol{X}  + \beta) - (\alpha_{0} + \alpha_{1} \boldsymbol{X} ) \\
    & = \beta
\end{split}
\end{equation}

Thus, the coefficient \( \beta \) in the logistic regression model serves as an empirical estimate of the polarization measurement \( \rho^T(p, Y=1, Y=0) \). If \( \beta > 0 \), it suggests that Democratic-leaning outlets (\(Y = 1\)) derive a higher utility from using a smiling image of Hillary Clinton compared to Republican-leaning outlets, implying a greater \(\Delta^T u\) for the Democratic-leaning outlets. Conversely, if \( \beta < 0 \), Republican-leaning outlets (\(Y = 0\)) have a higher utility difference compared to Democratic-leaning outlets, indicating a greater \(\Delta^T u\). Therefore, \( \rho^T(p, Y=1, Y=0) \) can be approximated by \( \beta \), and provides a quantitative measure of the ideological polarization across news outlets. 


\subsection{Drawbacks of the Two-Step Approach}
\label{ssec:drawbacks}

While the two-step model described in $\S$\ref{ssec:Two_Step_Model} is easy to interpret and apply, it relies on the assumption that the regression model is well-specified. However, this assumption can fail due to two natural reasons and lead to incorrect inference: (1) \textit{extraction bias}, and (2) \textit{omitted variable bias}. In this section, we discuss these biases and explain how they can manifest in our setting. 


{\it Extraction bias} occurs due to imperfections in the feature extraction process, or the first step of the two-step process \citep{wei2022unstructured}. That is, the true label $T$ may differ from the extracted label \(\hat{T} = f_1(Z)\) as follows:
\begin{equation}\label{eq:ext}
    T = f_1(Z) + \epsilon_1 = \hat{T}+ \epsilon_1
\end{equation}
The error in Equation \eqref{eq:ext} can be viewed as a measurement error. It is well-known that systematic dependence of this error on other relevant factors in the model can introduce biases or inconsistencies in the estimates; see Chapter 4 of \citet{wooldridge2010econometrics}. As such, we decompose this error into two parts, such that $\epsilon_1 = \epsilon_r + \epsilon_e$, where $\epsilon_r$ is the random i.i.d. noise in the measurement process, and $\epsilon_e$ is the extraction error that can be correlated with the other relevant image features contained in \(Z^{(-T)}\). For instance, a machine learning model designed to detect the facial expressions (e.g., whether a person is smiling) may also inadvertently capture brightness levels in the image, as brighter images are often associated with happier expressions. We can rewrite the relationship between $T$ and $\hat{T}$ as:
\begin{equation}\label{eq:ext_dec}
    T = \hat{T} + \epsilon_r + \epsilon_e.
\end{equation}
Then, the predicted feature $\hat{T}$ is a combination of the true feature of interest $T$ (e.g., smile) and a correlated nuisance feature $\epsilon_e$ (e.g., brightness). Thus, when we perform the second stage estimation, the estimated effect ($\beta$) captures the effect of the outlet on the biased measurement $\hat{T}$ rather than the true feature $T$. 

Next, \textit{omitted variable bias} arises from relevant factors omitted from the second step. This issue can arise even if we have the true labels $T$. To illustrate this point, consider the second-step relationship between the true label and covariates:
\begin{equation}\label{eq:ovb}
    T = f_2(Y,X) + \epsilon_2,
\end{equation}
where $\epsilon_2$ consists of unobserved variables that are not captured by the observed covariates $Y$ and $X$ through the semi-parametric function $f_2$. If this error term is independent of the covariates, we can identify function $f_2$ correctly. However, this is often a very strong assumption given the amount of information in \(Z^{(-T)}\) that is omitted from the model. For example, consider a collection of images by Donald Trump in articles by {\it CNN} and {\it Fox News}, where the smile labels are accurate, i.e., $\hat{T} = T$. Now, consider an image feature, such as the presence of the US flag. It is likely that the US flag is present more often in images from {\it Fox News}, given their higher share of nationalist viewers. At the same time, it is more likely for a politician to smile in formal settings with flags present in the background. In this scenario, estimating the parameters of Equation \eqref{eq:ovb} will lead to function $f_2$ picking up the association between the presence of the flag and the smile in the image because the presence of the flag is omitted from the model. Given the high-dimensional and unstructured nature of the images, numerous other image features can result in omitted variable bias in our estimates. To characterize the endogenous part of the error term $\epsilon_2$, we rewrite Equation \eqref{eq:ovb} as follows:
\begin{equation}\label{eq:ovb_dec}
    T = f_2(Y,X) + \epsilon_0 + \epsilon_{ov},
\end{equation}
where $\epsilon_0$ is the part that is independent of covariates, and $\epsilon_{ov}$ is the part that is correlated with $Y$ or $X$, which can lead to omitted variable bias. It is worth noting that image features that lead to extraction bias can also lead to omitted variable bias. For example, omitting image features like brightness can lead to omitted variable bias if brightness is correlated with both the actual presence of a smile in the image (not the extracted one) and the news outlet.

Together, if we want to estimate the second step of the two-step model in $\S$\ref{ssec:Two_Step_Model} by using $\hat{T}$ as the outcome, the errors in Equation \eqref{eq:ext_dec} also appear in the second step equation as follows:
\begin{equation}\label{eq:ovb_ext}
    \hat{T} = f_2(Y,X)+ \epsilon_0 + \epsilon_{ov} + \epsilon_r  + \epsilon_e,
\end{equation}
where $\epsilon_{ov}$ leads to omitted variable bias, $\epsilon_e$ leads to extraction bias due to the use of a machine learning model in the first step, and $\epsilon_r$ contributes to higher uncertainty in model estimates. In summary, the information loss resulting from extracting a single feature from an image can introduce both extraction bias and omitted variable bias. 

We present a more formal characterization of these biases and their derivations in Web Appendix $\S$\ref{appsec:drawback}. We also provide empirical evidence of how the two-step approach can be biased in estimating visual polarization in our setting in Web Appendix $\S$\ref{appsec:empericalevidence}.

\section{Our Approach: Polarization Measurement Using Counterfactual Image Generation}
\label{sec:PMCIG}
As discussed in $\S$\ref{sec:reduced_form}, the two-step approach is subject to information loss since it ignores the high-dimensional information in the image, which in turn can bias the estimated measure of polarization. Therefore, we develop a novel algorithm -- Polarization Measurement Using Counterfactual Image Generation (PMCIG) -- that recovers the polarization measure defined in $\S$\ref{sec:prob_defn} without incurring information loss. 

Recall that the {\it visual polarization} measure is defined as (see Definition \ref{defn:visual_pol_def}):
\[\rho^T(p, y_1, y_2) = \Delta^T u(p,y_1) - \Delta^T u(p,y_2),\] 
where  $\Delta^T u (p,y) = \mathbb{E}_{X}\left[u \left( Z(T=1), X,P,Y \right) - u \left(Z(T=0),X,P,Y \right) \mid P = p, Y = y\right]$. Before outlining our approach to measuring visual polarization, we first present a thought experiment of what an ideal experiment to measure $\rho^T(p, y_1, y_2)$ would look like. In order to measure $\Delta^T u (p,y)$ for a given politician-outlet pair $(p,y)$ for a feature/treatment $T$, we need to observe the difference in the editors' utilities (or choice probabilities) for two images that are exactly the same on all dimensions except $T$. Thus, an ideal experiment to measure $\rho^T(p, y_1, y_2)$ would involve presenting the editors of outlets $y_1$ and $y_2$ two images of politician $p$, that are exactly the same on all dimensions except the feature of interest, i.e., one where the politician is smiling (or the feature of interest $T$ is turned on, i.e., $Z(T=1) \equiv (T=1, Z^{(-T)})$) and one where s/he is not smiling (i.e.,  $Z(T=0) \equiv (T=0, Z^{(-T)})$). Then, the observed difference in relative choice probabilities of the two images ($Z(T=1)$ and $Z(T=0$)) across the two outlets can be linked to polarization measure $\rho^T(p, y_1, y_2)$.
 
However, our data comes from an observational setting rather than this ideal experiment. Therefore, we only observe realized combinations of $Z$ and $y$. For example, for an image $Z$ with $T=1$, we may see that it was chosen by outlet $y_1$, as in the case where Biden's smiling image was chosen by {\it CNN} in Figure \ref{fig:News_Sample}. However, we do not see what would be the probability of this image being chosen by $y_1$ if everything else about it was held the same, with only the feature/treatment of interest turned off (i.e., a non-smiling version of the same image of Biden with $T=0$). Similarly, we also do not see what would have been the likelihood of outlet $y_2$ choosing this image for the two cases: when $T=1$ and when $T=0$. Thus, we only observe one realized combination of $y$-$T$ for an image $Z$, as shown in Table \ref{tab:fact_counter}. However, to quantify/measure polarization, we need to be able to reliably estimate/model the three other counterfactual outcomes for each image $Z$ in the data. This limitation is similar in spirit to the well-established unobservability challenge in the potential outcomes framework \citep{angrist2009mostly}.

\begin{table}[htp!]
\centering
\small
\begin{tabular}{||c|c|c||}
\hline\hline
 & Treatment On: $Z(T=1)$ & Treatment off: $Z(T=0)$ \\
\hline\hline
Outlet $y_1$= {\it CNN} & $\checkmark$ & $X$ \\
Outlet $y_2$ = {\it Fox News} & $X$ & $X$ \\
\hline\hline
\end{tabular}
\caption{Factual and counterfactual outcomes for the example from Figure \ref{fig:News_Sample}, where {\it CNN} chose a smiling image of Biden.}
\label{tab:fact_counter}
\end{table}

As we can see from Table \ref{tab:fact_counter}, there are two key challenges that we need to overcome to measure {\it visual polarization}. 

\squishlist
\item \textbf{Challenge 1: Counterfactual Image Generation} \\ Our polarization measure is defined for two versions of an image that only differ in one feature ($T$). Using two versions of an image where the only difference is the presence of a smile ($T=1$ vs. $T=0$) addresses the issue of information loss in reduced-form approaches by using the rich information contained in images as opposed to mapping the image to a single feature. However, for any given image, we only have one version, where the image either has a smile or does not. Thus, for any focal image $Z$, the first challenge consists of obtaining two versions of the image that are exactly the same in all aspects except the smile.

\item \textbf{Challenge 2: Identification of the Polarization Parameter from Observed Data} \\ Second, even after we have two versions of each image ($Z(T=1)$ and $Z(T=0)$), we need to identify the polarization parameter $\rho^T(p, y_1, y_2)$ for the pair of news outlets ($y_1$, $y_2$), which is defined based on the news outlets' utility. However, identifying the utility function is not feasible because we do not observe the news outlets' choice set. Note that, in standard discrete choice models, identification of the utility function comes from observing which option an agent chooses from a set of alternatives \citep{train_2009}. However, in our setting, we only observe the chosen image. Therefore, we need a systematic approach to map the observed data to the polarization parameter without directly observing/identifying the utility functions. 

\squishend
The rest of this section is organized as follows. First, in $\S$\ref{ssec:solution}, we present the overview of our solution to these two challenges. Next, in $\S$\ref{ssec:Algorithm}, we present our unified algorithm for quantifying visual polarization.

\subsection{Overview of the Solution}
\label{ssec:solution}
In this section, we present an overview of our solution to the two challenges described earlier. We first present the generative component in $\S$\ref{sssec:cig} (to address Challenge 1). We then discuss our identification strategy in $\S$\ref{sssec:identification} (to address Challenge 2).

\subsubsection{Counterfactual Image Generation}
\label{sssec:cig}
To address the first challenge, we leverage the recent developments in generative image models that allow us to manipulate images such that we can modify one specific aspect of an image (e.g., feature $T$) while keeping everything else constant. Generative models have been employed in several recent studies on image analysis: \cite{athey2022smiles} use generative models to adjust features such as smiles in profile images to examine their effects on user preferences and economic transactions in a micro-lending platform. \cite{ludwig2024machine} utilize these models to manipulate facial features and explore how these alterations affect judicial decisions. Finally, in \cite{luo2024using}, these models create facial images to assess the impact of perceived gender traits on discrimination within online marketplaces. 

We define two operators, \(\pi^1\) and \(\pi^0\), which transform an image \(Z\) into either an image where the treatment/feature of interest is turned on \(Z(T=1)\) or one where it is turned off \(Z(T=0)\). The purpose of these operators is to create a set of control and treatment images that only differ in the feature of interest (e.g., the facial expression or smile).\footnote{We can easily extend this operation to a continuous case, where we manipulate the extent to which the feature is activated, e.g., the extent to which the politician smiles.}
\squishlist
    \item \(\pi^1\): This operator ensures that the feature or emotion \(T\) is turned on. For example, when the treatment of interest is smile, then \(\pi^1\) transforms an image of a politician into one where the politician is smiling, regardless of the initial state of \(T\) in the original image. Formally, applying \(\pi^1\) to an image \(Z\) results in:
    \begin{equation}
    \pi^1(Z) = (T = 1, Z^{(-T)}) = Z^1.
    \end{equation}

    \item \(\pi^0\): This operator ensures that the feature or emotion \(T\) is turned off. For instance, \(\pi^0\) transforms an image of a politician into one where the politician has a neutral expression, regardless of the initial state of \(T\). Formally, applying \(\pi^0\) to an image \(Z\) results in:
    \begin{equation}
    \pi^0(Z) = (T = 0, Z^{(-T)}) = Z^0.
    \end{equation}
\squishend

For a set of original images, together, these two operators give us pairs of treated and control images.

\begin{figure}[t]
    \centering
    \includegraphics[width=0.8\linewidth]{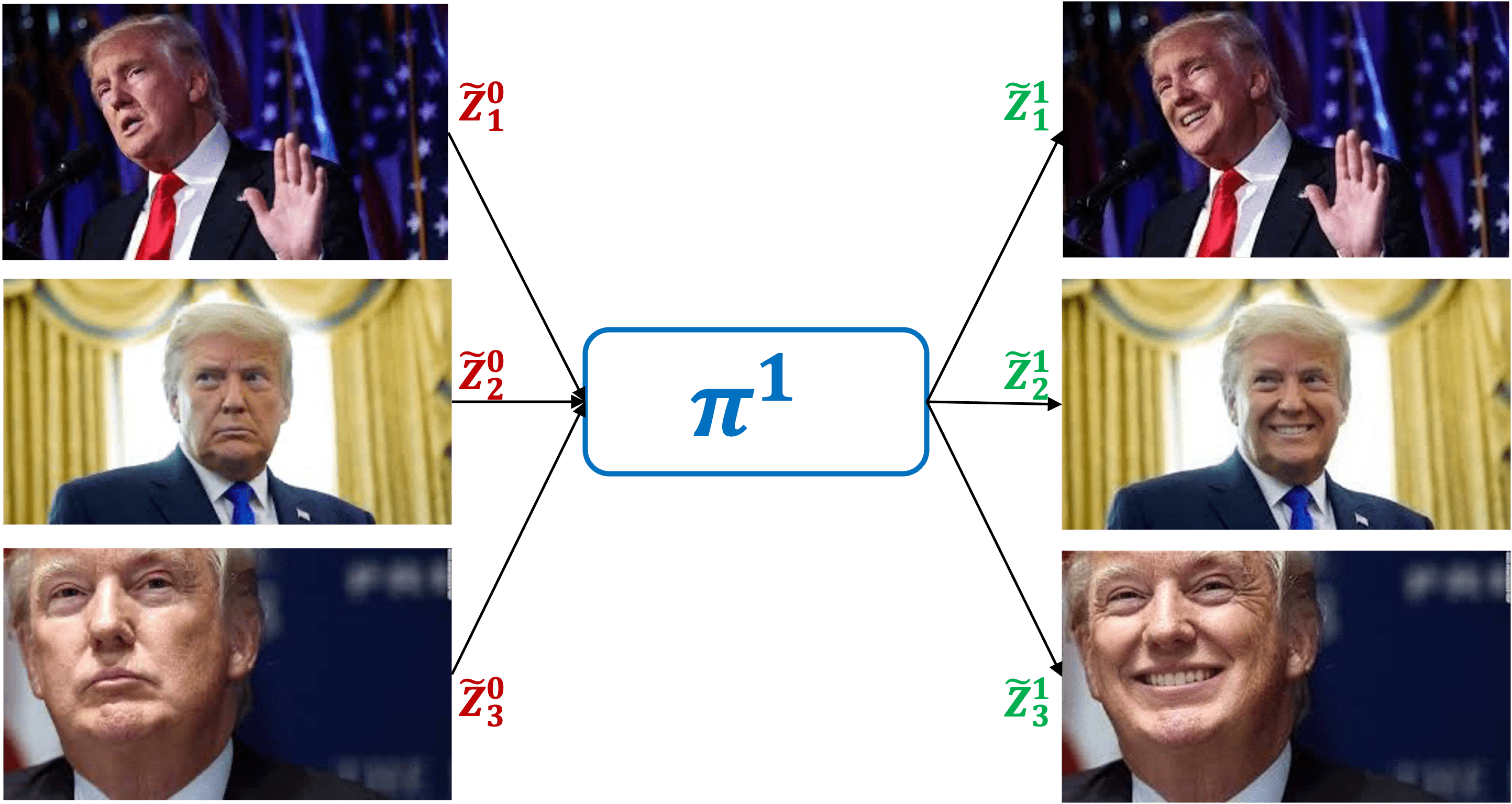}
    \caption{\small Counterfactual image generation for three neutral images of Trump.}
    \label{AddingSmile}
\end{figure}

\noindent {\bf Implementation Details:} We now discuss how we generate counterfactual images in our study. For each politician $p$, we first select three neutral images \( \tilde{Z}^{0}_{p1}, \tilde{Z}^{0}_{p2}, \tilde{Z}^{0}_{p3} \) (where $T=0$) from our dataset to ensure a balanced and unbiased representation of the politician.\footnote{Using three instead of one image ensures that the findings are not driven by the peculiarities of a specific image. In principle, it is possible to use a larger/fewer number of images for this task. We choose three because it balances noise reduction with the cost of generating counterfactual images. Focusing on three images also allows us to manually assess the quality of the generated images and avoid cases where the smile looks unnatural or unappealing.} As such, the function $\pi^0(\cdot)$ to generate the neutral image without the smile feature is an identity function. To generate the counterfactual version of each image with a smile, we employ \href{https://www.ailabtools.com/image-editor/}{AILabTools} as our \(\pi^1\) operator; this tool utilizes conditional Generative Adversarial Networks (cGANs) \citep{mirza2014conditional} and produces high quality and realistic outputs.\footnote{While many such tools are available, we use AILabTools because it allows us to precisely modify the focal images (by adding a smile) while maintaining the authenticity of the other features in the original image.} Thus, for each neutral image, we generate corresponding smiling versions \( \tilde{Z}^{1}_{p1}, \tilde{Z}^{1}_{p2}, \tilde{Z}^{1}_{p3} \), where the only difference is the added attribute (smile). Figure \ref{AddingSmile} illustrates how we apply $\pi^1$ to three neutral images of Donald Trump, which in turn gives us three corresponding neutral images of Donald Trump with a smile. We repeat this process for all the 30 politicians in our study to construct the set $\{ (\tilde{Z}^{0}_{p1}, \tilde{Z}^{0}_{p2}, \tilde{Z}^{0}_{p3} ,\tilde{Z}^{1}_{p1}, \tilde{Z}^{1}_{p2}, \tilde{Z}^{1}_{p3}) \}_p$. 

\subsubsection{Identification: Linking Polarization to the News Outlet Prediction Problem}
\label{sssec:identification}
We begin with a high-level overview of our identification strategy before formalizing our approach. The core challenge in this task lies in the fact that while we observe the image selected by a news outlet, we do not observe the full set of available choices. Our identification strategy leverages variation from similar events covered by multiple outlets (e.g., press conferences). We argue that in these cases, outlets have access to the same choice set, making the identification task analogous to a news outlet prediction problem: the selected image is the one that provides the highest utility to the outlet compared to all other similar images chosen by other outlets.

Recall that our measure is defined for a politician $p$ and any pair of news outlets $y_1$ and $y_2$. Consider a pair of images $\{ Z^{0}_{p}, Z^{1}_{p}\}$ for politician $p$, where $Z^{0}_{p}$ is the version of the image $Z$ with the feature $T$ turned off and $Z^{1}_{p}$ is the version with the feature $T$ turned on. Then, we can write the sample analogue estimator for the polarization parameter for this pair of images as follows:
\begin{equation}\label{eq:analog}
    \hat{\rho}(p,y_1,y_2) = \frac{1}{|\mathcal{X}_p|} \sum_{x \in \mathcal{X}_p} \left( \left[u(Z_p^1, x, p,y_1) - u(Z_p^0, x, p,y_1)\right] - \left[u(Z_p^1, x, p,y_2) - u(Z_p^0, x, p,y_2)\right] \right),
\end{equation}
where $\mathcal{X}_p$ is the set of all articles featuring politician $p$, so the sample analog estimator takes the average over all these articles featuring politician $p$. Clearly, if the utility function is known, we can directly estimate $\hat{\rho}(p,y_1,y_2)$ by integrating the above equation over all the articles. However, as highlighted in Challenge 2, the underlying utility function is not identified with our data because we do not observe the editors' choice sets of images. 

We propose a solution to this identification challenge that directly links the {\it visual polarization} parameter to an outlet-prediction model that can be identified from the data given without observing/identifying utility functions. We now introduce some additional notation to help with this task. In a setting with two outlets $y_1$ and $y_2$, we define $\tilde{Y}$ as a pseudo-variable that denotes the news outlet with he highest utility for producing an article with the content $x$, politician $p$, and image $z$. Formally, we can define $\tilde{Y}$ as follows:
\begin{equation}
    \tilde{Y} (z,x,p) = \argmax_{y_1,y_2} \{u(z, x, p,y_1)+\xi_1, \;\; u(z, x, p,y_2)+\xi_2 \},
\end{equation}
where $\xi_1$ and $\xi_2$ are independent and identically distributed terms that come from the Type 1 Extreme Value distribution. We can link the elements of the right-hand side (RHS) of Equation \eqref{eq:analog} to pieces identifiable from the data at hand by re-writing the sample analog estimator for {\it visual polarization} as:
\begin{equation}\label{eq:analog2}
\begin{split}
    \hat{\rho}(p,y_1,y_2) & = \frac{1}{|\mathcal{X}_p|} \sum_{x \in \mathcal{X}_p} \left( \left[u(Z_p^1, x, p,y_1) - u(Z_p^0, x, p,y_1)\right] - \left[u(Z_p^1, x, p,y_2) - u(Z_p^0, x, p,y_2)\right] \right)\\
    & = \frac{1}{|\mathcal{X}_p|} \sum_{x \in \mathcal{X}_p} \left( \left[u(Z_p^1, x, p,y_1) - u(Z_p^1, x, p,y_2) \right] - \left[u(Z_p^0, x, p,y_1) - u(Z_p^0, x, p,y_2)\right] \right)\\
    & =  \frac{1}{|\mathcal{X}_p|} \sum_{x \in \mathcal{X}_p} \left[ \log \left( \frac{e^{u(Z_p^1, x, p,y_1)}}{e^{u(Z_p^1, x, p,y_2)}}\right) -\log \left(\frac{e^{u(Z_p^0, x, p,y_1)}}{e^{u(Z_p^0, x, p,y_2)}}  \right) \right]\\
    & =  \frac{1}{|\mathcal{X}_p|} \sum_{x \in \mathcal{X}_p} \left[ \log \left( \frac{\Pr \left(\tilde{Y} = y_1 \mid  Z_p^1, x, p\right)}{\Pr \left(\tilde{Y} = y_2 \mid  Z_p^1, x, p\right)}\right) -\log \left( \frac{\Pr \left(\tilde{Y} = y_1 \mid  Z_p^0, x, p\right)}{\Pr \left(\tilde{Y} = y_2 \mid  Z_p^0, x, p\right)}\right) \right],
\end{split}
\end{equation}
In deriving Equation \eqref{eq:analog2}, we switch two elements (second line), apply $\log(\exp(\cdot))$ transformation (third line), and use a log-odds interpretation that connects the polarization parameter to a news outlet choice problem. As such, Equation \eqref{eq:analog2} allows us to define the polarization parameter as the difference between two log-odds ratios related to the pseudo-variable $\tilde{Y}$. 

We now link the identification of the polarization parameter to an outlet prediction problem (given article and image features). This link depends on the link between the pseudo-variable $\tilde{Y}$ and the actual $Y$ variable that determines the news outlet for an article. The following proposition characterizes the assumption needed for our identification:
\begin{proposition}
\label{prop:identification}
Suppose that we have data $\mathcal{D} = \{(X_i, P_i, Y_i, Z_i) \}_i$, where if an article $i$ is produced by news outlet $Y_i$, then $Y_i$ has higher utility from this article than other outlets. Then, the identification of the predicted probabilities for news outlet prediction task results in the identification of the polarization parameter. 
\end{proposition}

The proof of this proposition is simple: under the assumption that the outlet producing an article $i$ has the highest utility from it (compared to other outlets), we have $\tilde{Y} = Y$. Therefore, the identification of log-odds in Equation \eqref{eq:analog2} becomes equivalent to log-odds for variable $Y$. Let $\hat{g}(y \mid z,x,p)$ denote the probability of an article with image $z$, characteristics $x$, and politician $p$ is produced by news outlet $y$. If the condition in Proposition \ref{prop:identification} is satisfied, we can write {\it visual polarization} as follows:
\begin{equation}
    \hat{\rho}(p,y_1,y_2) =  \frac{1}{|\mathcal{X}_p|} \sum_{x \in \mathcal{X}_p} \left[ \log \left( \frac{\hat{g} \left(y_1 \mid  {Z}_p^1, x, p\right)}{\hat{g} \left(y_2 \mid  {Z}_p^1, x, p\right)}\right) -\log \left( \frac{\hat{g} \left( y_1 \mid  {Z}_p^0, x, p\right)}{\hat{g} \left( y_2 \mid  {Z}_p^0, x, p\right)}\right) \right]
\end{equation}
As such, the critical question is where in our data the condition in Proposition \ref{prop:identification} is satisfied, that is, the production of an article by an outlet implies having the highest utility from producing it compared to other outlets. Such areas in the data satisfy $\tilde{Y} = Y$ and characterize our identifying variation. For that purpose, we focus on important events that are covered by all news outlets, such as important events like press conferences by the politician. For instance, consider different images of Donald Trump during the ``Operation Warp Speed'' speech in our dataset, as shown in Figure \ref{fig:Different_Images}. Despite being captured from the same event, these images present Trump in varying ways, ranging from serious to expressive. We argue that in such events, our assumption that the producing outlet has the highest utility is more reasonable because all outlets have access to the same set of images. Therefore, the news outlet prediction from the set \( \mathcal{Y} \) can potentially mimic the true choice model over the set of images \(\mathcal{Z}\). Later in $\S$\ref{sssec:smile_change}, we present results highlighting how our prediction model utilizes this identifying variation in the data.


\begin{figure}[t]
    \centering
    \includegraphics[width=0.95\linewidth]{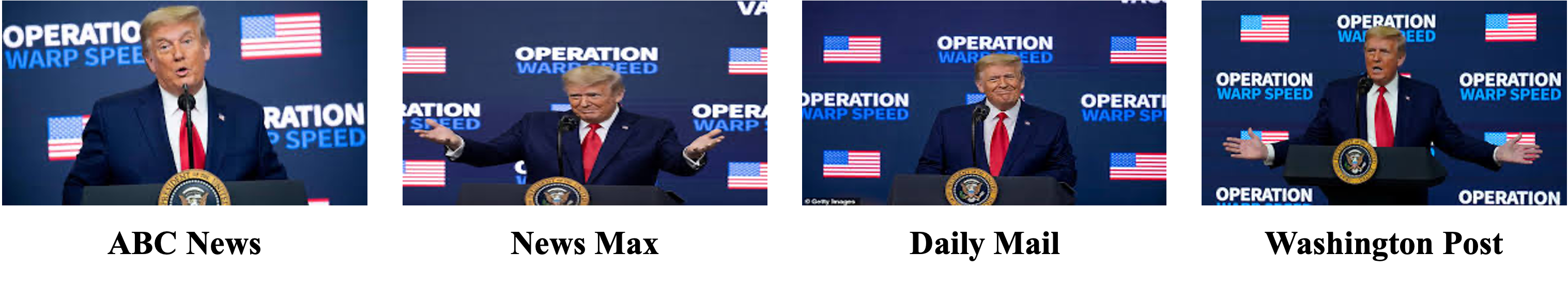}
    \caption{\small Visual coverage of Donald Trump's speech at ``Operation Warp Speed" by different news outlets.}
    \label{fig:Different_Images}
\end{figure}

A notable feature of our data is the presence of such similar patterns that form a clustering structure, where we observe extensive choice variation by outlets within a cluster of very similar events. In what follows, we discuss how we design the structure of our news outlet prediction model such that it will be able to fully exploit the variation in similar events.  

\noindent {\bf Model Architecture for the Multi-Modal News Outlet Prediction Problem:} We now discuss the estimation of our news outlet prediction model, \( g(y_i \mid Z_i, \boldsymbol{X}_i, P_i; \theta) \). Given the multi-modal nature of inputs (e.g., text, image, categorical variables), we can use any flexible semi-parametric model that can capture complex patterns in the data. However, there are important challenges that we need to address. First, we need to ensure that the model utilizes the variation in similar events to satisfy the condition in Proposition \ref{prop:identification}. Second, we need to ensure that our predictive model correctly identifies the link between the use of smile in an image and the outlet. For example, a purely loss-minimizing objective may end up estimating a model that underestimates the strength of the link between smile and the outlet by misattributing its link to features correlated with a smile. Therefore, instead of simply using a naive off-the-shelf prediction model, we impose some structure on the architecture of machine learning models we consider, as illustrated in Figure \ref{fig:ML_Model}. 

\begin{figure}[t]
    \centering
    \includegraphics[width=0.8\linewidth]{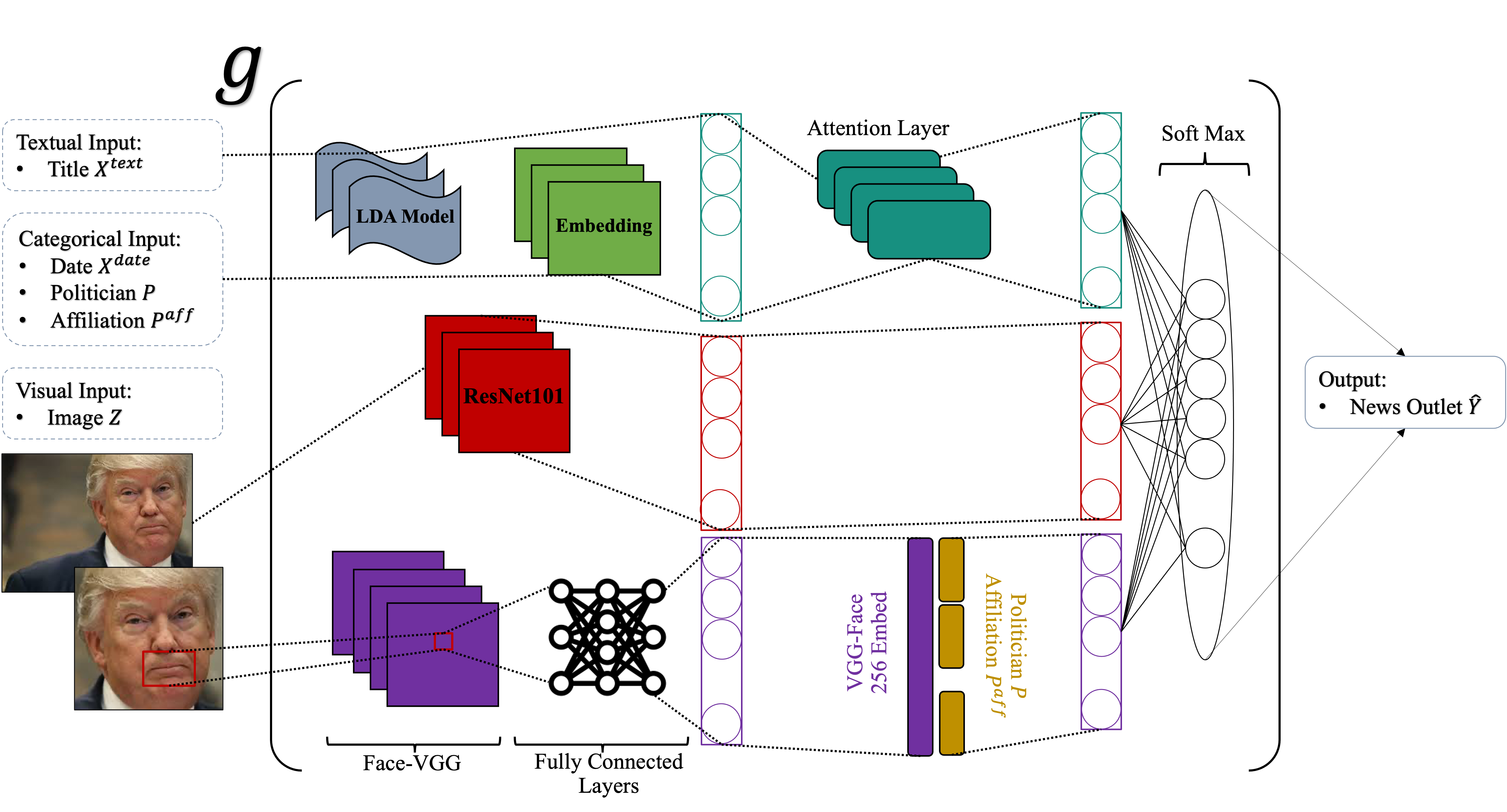}
    \caption{\small The multi-modal deep learning model for the news outlet prediction task}
    \label{fig:ML_Model}
\end{figure}
The multi-modal deep learning model has the following architecture:
\squishlist

\item First, to leverage the variation in similar events, we use the contextual information in news articles. This includes textual data (\( X^{\text{text}} \)), publication dates (\( X^{\text{date}} \)), politicians' names (\( P \)), political affiliations (\( P^\text{aff} \)), and image data (\( Z \)). Textual data is processed using Latent Dirichlet Allocation (LDA), which encodes each news title as a 40-dimensional topic vector. Second, we capture metadata such as dates, names, and affiliations as categorical variables by embedding them into dense vectors. Then, a \textit{Attention Mechanism} processes the structured data (\( X^{\text{text}} \), \( X^{\text{date}} \), \( P \), \( P^\text{aff} \)), and learns to prioritize the most relevant text and metadata features for the classification task \citep{vaswani2017attention}. 

\item Second, \textit{ResNet-101} processes the entire image \(Z\) by leveraging its capabilities from general image recognition tasks and extracts hierarchical and context-rich features from the broader visual content and scene information \citep{he2016deep}. Intuitively, given similarities in the event/news covered, these two parts of the model architecture effectively capture a given outlet's stylistic preference for certain styles of text and image features.

\item Finally, to accurately capture smiles and link them to news outlets, we use \textit{MTCNN} performs face detection and crops the face alone from the rest of the image \citep{zhang2016joint}. Then, the detected face is passed through the \textit{VGG-Face} network, which is pre-trained on facial expression data. This makes it highly effective at capturing facial attributes such as smiles \citep{parkhi2015deep}. Finally, we apply \textit{Chunk attention} to VGG-Face embeddings, and combine them with categorical data (politicians' names \( P \) and affiliations \( P^\text{aff} \)) to capture correlations between facial features and structured metadata related to the image. This step of separately capturing facial features and emotions ensures that the model accurately captures the relationship between smiles and news outlets, mitigating the risk of misattribution to correlated features.
\squishend

The final classification layer combines all modalities. The model is trained using the AdamW optimizer \citep{loshchilov2017decoupled}, which ensures efficient optimization and robust regularization. To train the model, we maximize the following entropy function over the training data:
\begin{equation}
\textbf{Entropy:} \quad \max_{\theta} \mathcal{H}(\theta) = \max_{\theta} \left( - \sum_{i=1}^{\mathcal{N}} \sum_{y \in \mathcal{Y}} g(Y_i = y \mid Z_i, X_i, P_i; \theta) \log g(Y_i = y \mid Z_i, X_i, P_i; \theta) \right) 
\end{equation}
The resulting model $\hat{g}$ is then used to estimate the polarization parameter using Equation \eqref{eq:analog2}. Please see Web Appendix  $\S$\ref{appssec:MLImplementation} for additional details on implementation and model training.

\subsection{Algorithm}
\label{ssec:Algorithm}
We now present our algorithm, \textit{Polarization Measurement Using Counterfactual Image Generation (PMCIG)} in Algorithm \ref{alg:pmcig}. We split our data $\mathcal{D}$ into training and test sets, denoted by $\mathcal{D}_{\text{train}}$ and $\mathcal{D}_{\text{test}}$, respectively. We use the training data to train the deep learning model, and use the estimated model to calculate the polarization measure on the test data. Since our main goal is polarization measurement, this form of cross-fitting helps avoid overfitting bias. The algorithm consists of three steps as described below:

\begin{algorithm}[htp!]
\caption{Polarization Measurement Using Counterfactual Image Generation (PMCIG)}
\label{alg:pmcig}
\small
\begin{algorithmic}[1]
\State \textbf{Input:} 
\Statex \hspace{0.5cm} Dataset \(\mathcal{D} = \{(Z_i, \boldsymbol{X}_i, P_i, y_i)\}_{i \in \mathcal{N}}\)
\Statex \hspace{0.5cm} Split \(\mathcal{D}\) into training set \(\mathcal{D}_{\text{train}}\) and test set \(\mathcal{D}_{\text{test}}\)
\Statex \hspace{0.5cm} Selected neutral images \(\tilde{Z}_p \notin \mathcal{D}_{\text{train}} \text{ for each } p \in P\)
\Statex \hspace{0.5cm} Number of images for each politician used for counterfactual image generation $S$
\Statex \hspace{0.5cm} News outlets \(\mathcal{Y} = \{y^k\}_{k=1}^K\), with a baseline \(y^0\) determined as the neutral outlet

\State \textbf{Output:} 
\Statex \hspace{0.5cm} Polarization Measurement \(\hat{\rho}^T (p,y^k,y^0)\) for each observation in $\mathcal{D}_{\text{test}}$

\State \textbf{Step 1: Apply Transformation \(\pi\) in $\S$\ref{sssec:cig}}
\State For each \(p \in \mathcal{P}\) and image index $s$ ($1\leq s\leq S$) for the neutral image, apply \(\pi^1\) to obtain:
\[
\tilde{Z}^1_{ps} \gets \pi^1(\tilde{Z}_{ps}), \quad \tilde{Z}^0_{ps} \gets \tilde{Z}_{ps} \text{ for each } p \in \mathcal{P}
\]

\State \textbf{Output:} \( \mathcal{Z}_p  = \{\tilde{Z}^1_{ps}, \tilde{Z}^0_{ps}\}_{s=1}^S \text{ for each } p \in \mathcal{P}\).

\State \textbf{Step 2: Train the Multi-Modal Deep Learning Model \(g\) in $\S$\ref{sssec:identification} on \(\mathcal{D}_{\text{train}}\)}
\State Given the training set \(\mathcal{D}_{\text{train}} = \{(Z_i, \boldsymbol{X}_i, P_i, y_i)\}_{i \in \mathcal{N}_{\text{train}}}\), learn parameters of model $g(\cdot;\theta) \in \mathcal{G}$:
\begin{equation*}
\begin{aligned}
&\hat{\theta} \gets \argmax_{\theta} \mathcal{H}(\theta) = \argmax_{\theta} \left( - \sum_{i=1}^{\mathcal{N}_{train}} \sum_{y \in \mathcal{Y}} g(Y_i = y \mid Z_i, X_i, P_i; \theta) \log g(Y_i = y \mid Z_i, X_i, P_i; \theta) \right) 
\end{aligned}
\end{equation*}

\State \textbf{Output:} Trained model \(\hat{g}(Y_i = y \mid Z_i, X_i, P_i; \hat{\theta})\)

\State \textbf{Step 3: Estimate Probabilities and Calculate \(\hat{\rho}^T (p,y^k,y^0)\) on \(\mathcal{D}_{\text{test}}\)}


\State For each politician $p \in \mathcal{P}$ and any news outlet $y^k$, calculate the article-level estimated Polarization Measurement for the each article \(j \in \mathcal{D}_{\text{test}}\) as \(\hat{\rho}^{T}_{j} (p,y^k,y^0)\) following:
\[
\hat{\rho}_j (p,y^k,y^0) \gets  \frac{1}{S} \sum_{s = 1}^S \left[ \log \left( \frac{\hat{g} \left(y^k \mid  \tilde{Z}_{ps}^1, \boldsymbol{X}_j, p\right)}{\hat{g} \left(y^0 \mid  \tilde{Z}_{ps}^1, \boldsymbol{X}_j, p\right)}\right) -\log \left( \frac{\hat{g} \left( y^k \mid  \tilde{Z}_{ps}^0, \boldsymbol{X}_j, p\right)}{\hat{g} \left( y^0 \mid  \tilde{Z}_{ps}^0, \boldsymbol{X}_j, p\right)}\right) \right]
\]

\State Calculate the aggregated Polarization Measurement for a given politician $p$, for each news outlet $y^k \in \mathcal{Y}$ as:
\[
\hat{\rho}^T (p,y^k,y^0) \gets \frac{1}{N_{p}^{\text{test}}} \sum_{j \in \mathcal{D}_{\text{test}} , P_j = p} \hat{\rho}^{T}_{j} (p, y^k, y^0) 
\]


\end{algorithmic}
\end{algorithm}

\squishlist
\item In the first step, we generate the counterfactual image versions for each politician. To do so, we take three neutral images from each politician $p$ from the test data and apply the function $\pi^1$ using the procedure outlined in $\S$\ref{sssec:cig}. An important consideration in choosing the neutral images is that the images are within the joint distribution of the training data. If the image is very different from those in the training data, our $\hat{g}$ estimates can be inaccurate. 

\item In the second step of our algorithm, we use the model architecture presented in $\S$\ref{sssec:identification} to estimate the model $\hat{g}$ using training data $\mathcal{D}_{\emph{train}}$, which is a random subset of 85\% of all articles in data $\mathcal{D}$. 

\item Finally, in the third step, we measure polarization on the held-out test data set $\mathcal{D}_{\emph{test}}$, which is not used in the process of model building. This step incorporates the idea of cross-fitting presented in the literature on the intersection of machine learning and econometrics, and ensures that our estimates do not exhibit overfitting bias \citep{gentzkow2019measuring, chernozhukov2018double}. In the final part of our algorithm, we aggregate over all articles and counterfactual image versions to measure polarization for each politician $p$ in each outlet $y^k$ relative to a baseline outlet $y^0$. 
\squishend
In our empirical application, we use Reuters as the baseline news outlet given its reputation to be a non-partisan news source. However, we can easily change that baseline to any other outlet to obtain polarization measures. To the extent that Reuters is the neutral point zero on the spectrum, we can interpret our polarization estimates as measures of visual slant. 

\section{Empirical Evaluation and Findings}
\label{sec:emp_eval}
We now present the results from our empirical analysis. First, in $\S$\ref{ssec:CounterfaculValidation}, we provide evidence demonstrating that the counterfactual images generated in the Step 1 of our algorithm are consistent with the original images, preserving all features except the added smile. Next, $\S$\ref{ssec:MLPerformance}, we present results on the performance of our news outlet prediction model from Step 2 of our algorithm. Finally, in $\S$\ref{ssec:MainResults}, we share the Step 3 results on estimates of political polarization in visual content, and present findings at both the news outlet and politician levels.

\subsection{Validating Counterfactual Image Consistency}
\label{ssec:CounterfaculValidation}

We briefly discuss the validation of the counterfactual images using our cGANs toolkit. It is important to ensure that adding a smile does not systematically change other contextual features of the image. Ensuring this consistency is essential for isolating the effect of the smile without introducing biases from other image characteristics, such as brightness or colorfulness (see $\S$\ref{ssec:drawbacks}). To that end, we measure a few other image characteristics (such as brightness and colorfulness) before and after applying the smile operator \( \pi^{1} \) and use the Kolmogorov-Smirnov (K-S) test to compare the distributions of these characteristics for the original and smiley images. Our tests suggest that there are no statistically significant differences between the original and smiley images for other contextual features. This confirms that the operator \( \pi^{1} \) does not introduce significant changes in other characteristics, ensuring that the counterfactual images remain consistent with the original ones. Detailed histograms and a complete analysis of the test results are shown in the Web Appendix $\S$\ref{appssec:GeneratingCounterfactual}. 

\subsection{Performance of the News Outlet Prediction Model}
\label{ssec:MLPerformance}
In this section, we present the performance of our multi-modal multi-class classification problem for news outlet prediction. As mentioned earlier, we use an 85\%-15\% split between training and test data. We consider four measures to evaluate the model's predictive performance -- (1) Accuracy, (2) Precision, (3) Recall, and (4) Weighted Cross-Entropy (WCE). Detailed explanations of these metrics are provided in Web Appendix $\S$\ref{appssec:result_pred}. 

We use the above performance measures to evaluate different versions of our deep learning model on the test data and present the results in Table \ref{tab:classification_performance}. Each row shows a more complex model that progressively adds more explanatory variables/features. The first row considers a model that only uses categorical inputs such as politician identity, party affiliation, and date. We then add textual information in the second row and add image features in the final row. Two key points emerge from Table \ref{tab:classification_performance}. First, our final outlet prediction model achieves a remarkably good performance of approximately 44\%. It is worth emphasizing that a random benchmark only achieves a 5\% accuracy (given that there are 20 news labels). Second, we notice that the image information greatly contributes to the predictive accuracy of the model: while the non-image models achieve a maximum accuracy of 16.72\%, the multi-modal model achieves an accuracy of 43.81\%. This aligns with prior research that highlights the importance of visual data in predictive tasks \citep{dzyabura2023leveraging}. 

\begin{table}[t]
\centering
\small
\begin{tabular}{lcccc} 
\hline\hline
\textbf{Features Used} & \textbf{Accuracy (\%)} & \textbf{Precision (\%)} & \textbf{Recall (\%)} & \textbf{WCE Loss} \\ \hline
Politician, Party, Date &  11.39 & 10.99 & 15.18 & 2.76 \\ 
Politician, Party, Date, Text &  16.72 & 15.01 & 17.08 & 2.66 \\ 
Politician, Party, Date, Text, Image &  43.81 & 43.65 & 41.40 & 2.29 \\ 
\hline\hline
\end{tabular}
\caption{Performance metrics for 20-class classification task using 5\% smoothing factor with different inputs on the test data.}
\label{tab:classification_performance}
\end{table}

We present additional details on the model's performance measures across outlets in Web Appendix \ref{appssec:result_pred}. In the rest of this section, we provide some additional interpretation and intuition for the results from the black-box predictive model. In $\S$\ref{sssec:smile_pred}, we discuss how our model captures the information in a smile, and in $\S$\ref{sssec:smile_change}, we examine how outlet predictions shift when we add a smile to an image. 

\subsubsection{How Does the Model Account for Smiles?}
\label{sssec:smile_pred}
Understanding how a multi-modal model learns and utilizes visual features is essential for interpreting predictions, especially in complex tasks like news outlet classification. However, this can be challenging since visual features are high-dimensional. In particular, it is important to examine whether the model is really utilizing information from facial features and emotions or whether it is simply utilizing the broader contextual information in the image for prediction (such as background color, text, etc.). 

We employ a recently developed, popular tool -- Grad-CAM -- to visualize the contribution of specific visual features to the model's predictions by helping us understand the attention patterns of different components in the model \citep{selvaraju_etal_2017}. Grad-CAM highlights the regions of an image that drive the decision-making process by examining the gradient of the model's prediction function \( g(\cdot) \) with respect to its inputs. This technique is particularly helpful in our architecture, where two different CNN branches—Face-VGG and ResNet101—process facial features and global image context, respectively. We can therefore use Grad-CAM to analyze attention patterns in both CNN branches separately. 

\begin{figure}[htp!]
    \centering
    \includegraphics[width=0.7\linewidth]{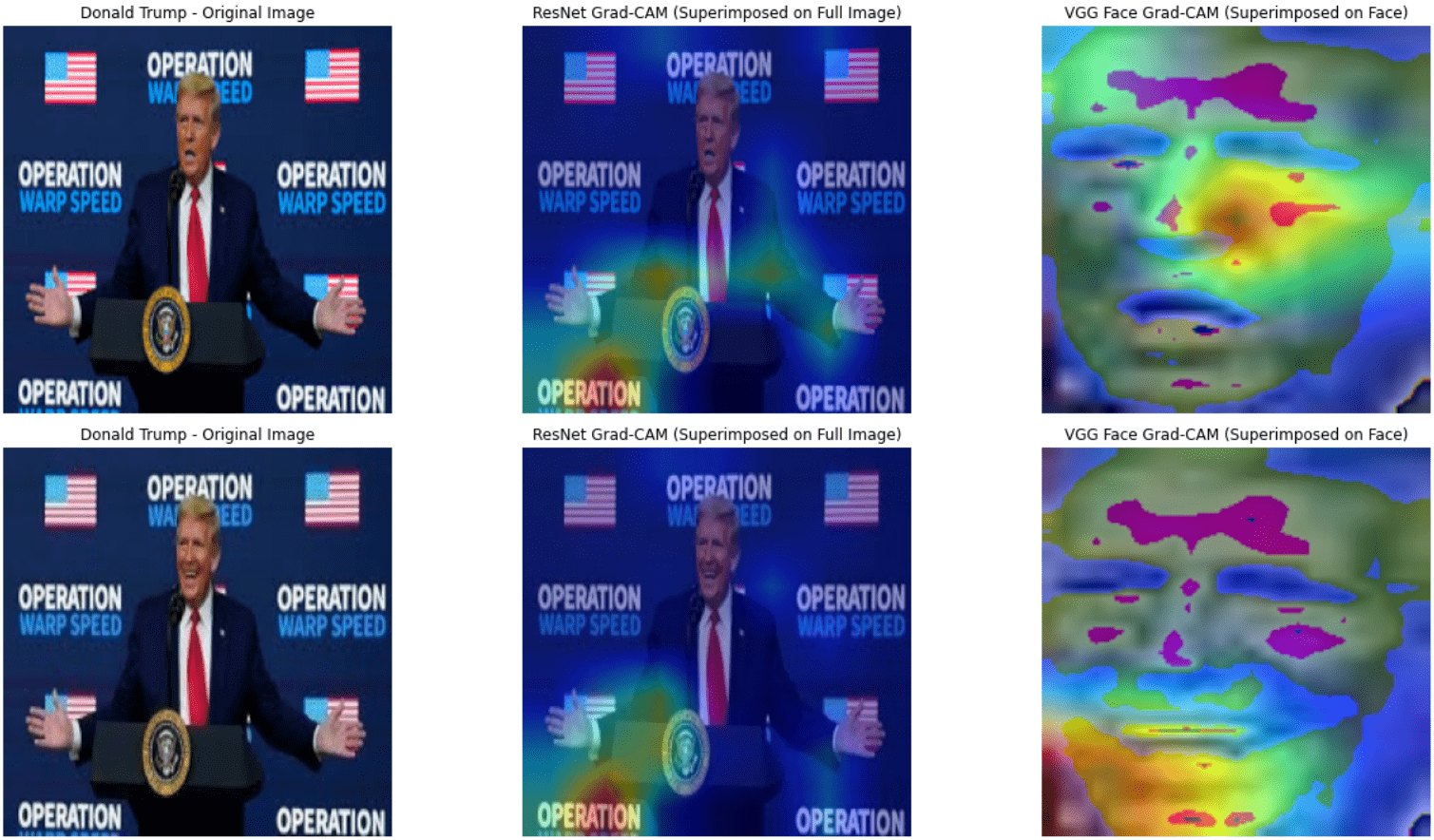}
    \caption{\small Grad-CAM visualizations for theoriginal and smile-added versions of a \textit{Washington Post} image. The ResNet101 heatmaps (right column) capture information on the global image context, while the Face-VGG heatmaps (middle column) highlight the mouth region and facial expression changes.}
    \label{fig:gradcam_example}
\end{figure}

As an illustrative example, consider the last image of Donald Trump in Figure \ref{fig:Different_Images}, from \textit{Washington Post}. Figure \ref{fig:gradcam_example} shows visualizations for two versions of this image: the original image (top row) and a smile-added version (bottom row) generated using the transformation function \(\pi^1\). The Grad-CAM visualizations illustrate two distinct patterns of model attention. First, the modified ResNet101 heatmaps predominantly focus on broader contextual features, such as the lectern, the logo, and background elements, including the event name (``Operation Warp Speed"). The gradients of \(g(\cdot)\) with respect to the global image embeddings emphasize that ResNet101 captures cues related to the overall spatial and contextual layout of the image, such as positioning and visual saliency. Second, the modified Face-VGG heatmaps display heightened sensitivity to facial features, with particular emphasis on the mouth region. For the original image, the gradient of \(g(\cdot)\) with respect to the facial embeddings shows moderate activation around the mouth, corresponding to a neutral expression. When a smile is introduced, the activation in the mouth region intensifies and extends to other facial areas associated with smile dynamics. This pattern suggests that Face-VGG is able to successfully detect smile-related features and subtle changes in facial expressions, which provides strong evidence in support of including Face-VGG in our network architecture.

\subsubsection{How Does Adding a Smile Change Predictions?}
\label{sssec:smile_change}

In the previous section, we saw that our outlet prediction model can capture the information in the smile/facial features of the focal politician. We now examine how adding a smile to a focal image changes news outlet predictions.  As discussed in $\S$\ref{sssec:identification}, our identifying assumption requires that the dataset contains similar images (e.g., from the same/similar events, as in the case of Operation Warp Speed) from a diverse set of news outlets, where the main difference between the images is the facial expression of the focal politician. As such, our goal in this section is to demonstrate how our multi-modal ML model $g$ not only learns the structure of the comparable images for similar events but also captures the editorial preferences of different news outlets for smiles versus neutral expressions. 


To do so, we turn to unsupervised learning techniques. First, we extract embeddings for all images using ResNet101. Each image is represented as a high-dimensional vector $\mathbf{e}_i \in \mathbb{R}^d$, where $d$ denotes the dimensionality of the embedding space. These embeddings capture general visual features of the images, such as texture, color, and background features, without being tied to specific news outlets. To make the embeddings interpretable and suitable for visualization, we reduce their dimensionality to two dimensions using Principal Component Analysis (PCA) and t-Distributed Stochastic Neighbor Embedding (t-SNE), followed by clustering to identify patterns. We present the details of PCA and t-SNE in our context in Web Appendix $\S$\ref{appssec:pcatsne}.

\begin{figure}[t]
  \centering
  \begin{subfigure}[t]{0.51\textwidth}
    \centering
    \includegraphics[width=\textwidth]{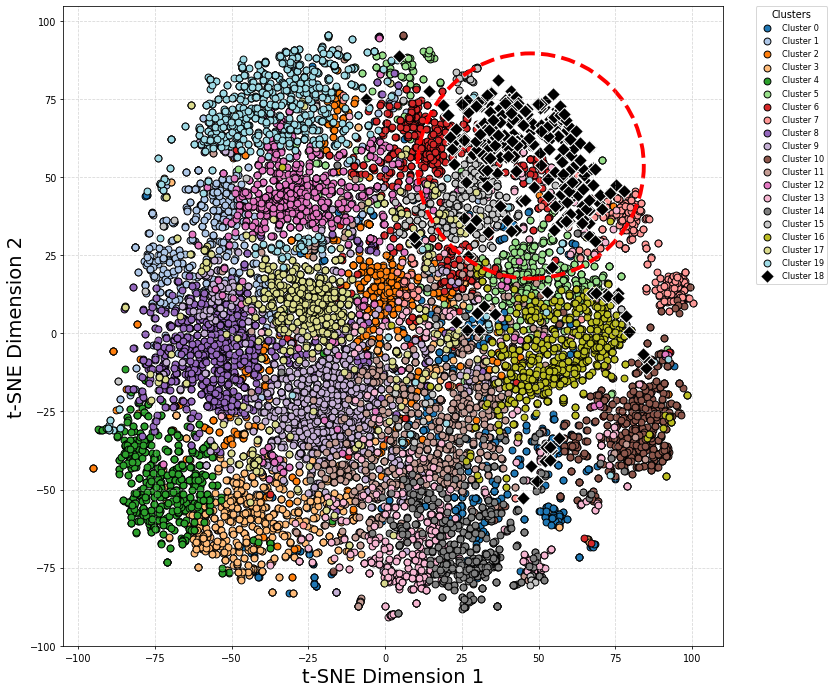}
    \caption{\small Clusters for Donald Trump: Images grouped into 20 clusters based on t-SNE projections.}
    \label{fig:trump_cluster}
  \end{subfigure}\hfill
  \begin{subfigure}[t]{0.48\textwidth}
    \centering
    \includegraphics[width=\textwidth]{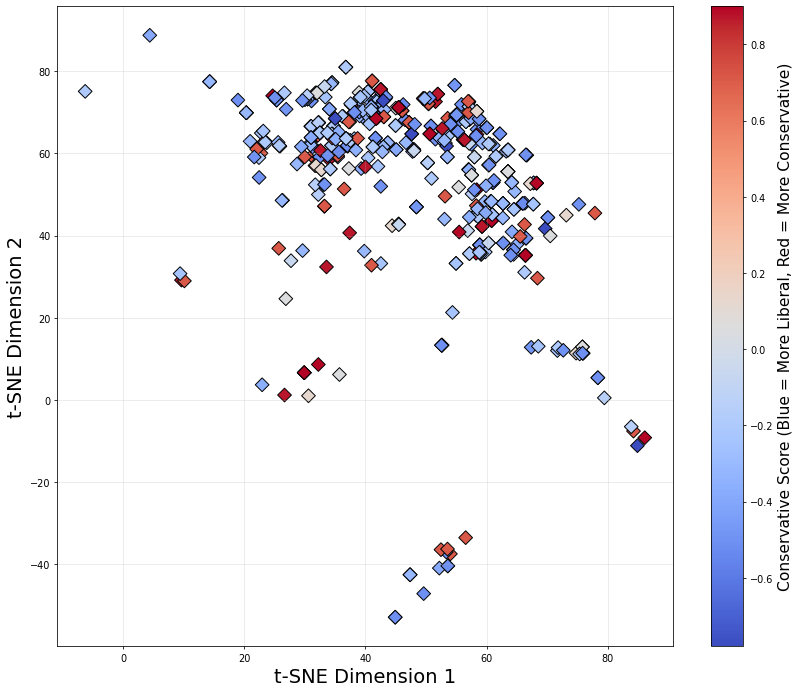}
    \caption{\small t-SNE plot for Cluster 18: Highlighting conservative scores for images in this cluster.}
    \label{fig:trump_cluster2}
  \end{subfigure}
  \caption{\small Image clustering based on t-SNE. }
  \label{fig:figure8}
\end{figure}

Figure \ref{fig:trump_cluster} shows the t-SNE projection for all images of Donald Trump, grouped into 20 clusters based on their visual similarity. Each cluster captures images with shared contextual or visual features and provides a conceptualization of ``similar events". For instance, consider Cluster 18, which is highlighted in Figure \ref{fig:trump_cluster2}. Here, each point represents an image, and the color indicates the conservative score of the news outlet that selected the image, ranging from liberal (blue) to conservative (red). The diversity in news outlets within the cluster suggests that multiple news outlets show images that are similar in structure/background (likely covering similar events). This variation is critical since it allows the outlet prediction model to capture how variation in facial features influences the likelihood of an image being chosen by a given outlet, conditional on the other image features (e.g., background, text, texture) being similar.

Figure \ref{fig:cluster_interpretation1} illustrates examples from Cluster 18, showcasing images used by different news outlets, while Figure \ref{fig:cluster_interpretation2} displays the corresponding images for the highlighted IDs. Within this cluster, consider the image corresponding to ID 196491 in our dataset by the \textit{Washington Post}, which is the same image of Donald Trump used in the previous section. In this image, Trump appears unhappy, leaning slightly towards an angry expression. As shown in Figure \ref{fig:cluster_interpretation1}, this image is initially positioned on the top right-hand side of the plot. Next, we modify the original image (ID 196491) by applying the transformation function \(\pi^1\), which adds a smile to Donald Trump while preserving the contextual elements of the image. After computing the embedding for this modified image, its coordinates in the t-SNE space are updated, as illustrated by the gold dot in the scatter plot. Interestingly, the smile-added version of the original image shifts to the left. Now, the updated position results in a shorter distance to images such as ID 36529 from the \textit{Daily Mail} and ID 72022 from \textit{Newsmax}, compared to its distance to the original image ID 196491 by the \textit{Washington Post} and ID 12736 by \textit{ABC News}. This demonstrates how the addition of a smile alters the embedding position to align with images that share similar facial expressions, even though the overall context remains unchanged.

\begin{figure}[t]
  \centering
  \begin{subfigure}[t]{0.63\textwidth}
    \centering
    \includegraphics[width=\textwidth]{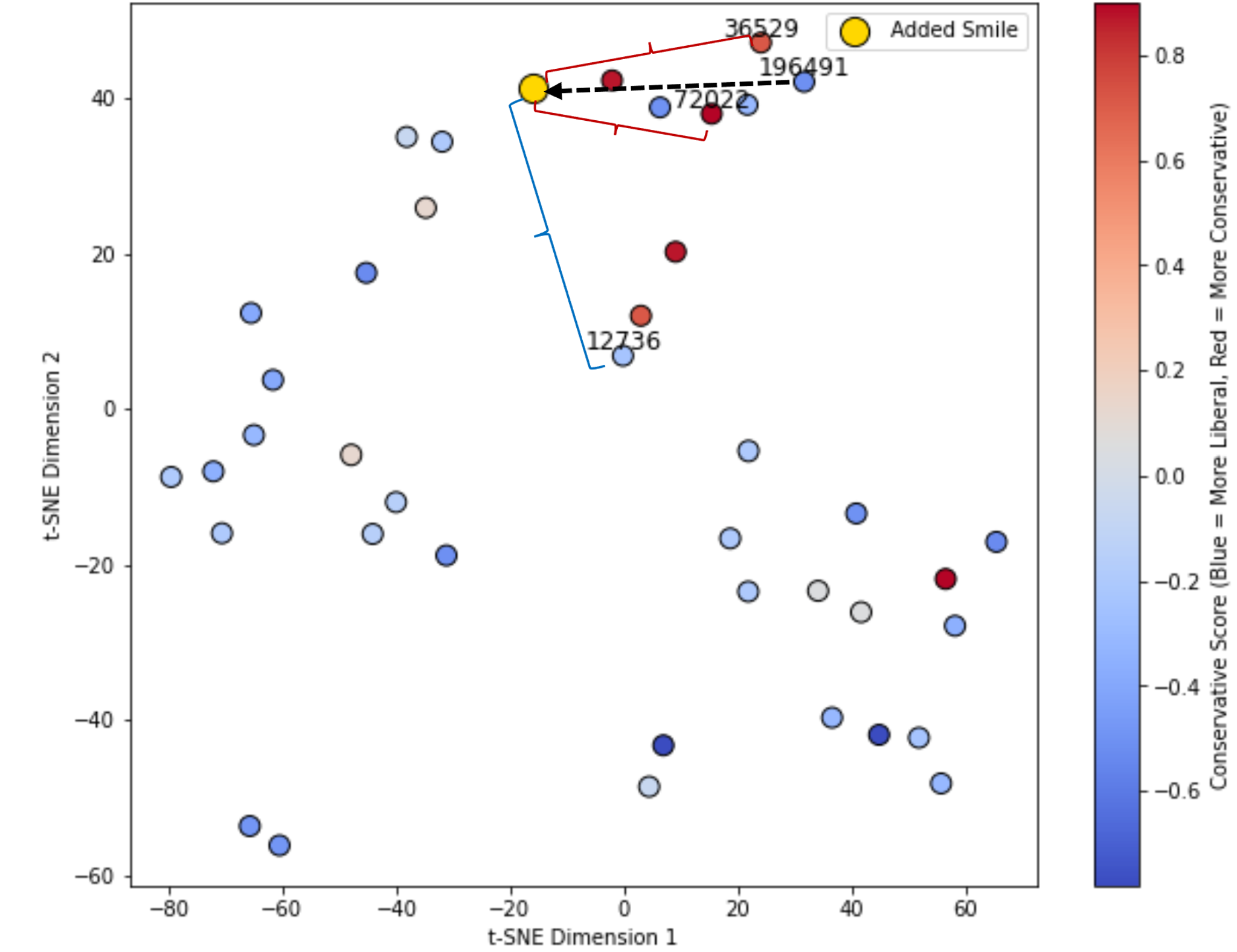}
    \caption{\small Embedding positions of the original image (ID: 196491), similar images, and the image with a smile added (highlighted in yellow).}
    \label{fig:cluster_interpretation1}
  \end{subfigure}\hfill
  \begin{subfigure}[t]{0.35\textwidth}
    \centering
    \includegraphics[width=\textwidth]{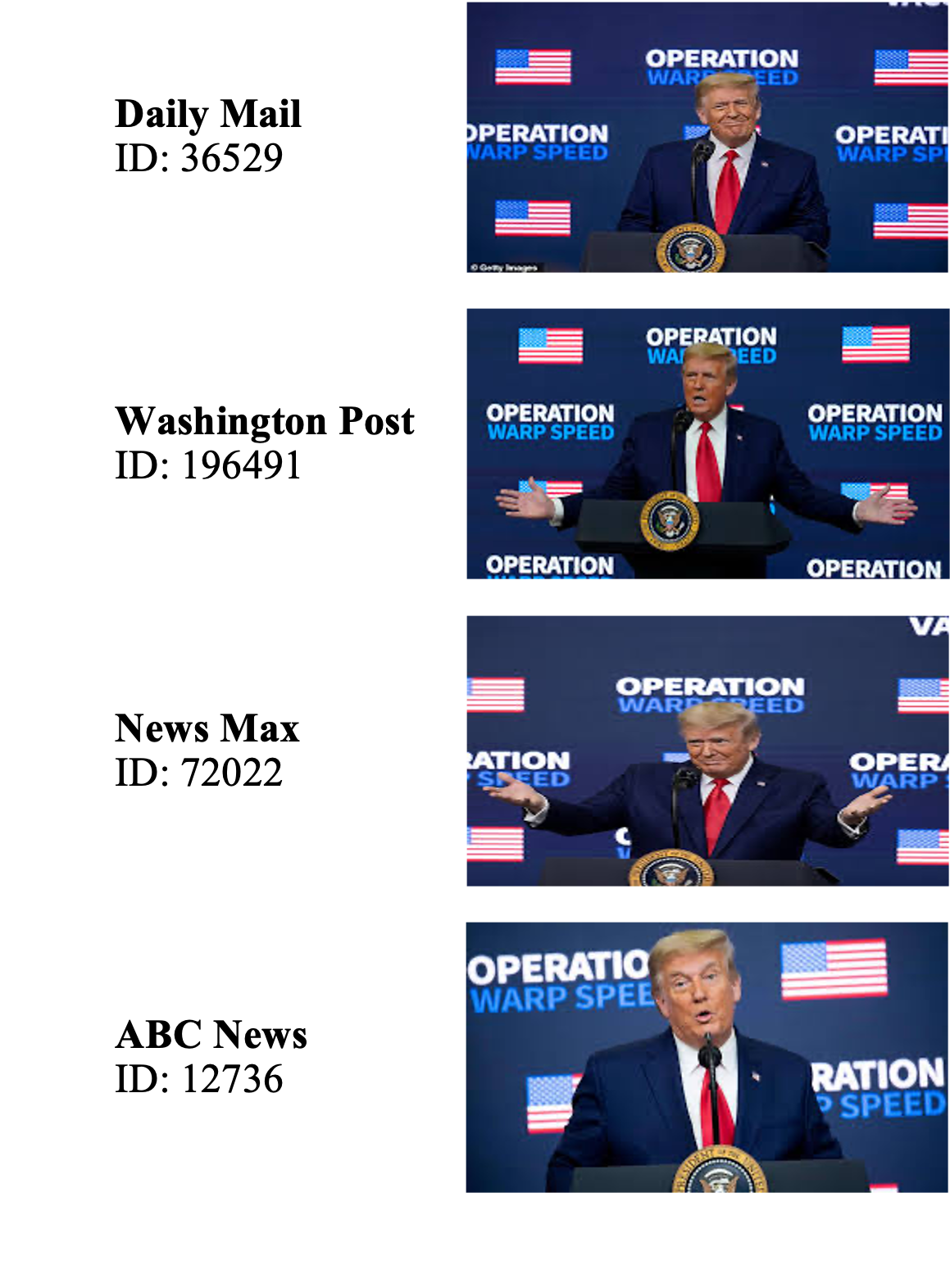}
    \caption{\small Images corresponding to highlighted points in the t-SNE plot.}
    \label{fig:cluster_interpretation2}
  \end{subfigure}
  \caption{\small t-SNE plot of the Cluster 18 of Donald Trump's images illustrating the effect of adding a smile. }
  \label{fig:cluster_interpretation}
\end{figure}

Although the t-SNE coordinates provide interpretable insights into how our model works, they do not capture the full complexity of the image features that our model $\hat{g}$ captures. Therefore, in Table \ref{tab:smile_predictions}, we present our model's outlet predictions for the fourth image from Figure 4 (the one that was originally shown in Washington Post) for two cases -- the original image and smile-added image version. We see substantial differences in the outlet predictions between the original image and the counterfactual version with the smile. Notably, we find that adding a smile substantially increases the predicted probability for \textit{Daily Mail} (from 2.89\% to 20.97\%) and \textit{Newsmax} (from 7.10\% to 16.52\%), slightly decreases the predicted probability for \textit{ABS News} (from 2.59\% to 1.70\%), and substantially decreases the predicted probability for \textit{Washington Post} (from 30.98\% to 3.32\%). 


\begin{table}[t]
\centering
\small
\begin{tabular}{lcc}
\hline\hline
\textbf{News Outlet} & \textbf{Original Image (\%)} & \textbf{Smile-Added Image (\%)} \\ \hline
{\it Daily Mail}           & 2.89                        & 20.97                           \\
{\it Newsmax}           & 7.10                        & 16.52                           \\
{\it Washington Post}      & 30.98                         & 3.32                           \\
{\it ABC News}             & 2.59                         & 1.70                            \\ 
\hline\hline
\end{tabular}
\caption{Outlet predictions for the original and smile-added versions of the Donald Trump image.}
\label{tab:smile_predictions}
\end{table}

In summary, we see that adding a smile systematically shifts the outlet prediction probabilities, consistent with the finding in the prior work in the advertising domain that shows facial expressions and emotions can shape audience engagement by capturing attention and influencing perception \citep{teixeira2012emotion}. Specifically, we see an increase (decrease) in the predicted probabilities for right-leaning (left-leaning) news outlets when we add a smile to an image of Donald Trump. Nevertheless, this is just a single instance in the data that we use to illustrate the intuition behind how our algorithm works. In the next section, we present more systematic measures to quantify visual polarization. 

\subsection{Visual Slant and Polarization}
\label{ssec:MainResults}
We now present our main results on visual slant. All results are direct applications of Algorithm \ref{alg:pmcig} in $\S$\ref{ssec:Algorithm}. Before presenting the results, we review a few important considerations in applying our algorithm. First, recall that our visual slant parameter requires a neutral outlet denoted by $y_n$ in Definition \ref{defn:visual_slant_def}. We use {\it Reuters} as the baseline outlet $y_n$ to measure visual slant, as {\it Reuters} is a largely neutral and fact-based outlet.\footnote{It is worth emphasizing that one could easily change the choice of base. Although the visual slant measure will change with a different choice, the overall visual polarization between two outlets remains unchanged.} Second, we use a train-test split, and all the polarization and slant measures are shown for the test data $\mathcal{D}_{\emph{test}}$ using the estimates obtained from the training data $\mathcal{D}_{\emph{train}}$. Third, for each politician, we start with three neutral images, generate counterfactual smiling versions of these three images, and use these six images in our polarization and slant measurement for each politician. 

This section is organized as follows. In $\S$\ref{sssec:pol_dist}, we present the overall distribution of visual slant across all outlets and politicians. We then explore the extent of heterogeneity across news outlets in $\S$\ref{sssec:pol_outlet} to see which outlets exhibit higher levels of visual slant and validate our measure using external measures of media slant. Finally, in $\S$\ref{sssec:pol_politician}, we document the heterogeneity in visual slant at the politician level.

\subsubsection{Distribution of Visual Slant} 
\label{sssec:pol_dist}
An important feature of our algorithm is its ability to produce an individual-level measure of visual slant \(\hat{\rho}^T_{i}(p_i, y^{k}_{i}, y^{Reuters}_{i}) \). As such, for each article $i$ featuring politician $p$, we can estimate 19 visual slant measures corresponding to all the outlets other than Reuters. Based on the scores from \cite{faris2017partisanship}, \cite{allsides2024}, and \cite{flaxman2016filter}, we categorize the three most Republican-leaning news outlets as: \textit{Fox News}, \textit{Newsmax}, and \textit{Daily Mail}, and the three most Democratic-leaning news outlets as: \textit{Washington Post}, \textit{CNN}, and \textit{The New York Times}. Figure \ref{fig:Hist_Overall} shows the distributions of visual slant for Democratic and Republican politicians as featured in Democratic- and Republican-leaning news outlets. If there were no visual slant or polarization, all distributions would be tightly centered around zero. However, we observe that for both Republican- and Democratic-leaning outlets, the individual-level visual slant measures deviate from zero and exhibit significant dispersion. This indicates the presence of a visual slant in these outlets. Additionally, we notice a clear divide in the distributions of Democratic- and Republican-leaning outlets for both Republican and Democratic politicians, pointing to evidence of visual polarization. 

\begin{figure}[htp!]
  \centering
  \begin{subfigure}[t]{0.5\textwidth}
    \centering
    \includegraphics[width=\textwidth]{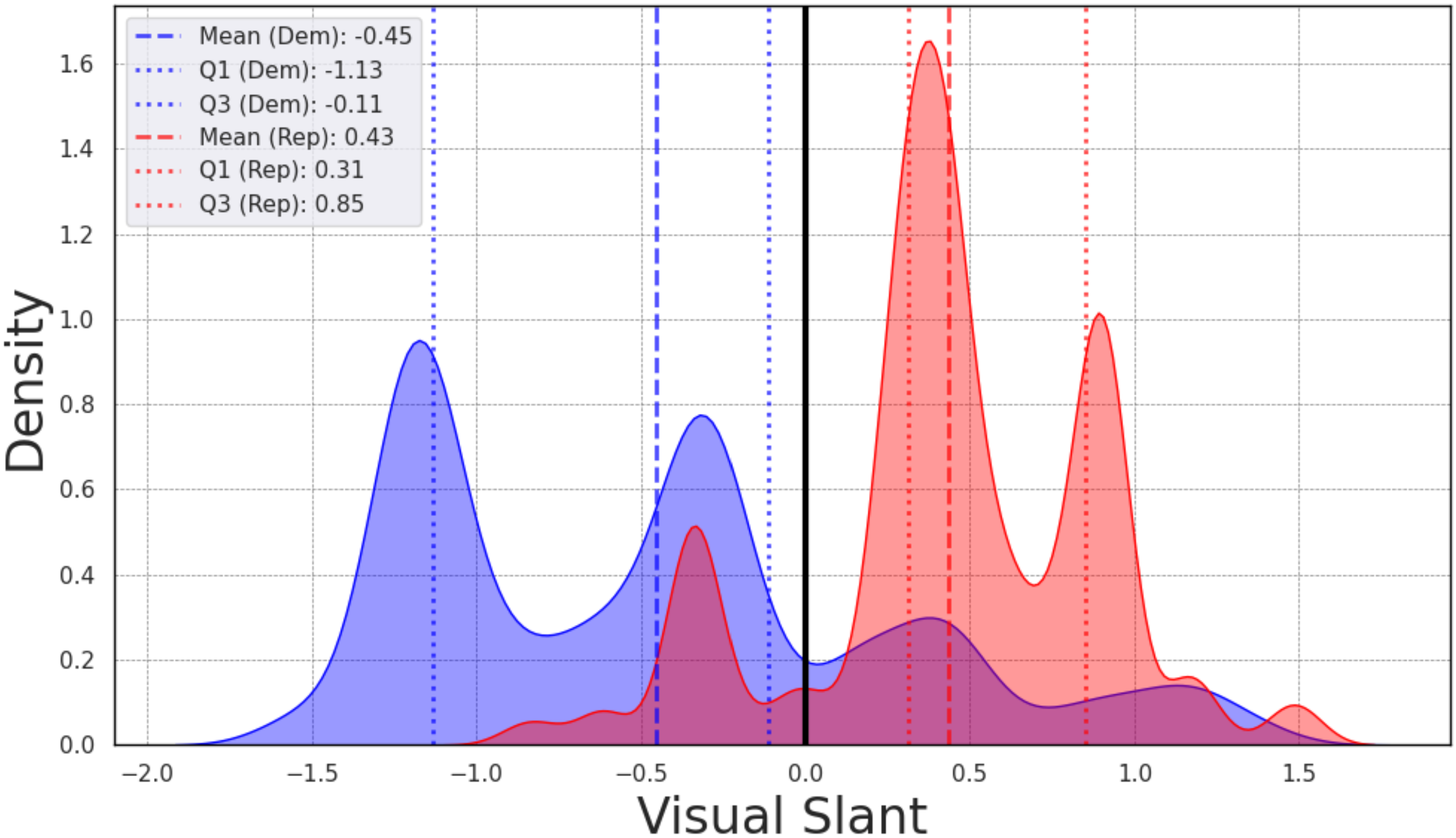}
    \caption{\small Republican politicians.}
    \label{fig:Hist_Overall1}
  \end{subfigure}\hfill
  \begin{subfigure}[t]{0.5\textwidth}
    \centering
    \includegraphics[width=\textwidth]{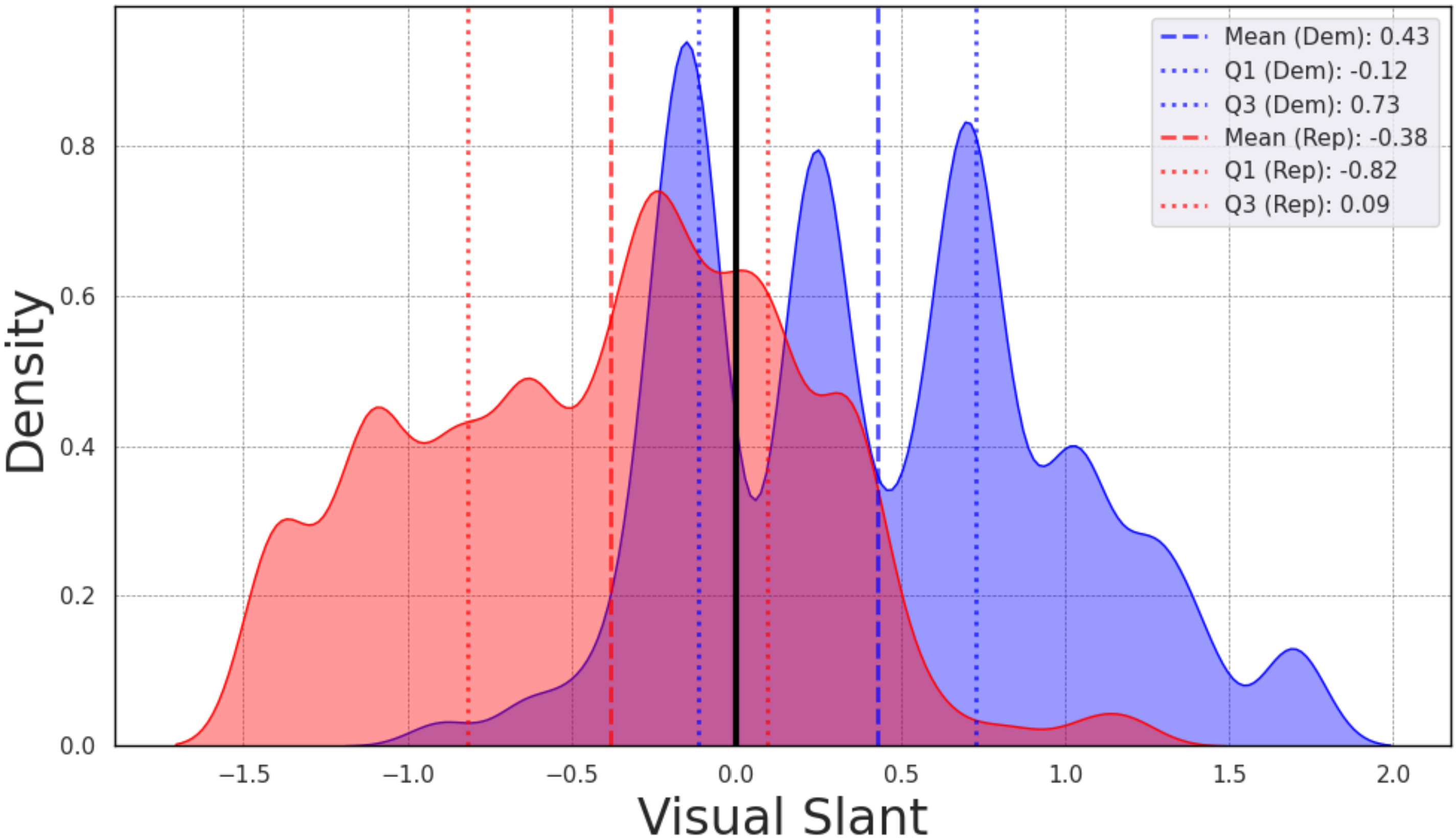}
    \caption{\small Democratic politicians.}
    \label{fig:Hist_Overall2}
  \end{subfigure}
  \caption{\small Histograms of {\it visual slant} for Democratic and Republican Politicians in Democratic- and Republican-leaning News Outlets.}
  \label{fig:Hist_Overall}
\end{figure}

Next, we examine if there is a systematic difference between Republican- and Democratic-learning outlets by examining their averages in each panel. We denote the set of Republican and Democratic politicians by $\mathcal{P}_R$ and $\mathcal{P}_D$, respectively. Similarly, we denote the sets of Republican- and Democratic-leaning outlets defined above are denoted by $\mathcal{Y}_R$ and $\mathcal{Y}_D$, respectively. For Republican politicians, in Figure \ref{fig:Hist_Overall1}, the distributions of \( \hat{\rho}^T_{i} \) show distinct differences between left- and Republican-leaning news outlets. The visual slant measurement for left-leaning news outlets is \textit{-0.45}, implying that, on average, smiling images of Republican politicians decrease the utility of a news outlet classified as Democratic-leaning compared to Reuters. In contrast, the visual slant measurement for Republican-leaning news outlets is \textit{0.43}, suggesting that, on average, smiling images of Republican politicians significantly increase their utility compared to Reuters. Therefore, for Republican politicians, we can conclude that:
\begin{equation}
\hat{\rho}^T(p \in \mathcal{P}_R, y^{k} \in \mathcal{Y}_D, y^{Reuters}) < 0 < \hat{\rho}^T(p \in \mathcal{P}_R, y^{k} \in \mathcal{Y}_R, y^{Reuters})
\end{equation}
For Democratic politicians, in Figure \ref{fig:Hist_Overall2}, the visual slant distribution for Democratic and Republican news outlets also shows clear differences. The average visual slant for Democratic news outlets is \textit{0.43}, indicating that, on average, images of smiling Democratic politicians increase the utility of a Democratic-leaning outlet compared to Reuters (neutral baseline). Conversely, the average visual slant for Republican-leaning new outlets is \textit{-0.38}, suggesting that, on average, smiling images of democratic politicians decrease their utility compared to Reuters. Therefore, for Democratic politicians, we can conclude that:
\begin{equation}
\hat{\rho}^T(p \in \mathcal{P}_D, y^{k} \in \mathcal{Y}_R, y^{Reuters}) < 0 < \hat{\rho}^T(p \in \mathcal{P}_D, y^{k} \in \mathcal{Y}_D, y^{Reuters})
\end{equation}

In Web Appendix $\S$\ref{appssec:hyptest}, we employ two statistical tests to analyze these differences: the Kolmogorov-Smirnov (K-S) test and one-sample t-tests. Both confirm that there are significant differences between the distributions and means of Republican- and Democratic-leaning outlets for both Democratic and Republican politicians. 

\subsubsection{Visual Slant Across News Outlets}
\label{sssec:pol_outlet}
We now document the extent to which visual slant varies across outlets. Figure \ref{fig:side_by_side_smile_effect} shows our visual slant measure for each news outlet ($y^k \in \mathcal{Y}$) for Republican politicians as \( \hat{\rho}^T(p \in  \mathcal{P}_R, y^{k}, y^{Reuters}) \)), and for Democratic politicians as \( \hat{\rho}^T(p \in \mathcal{P}_D, y^{k}, y^{Reuters}) \), ranked in bar charts. In Figure \ref{fig:side_by_side_smile_effect1}, we see that Democratic news outlets have a positive visual slant measurement for Democratic politicians and a negative visual slant for Republican politicians. Conversely, in Figure \ref{fig:side_by_side_smile_effect2}, we see that Republican news outlets exhibit a positive visual slant measurement for Republican politicians and a negative visual slant measurement for Democratic politicians. 

\begin{figure}[htp!]
  \centering
  \begin{subfigure}[t]{0.50\textwidth}
    \centering
    \includegraphics[width=\textwidth]{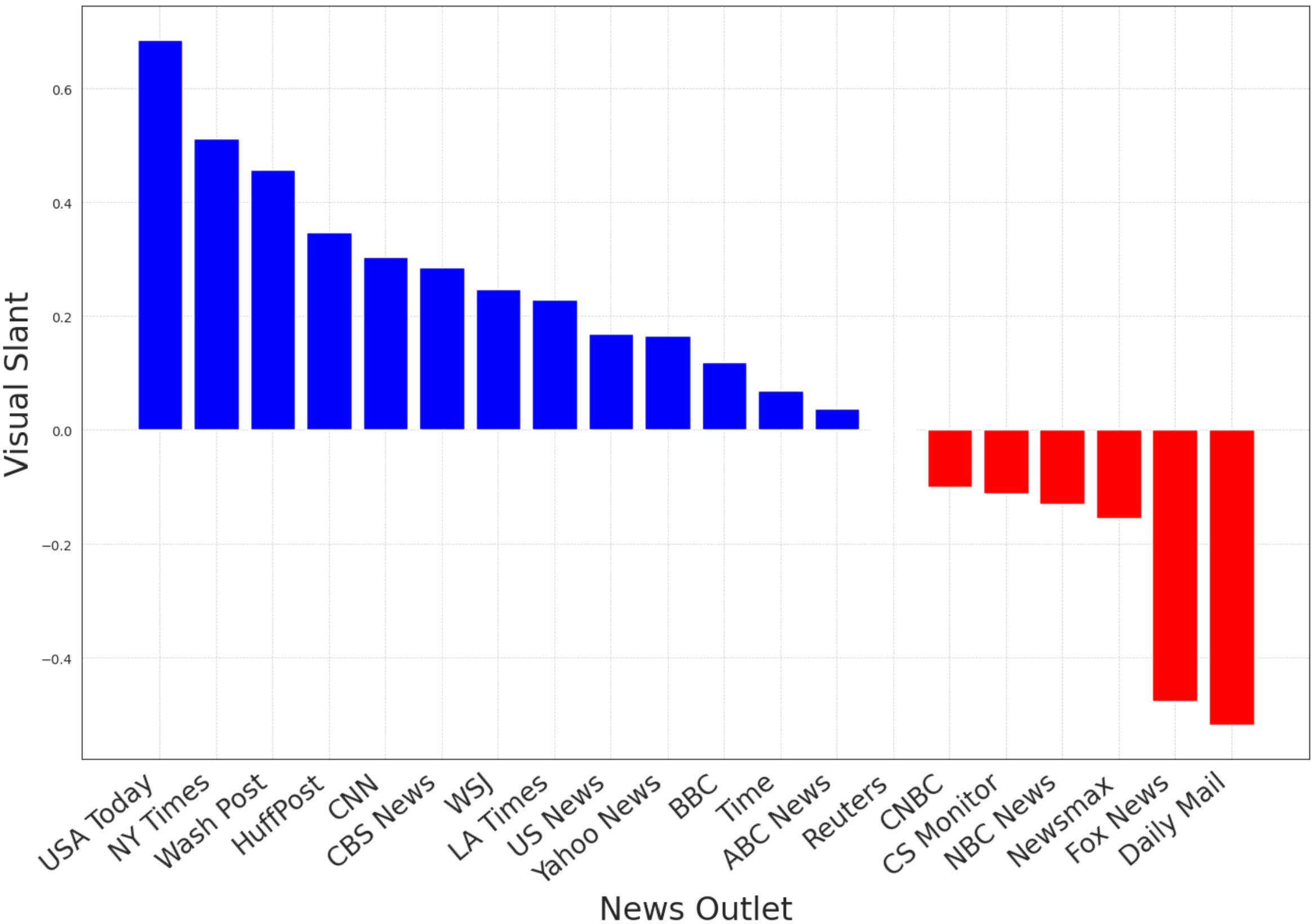}
    \caption{\small Democratic politicians.}
    \label{fig:side_by_side_smile_effect1}
  \end{subfigure}\hfill
  \begin{subfigure}[t]{0.50\textwidth}
    \centering
    \includegraphics[width=\textwidth]{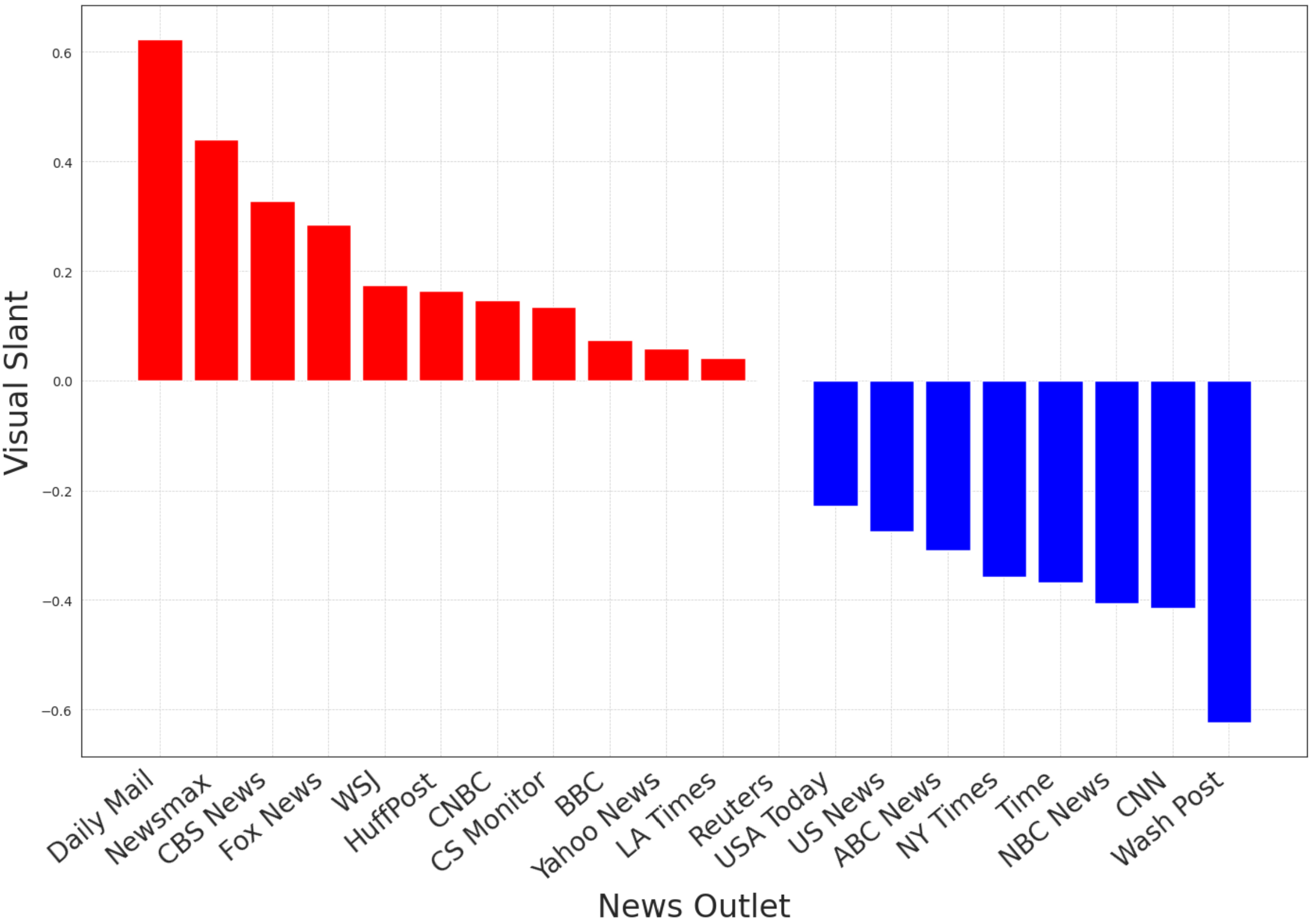}
    \caption{\small Republican politicians.}
    \label{fig:side_by_side_smile_effect2}
  \end{subfigure}
  \caption{\small Overall visual slant measures for Democratic and Republican politicians by news outlet.}
  \label{fig:side_by_side_smile_effect}
\end{figure}

Further, outlets such as \textit{Fox News} and \textit{Daily Mail} display a more positive visual slant measurement for Republican politicians, indicating a more favorable portrayal of these politicians. Simultaneously, these outlets show a more negative visual slant measure of Democratic politicians, displaying a less favorable portrayal of these politicians. On the other hand, outlets like \textit{CNN}, \textit{Washington Post}, and \textit{The New York Times} exhibit a more positive visual slant measurement for Democratic politicians and a more negative visual slant for Republican politicians. Interestingly, some outlets, such as the \textit{Wall Street Journal} and \textit{CBS News}, demonstrate a positive visual slant for both Democratic and Republican politicians. This suggests that their editorial approach may balance portrayals of politicians from both parties, reflecting a more centrist ideological positioning rather than strong partisan alignment.

To establish a single outlet-specific visual slant measurement, we compute the difference between the visual slant measures for Republican and Democratic politicians within each outlet. A larger difference indicates a stronger conservative visual slant. Accordingly, we define our unified measure as the \textit{Conservative Visual Slant (CVS)} as follows:
\begin{equation}\label{equ:CVS}
CVS(y_k) = 
\hat{\rho}^T(p \in \mathcal{P}_R, y^{k}, y^{\text{Reuters}}) 
- \hat{\rho}^T(p \in \mathcal{P}_D, y^{k}, y^{\text{Reuters}}), 
\quad \text{where } y^k \in \mathcal{Y}.
\end{equation}

Intuitively, this metric captures the degree to which an outlet portrays Republican politicians more favorably compared to Democratic politicians. For example, if a Republican-leaning outlet scores high on this measure, it implies that it emphasizes showing Republican politicians with a smile while depicting Democratic politicians without a smile. Figure \ref{fig:smile_effect_BarPlot_NewsOutlet} illustrates the conservative visual slant scores for all outlets in our dataset, sorted in increasing order. The results reveal that outlets such as \textit{Daily Mail} and \textit{Fox News} demonstrate high conservative visual slant, strongly favoring Republican politicians. On the other hand, outlets like \textit{Washington Post}, \textit{USA Today}, and \textit{NY Times} exhibit a pronounced liberal slant, favoring Democratic politicians. Meanwhile, outlets such as \textit{Wall Street Journal}, \textit{BBC News}, and \textit{CBS News} exhibit relatively neutral or low levels of slant. Overall, there is significant heterogeneity in visual slant across outlets, providing insights into which outlets are the most and least polarized visually.

\begin{figure}[htp!]
    \centering
    \includegraphics[width=.7\linewidth]{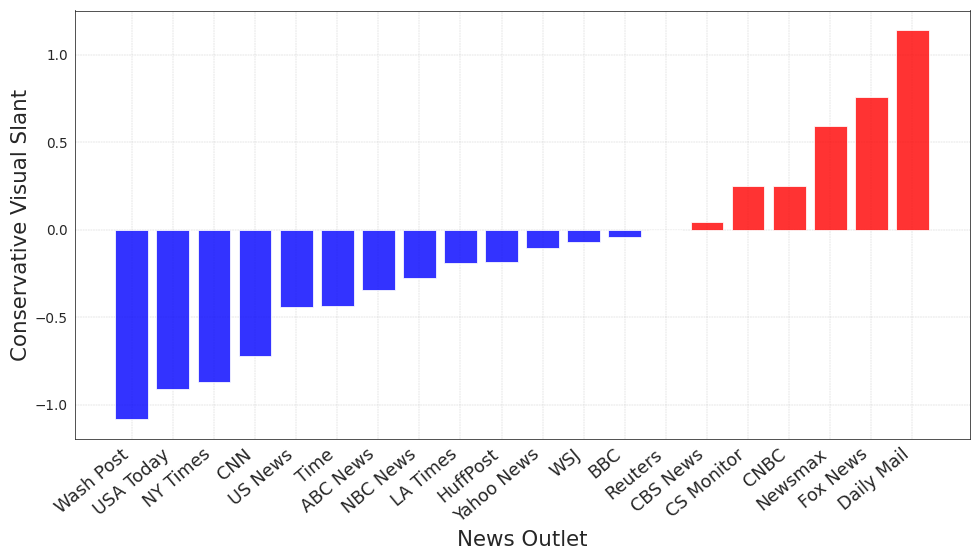}
    \caption{\small {\it Conservative Visual Slant (CVS)} of news outlets.}
    \label{fig:smile_effect_BarPlot_NewsOutlet}
\end{figure}

Lastly, we perform two validation tests to show how our proposed {\it CVS} measure relates to the existing measures of media slant, and how it performs relative to benchmarks that measure visual slant. We present a brief summary of our validation exercise below and refer readers to Web Appendix $\S$\ref{appssec:News_Outlet_Level} for a comprehensive analysis of outlet-level visual slant, including individual histograms, statistical summaries, and additional tests to quantify bias and polarization across news outlets. 
\squishlist
\item \textit{Consistency with external measures of media slant:} We begin by assessing the correlation between our {\it CVS} measure and existing media slant measures derived from independent sources not used in our algorithm. Specifically, we consider three existing measures from \cite{faris2017partisanship}, \cite{flaxman2016filter}, and \cite{allsides2024}. For instance, \cite{faris2017partisanship} quantifies media slant based on the proportion of a media outlet's stories shared on Twitter by users who predominantly retweet conservative-leaning sources. Since our algorithm does not incorporate the data used to construct these external slant measures, a positive correlation between our {\it CVS} measure and these benchmarks would provide validation for our approach. For the 20 outlets in our dataset, we find significant and positive correlations of 0.79, 0.55, and 0.81 with the measures from \cite{faris2017partisanship}, \cite{flaxman2016filter}, and \cite{allsides2024}, respectively. We visualize these correlations and conduct formal statistical tests in Web Appendix $\S$\ref{appssec:MoreResults} to further support this validation.

\item \textit{Comparison to existing measures of visual slant:} Second, we compare our {\it CVS} measure with existing outlet-specific visual slant measures from prior literature. Specifically, we examine the visual slant measure proposed by \cite{boxell2021slanted}, which employs a reduced-form approach similar to that discussed in $\S$\ref{sec:reduced_form}. Our objective is to determine which measure better aligns with existing polarization benchmarks discussed earlier. In the main text, we focus on the conservative share score by \citet{faris2017partisanship}, a widely recognized benchmark in media slant research, including in \citet{boxell2021slanted}. To quantify the alignment between these visual slant measures and the conservative share score from \cite{faris2017partisanship}, we conduct both Pearson and Spearman correlation analyses, assessing how well each measure captures the overall slant of news outlets. Given that our dataset and \cite{boxell2021slanted} share 11 common outlets, we focus on this subset for direct comparison.

\begin{figure}[htp!]
  \centering
  \begin{subfigure}[t]{0.50\textwidth}
    \centering
    \includegraphics[width=\textwidth]{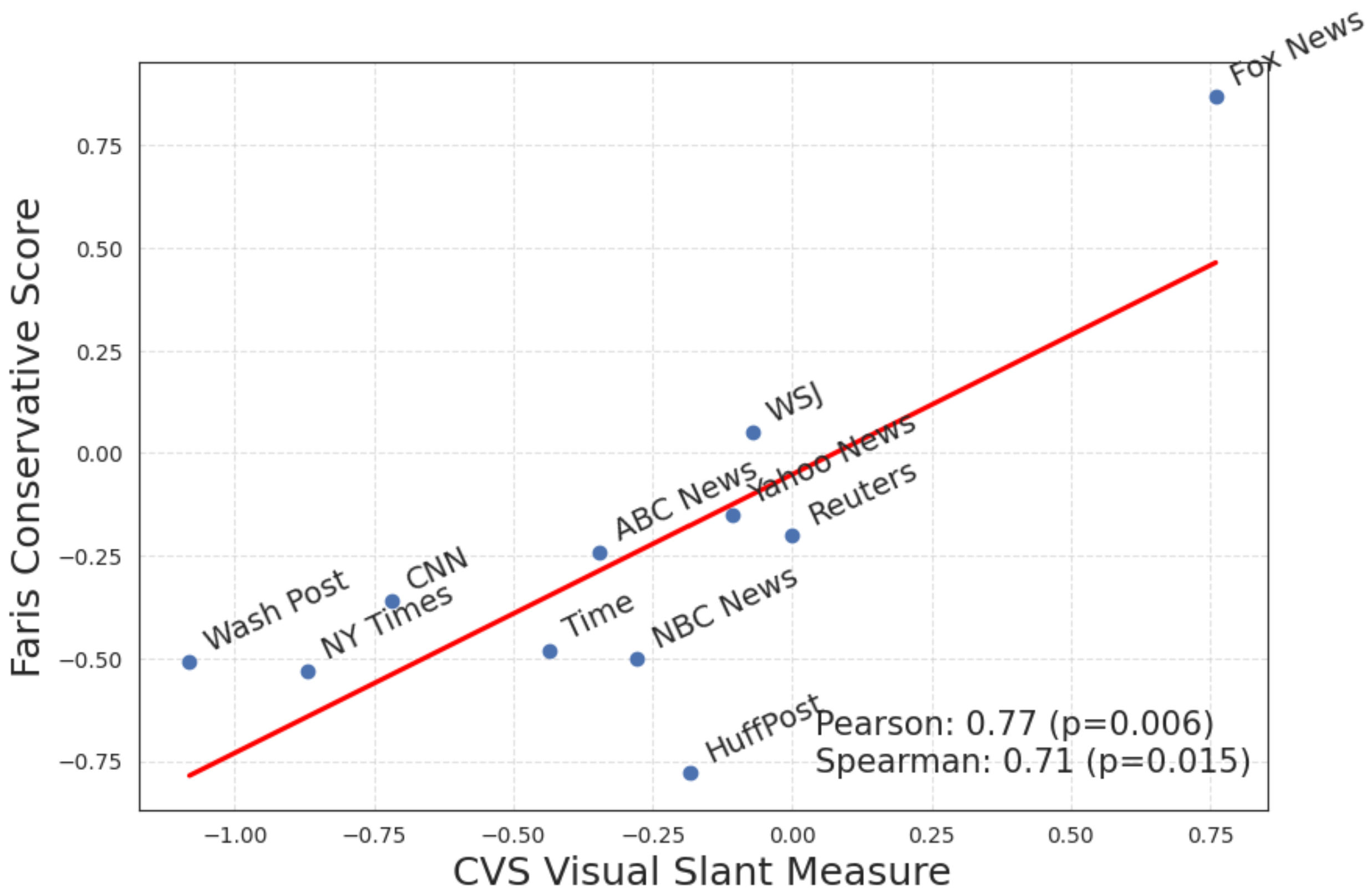}
    \caption{\small {\it Conservative Visual Slant (CVS)} from PMCIG}
    \label{fig:csv_vs_boxell1}
  \end{subfigure}\hfill
  \begin{subfigure}[t]{0.50\textwidth}
    \centering
    \includegraphics[width=\textwidth]{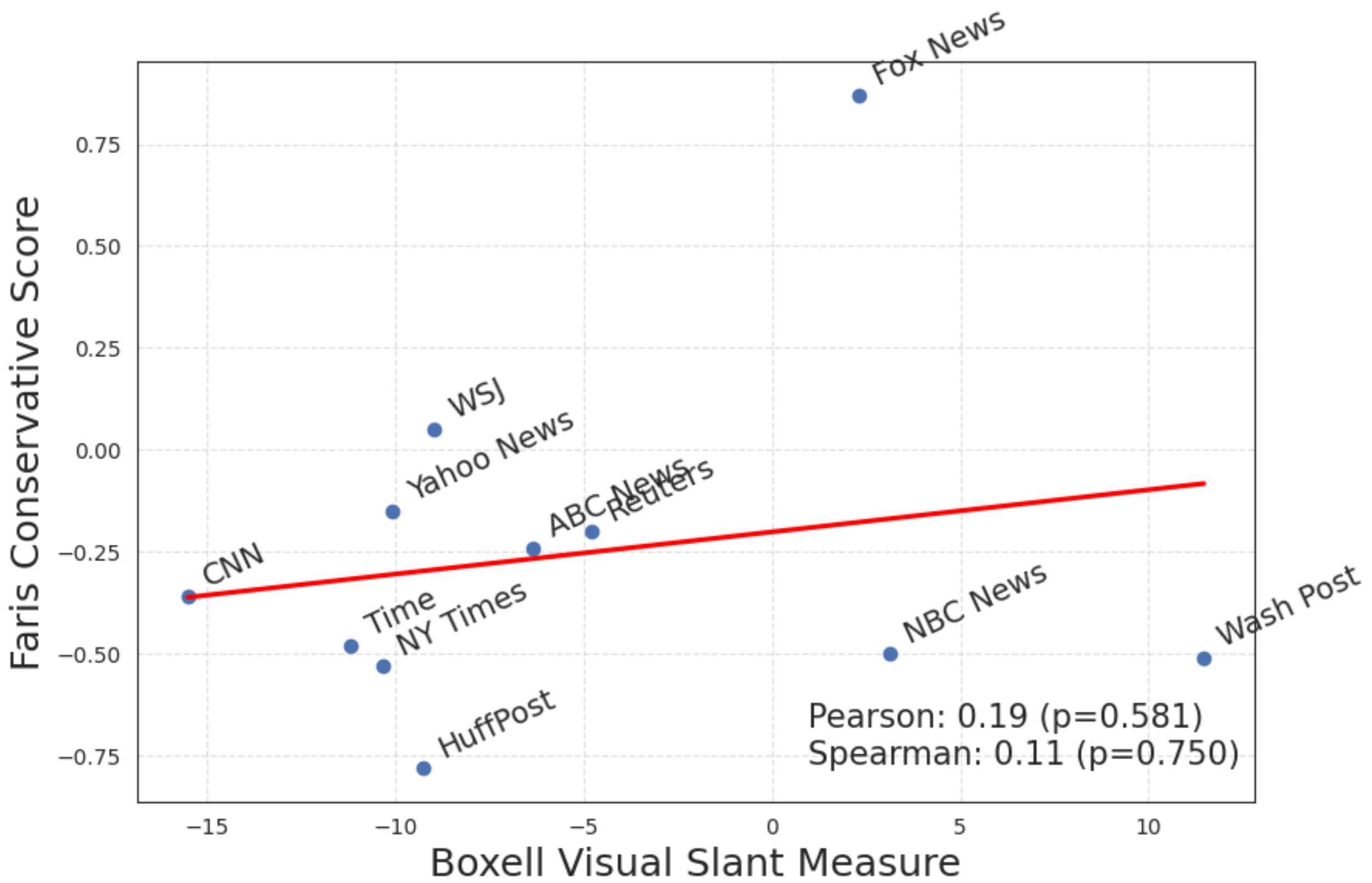}
    \caption{\small Visual Slant Measure from \cite{boxell2021slanted}.}
    \label{fig:csv_vs_boxell2}
  \end{subfigure}
  \caption{\small Comparison of {\it CVS} and visual slant measure by \cite{boxell2021slanted} against Conservative Share Score from \cite{faris2017partisanship}.}
  \label{fig:csv_vs_boxell}
\end{figure}
Figure \ref{fig:csv_vs_boxell} presents a comparative analysis of visual slant measures across news outlets. In both plots, the x-axis represents each outlet’s \textit{conservative share score} from \citet{faris2017partisanship}, while the y-axis represents either our {\it CVS} metric (Figure \ref{fig:csv_vs_boxell1}) or the visual slant measure from \cite{boxell2021slanted} (Figure \ref{fig:csv_vs_boxell2}). Each point corresponds to a news outlet, with a red trendline illustrating the relationship between visual slant and conservative share scores. As shown in these figures, our {\it CVS} measure exhibits a statistically significant and strong Pearson and Spearman correlation with the conservative share score from \cite{faris2017partisanship}, whereas the visual slant measure from \cite{boxell2021slanted} shows only a weak, statistically insignificant correlation. This finding further validates our method, demonstrating that our {\it CVS} measure more accurately captures media slant compared to prior visual slant metrics. Additionally, in Web Appendix $\S$\ref{appssec:comparison}, we replicate the analysis presented in this part using two other widely recognized conservative share scores and demonstrate the robustness of our findings.
\squishend
In summary, these results demonstrate the validity of our approach in comparison to existing benchmark measures of media bias. 

\subsubsection{Visual Polarization Across Politicians}
\label{sssec:pol_politician}

\begin{figure}[htp!]
    \centering
    \includegraphics[width=0.8\linewidth]{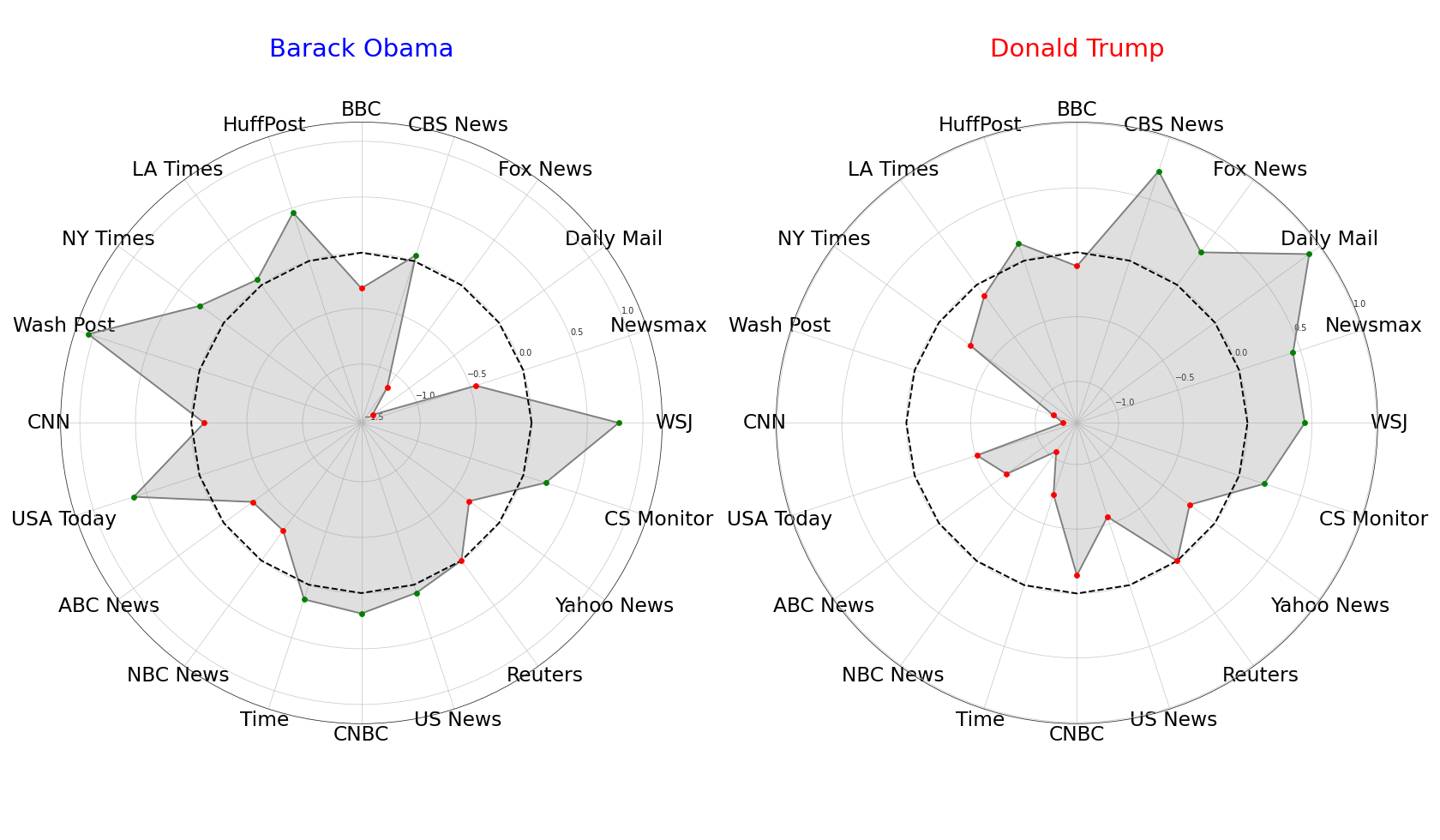}
    \caption{\small Radar plots of visual slant in each news outlet for Barack Obama and Donald Trump.}
    \label{fig:fourpoliticinas}
\end{figure}

We now examine how different politicians are portrayed across news outlets and identify the politicians with the most polarizing depictions in media. Figure \ref{fig:fourpoliticinas} presents the visual slant measures for two prominent figures: Barack Obama and Donald Trump. Consistent with our earlier findings, we observe a clear divide in visual slant scores -- Obama receives more favorable portrayals in Democratic-leaning outlets, while Trump is depicted more positively in Republican-leaning outlets. We extend this analysis to other politicians, with detailed results provided in Web Appendix $\S$\ref{appssec:Politicians_Level}.

Next, we develop a measure for the overall extent of visual polarization for each politician. Intuitively, we expect to observe a greater extent of variability in visual slant measures across outlets for a more polarizing politician. As such, we quantify the {\it Overall Visual Polarization (OVP)} for each politician, which is measured by the standard deviation of the polarization across all news outlets as:
\begin{equation}
\label{equ:OVP}
OVP(p) = 
\sqrt{\frac{1}{|\mathcal{Y}|} \sum_{y^k \in \mathcal{Y}} \left( \hat{\rho}^T(p, y^{k}, y^{\text{Reuters}}) - \mu_p \right)^2},
\quad \text{where} \quad \mu_p = \frac{1}{|\mathcal{Y}|} \sum_{y^k \in \mathcal{Y}} \hat{\rho}^T(p, y^{k}, y^{\text{Reuters}})
\end{equation}

\begin{figure}[htp!]
    \centering
    \includegraphics[width=0.7\linewidth]{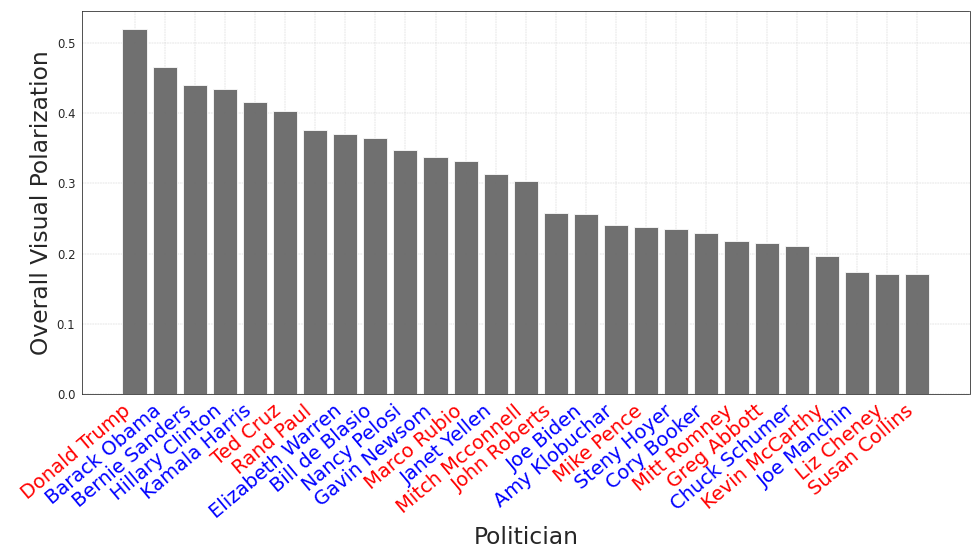}
    \caption{\small {\it Overall Visual Polarization (OVP)} of politicians across all news outlets, ranked from most to least polarized.}
    \label{fig:Overall_Politicians}
\end{figure}

Figure \ref{fig:Overall_Politicians} ranks the politicians from most polarized to least polarized based on this criterion. We see that Donald Trump and Barack Obama are the top two politicians with the most visually polarizing portrayal across all outlets. Given that both were the president/presidential candidate for significant chunks of our observation period and at the forefront of multiple polarizing discussions and events, this is understandable. Further, Rand Paul and Ted Cruz, both prominent Republican politicians, also show high levels of polarization, likely attributed to their roles in policy debates and media prominence during key political events. On the Democratic side, figures such as Bernie Sanders, Hillary Clinton, and Kamala Harris display significant polarization, reflecting their leadership roles and ideological positions within the party. Further, on the right end of Figure \ref{fig:Overall_Politicians}, we observe politicians with lower {\it OVP} scores. Notably, Susan Collins, Liz Cheney, and Joe Manchin rank among the least polarized, aligning with their reputations for bipartisanship and moderation \citep{cnn_2018}. These findings have important implications for political strategy, particularly in how politicians shape their election campaigns when deciding whether to mobilize their base or appeal to a broader, centrist electorate. 


Finally, we seek to validate our politician-specific {\it OVP} measure by comparing it with external indicators of a politician's level of polarization. However, there exists no widely accepted, politician-specific polarization metric. To overcome this issue, we propose that a politician’s ideological alignment with their primary constituency serves as a meaningful proxy, since more polarizing politicians are likely to perform better in ideologically aligned constituencies and struggle in misaligned ones. To quantify this alignment, we use the politician’s party's success in the 2016 election within their state, by measuring their party's percentage point advantage in that election. We then analyze the correlation between this measure and our {\it OVP} metric. Figure \ref{fig:OVP_Vote} presents the results from this exercise. We see a strong and significant correlation between the two measures, which further supports the validity of our {\it OVP} measure.

\begin{figure}[htp!]
    \centering
    \includegraphics[width=0.7\linewidth]{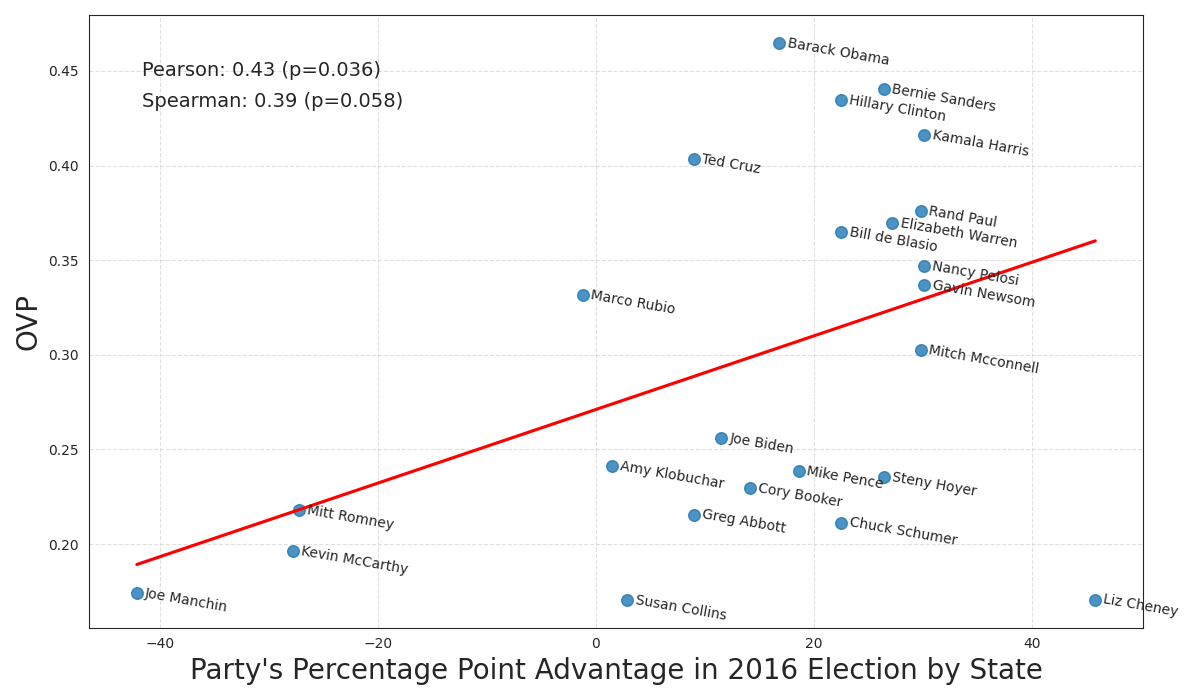}
    \caption{\small Relationship between a politician's {\it Overall Visual Polarization (OVP)} and their ideological alignment with their constituency.}
    \label{fig:OVP_Vote}
\end{figure}

This finding supports the idea that politicians from states with strong partisan leanings tend to adopt more pronounced partisan positions without seeking to appeal to a broad electorate. For example, Bernie Sanders (Vermont) and Ted Cruz (Texas) come from states with clear ideological identities and exhibit high {\it OVP} values, likely reflecting their strong partisan stances and the resulting polarized media portrayals. In contrast, Joe Manchin (West Virginia) and Susan Collins (Maine), who represent states where their party is in the minority, have lower {\it OVP} values, which reflects their moderate positions.

\section{Conclusion}
\label{sec:conclusion}
In this paper, we present a framework for measuring slant and polarization in the visual content accompanying news articles. We propose the Polarization Measurement Using Counterfactual Image Generation (PMCIG) algorithm, which quantifies news outlets' preference for slanted imagery -- such as smiling images -- to convey positive (vs. negative) representation of politicians. Our framework combines the economic structure of the problem with generative models to generate comparable counterfactual images and measure polarization in visual content. Notably, our algorithm overcomes the key limitations in traditional descriptive methods due to information loss in the feature extraction phase by using the rich information contained in images. 

In our empirical analysis, we apply the PMCIG framework to a decade-long dataset that covers 20 major news outlets and 30 prominent politicians. We identify clear patterns of ideological slanting and political polarization in the visual representation of political figures. We validate our measure of visual slant by demonstrating a high correlation between our measure and the existing measures used for media slant and partisanship. Our framework measures visual slant and polarization with detailed granularity, highlighting differences both at the outlet level and for individual politicians. Among outlets, we find that {\it Daily Mail} and {\it Fox News} display the strongest Republican-leaning visual slant, while {\it Washington Post} and {\it The New York Times} exhibit the strongest Democratic-leaning slant. In contrast, {\it CBS News} and {\it Wall Street Journal} are among the outlets with the lowest overall visual slant. At the individual level, Donald Trump and Barack Obama stand out as the most polarizing figures, whereas Joe Manchin, Liz Cheney, and Susan Collins are among the least polarizing in their visual portrayal across news outlets. 

In summary, the PMCIG framework offers a systematic approach to analyze how ideological preferences shape visual content in news media and contribute to polarization. Nevertheless, our paper has limitations that serve as excellent avenues for future research. For instance, our analysis is based on data from the United States, and extending the framework to other regions, such as Europe, could uncover cross-cultural differences in visual polarization. Additionally, while the framework measures visual slant and polarization, it does not examine the downstream impact on individuals' beliefs or behavior, which could be fruitful avenues for future research. Future studies could also apply the PMCIG framework to other forms of media, such as social media or advertising, and investigate the extent of ideological slanting/bias in these settings. 


\section*{Competing Interests Declaration}
Author(s) have no competing interests to declare.



\setcounter{table}{0}
\setcounter{page}{0}
\setcounter{figure}{0}

\renewcommand{\thetable}{A\arabic{table}}
\renewcommand{\theequation}{A.\arabic{equation}}
\renewcommand{\thepage}{\roman{page}}
\renewcommand{\thefigure}{A.\arabic{figure}}

\renewcommand{\thepage}{\roman{page}}
\newpage
\renewcommand*\appendixpagename{Web Appendix}

\begin{appendices}

\section{Details of Data Collection and Cleaning}
\label{appsec:data}

\subsection{Details of Politicians and News Outlets}
\label{appssec:polnews}
Two important details about our data collection relate to the set of politicians and news outlets. Below we present these two lists:

\squishlist
\item \textbf{Politicians:} To select the sample of 30 politicians, we first created a large set of individuals who have run for public offices from either Democratic or Republican party. We then sampled a set of top 30 politicians based on their search volume on Google Trends. As such, the higher the search volume for a politician, the more likely it is to select that politician for our main sample. The resulting sets of Democratic and Republican politicians is presented below:
\squishlist
\item \textit{Democratic Politicians}: This list contains (1) Barack Obama, (2) Joe Biden, (3) Hilary Clinton, (4) Bernie Sanders, (5) Nancy Pelosi, (6) Kamala Harris, (7) Chuck Schumer, (8) Corey Booker, (9) Amy Klobuchar, (10) Elizabeth Warren, (11) Gavin Newsom, (12) Bill de Blasio, (13) James Clyburn, (14) Janet Yellen, (15) Joe Manchin, and (16) Steny Hoyer. 

\item \textit{Republican Politicians}: This list contains (1) Donald Trump, (2) Ted Cruz, (3) Marco Rubio, (4) Mitt Romney, (5) Mitch McConnell, (6) Greg Abbott, (7) Rand Paul, (8) Mike Pence, (9) Kevin McCarthy, (10) Susan Collins, (11) Liz Chenney, (12) John Roberts, (13) Hal Rogers, (14) Andy Biggs.
\squishend

\item \textbf{News Outlets:} We selected top 20 news outlets following the work by \cite{flaxman2016filter}. Unlike politicians, news outlets do not have a clear political affiliation and party. However, there are several indices for media slant based on the readership and coverage \citep{groseclose2005measure, flaxman2016filter}. For example, Figure \ref{fig:news_outlets} shows the news outlets in our data using the \textit{conservative share} measure used by \cite{flaxman2016filter}. In our main analysis, we do not use the media slant indices for model building and measuring political polarization in visual content. However, we use these measures for validation of our main measures.  
\squishend
\begin{figure}[t]
    \centering
    \includegraphics[width=0.7\textwidth]{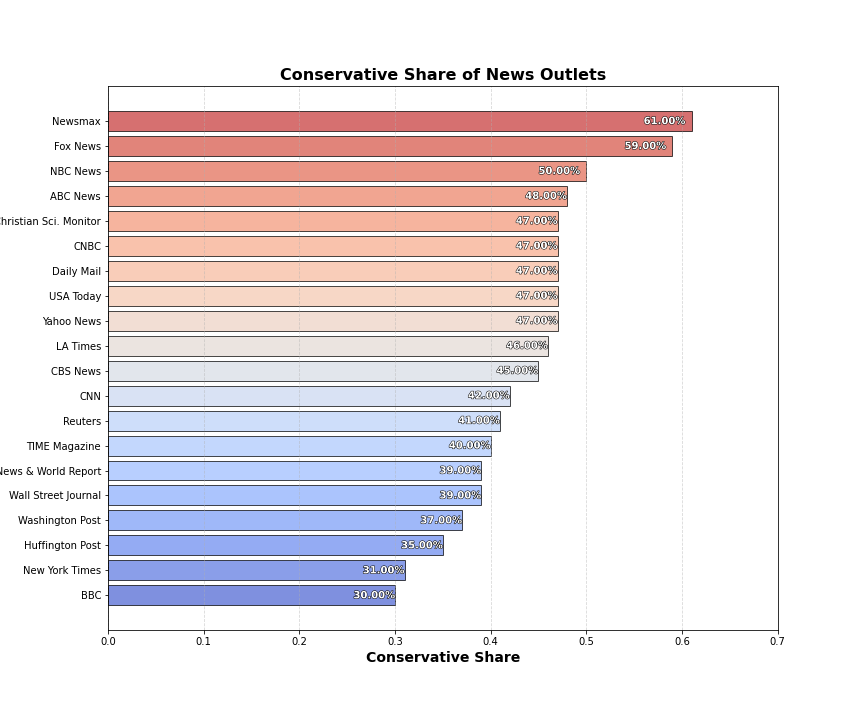}
    \caption{\small Conservative share of news outlets, from \citet{flaxman2016filter}.}
    \label{fig:news_outlets}
\end{figure}

\subsection{Details of Data Cleaning}
\label{appssec:cleaning}

\begin{figure}[htp!]
    \centering
    \includegraphics[width=0.9\textwidth]{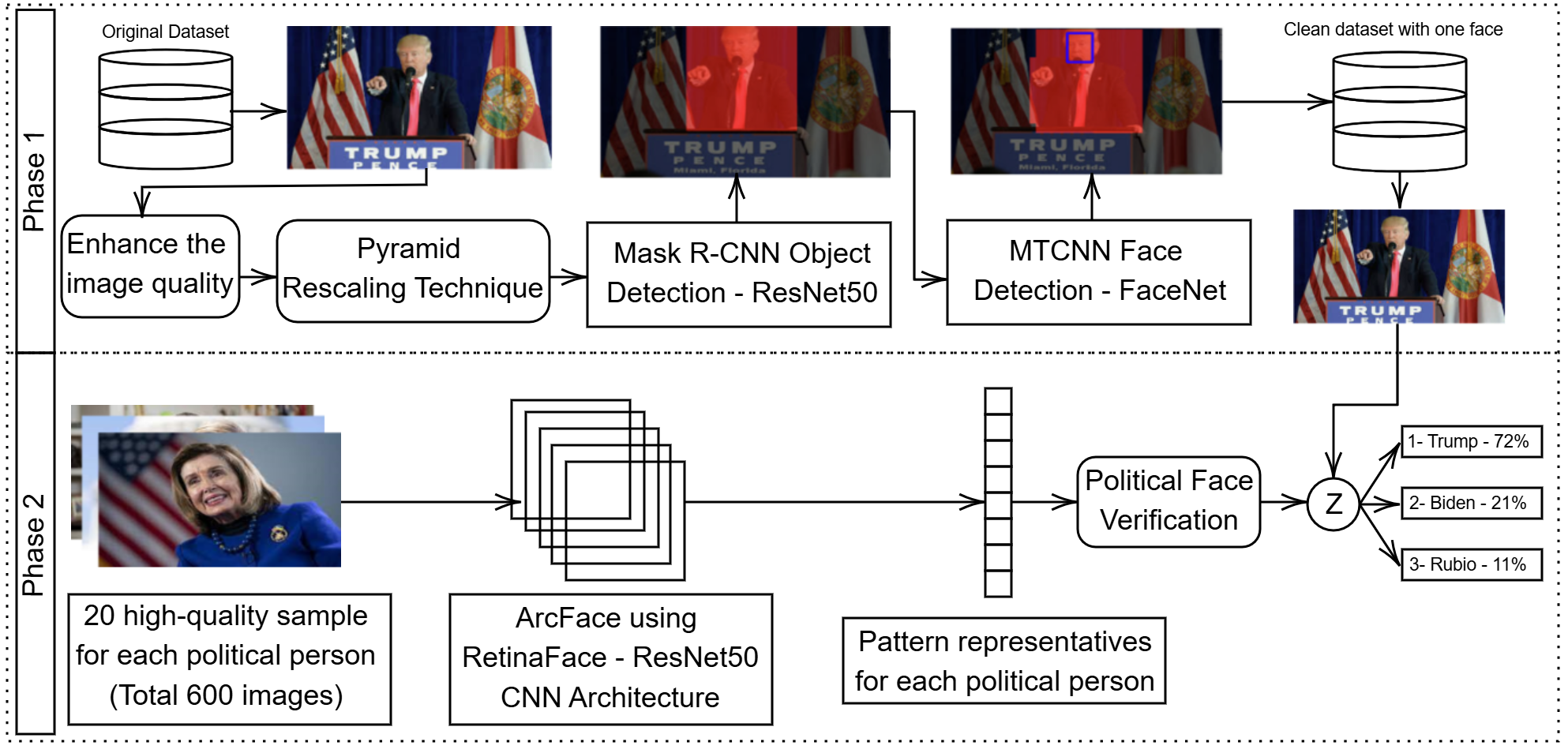} 
    \caption{\small Two-phase computer vision framework for political person verification.}
    \label{fig:1}
\end{figure}

In this section, we present details about data cleaning. We first present the manual verification procedure we used to ensure that SerpAPI does not miss important data. To address the limitations of the automated method, we employ an additional manual process for quality control. This involves saving HTML files of search results locally and using Beautiful Soup to extract images and metadata. This method ensures the accuracy and completeness of our dataset. 

We now present the two-phase data cleaning procedure in Figure \ref{fig:1} in greater detail. The first phase, designed for accurate face detection, consists of two main steps. Initially, we perform pyramid rescaling, adjusting each image to three scales (1.1, 0.8, and 0.6) to enhance focus, similar to adjusting a camera lens \citep{lin2017feature}. Following this, Mask R-CNN, which uses the ResNet50 architecture as its backbone, identifies and outlines humans within these images \citep{he2017mask}, referring to an object detection and instance segmentation task. Next, we use multitask cascaded Convolutional Neural Networks (MTCNN) for face detection within the regions identified by Mask R-CNN. The detected faces are subsequently processed by FaceNet for face recognition \citep{zhang2016joint}. This dual approach not only increases the efficiency of our face detection process but also significantly improves its accuracy. 

The second phase is the face verification stage of our computer vision framework. We initiate the process by feeding ArcFace with twenty samples for each of the thirty politicians to create a unique facial pattern representative for each political person \citep{deng2019arcface}. Then, we take the images from the first phase, which are confirmed to have just one face, and run them through ArcFace. This system compares the detected face with our database of political figures and ranks the top three matches, providing a similarity percentage for each. 

In total, we retain 63,188 images with exactly one face, that have been verified by our dual framework. These images represent a clean and reliable dataset for our study. The distribution of the final data is shown in Figure \ref{fig:2}. This dataset allows us to accurately analyze media representation of political figures, ensuring that each image is correctly attributed to the intended individual.

\begin{figure}[htp!]
    \centering
    \includegraphics[width=0.9\textwidth]{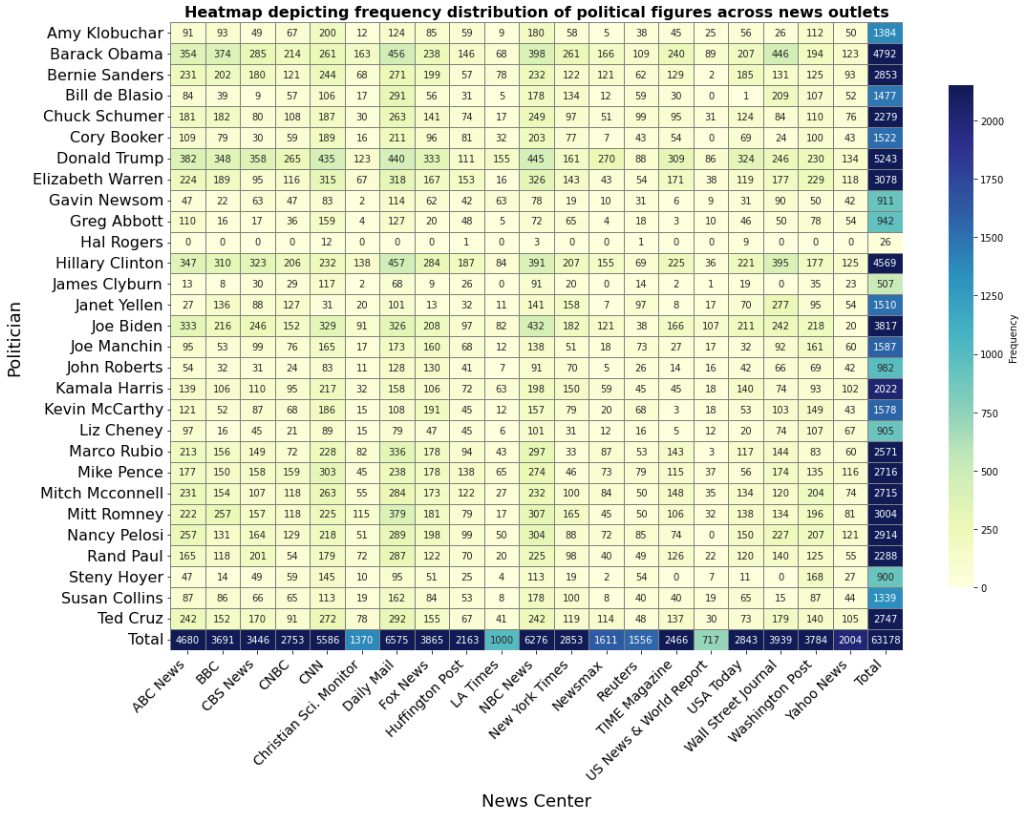} 
    \caption{\small Frequency distribution of politicians across outlets.}
    \label{fig:2}
\end{figure}

\section{Theoretical Analysis of Bias for the Two-step Model}
\label{appsec:drawback}

\subsection{Extraction Bias}
\label{sssec:ExractionBias}

As discussed in $\S$\ref{ssec:drawbacks}, \textit{extraction bias} arises from the use of the predicted variable \(\hat{T}\) in the analysis because the true variable \(T\) is not observable. This issue is primarily related to the machine learning model \(f_1\), where transfer learning is employed to predict the feature of interest. We provide the theoretical econometric framework demonstrating how extraction bias can lead to biased estimation of \(\beta\).

Recall that \(T\) is not directly observed; rather, it is extracted from the high-dimensional unstructured image data $Z$ using a first-stage machine learning model $f_1$. This leads to two types of errors in the estimate \(\hat{T}\): Measurement Error (\(\epsilon_{r}\)) and Extraction Error (\(\epsilon_{e}\)).

\begin{remark}\label{rem:measurmentbias}
Measurement Error (\(\epsilon_r\)) arises from inaccuracies in predictions or measurements that are not systematically correlated with the true value of \(T\) \citep{hansen2022econometrics}. This can be formally defined as:
\[
T = \hat{T} + \epsilon_r \quad \text{where} \quad \text{Cov}(T, \epsilon_r) = 0
\]
\end{remark}

\begin{remark}\label{rem:extractiobbias}
Extraction Error (\(\epsilon_e\)) occurs when the estimation method includes additional, indirectly related features \citep{wei2022unstructured}. This error implies that the error term is correlated with the true value \(T\):
\[
T = \hat{T} + \epsilon_e \quad \text{where} \quad \text{Cov}(T, \epsilon_e) \neq 0
\]
\end{remark}
Measurement error can be viewed as random noise added to the estimate \(\hat{T}\), resulting in a variable that is not systematically biased. In contrast, extraction errors occur when the machine learning model, trained to predict the specific feature of interest \(T\), inadvertently captures additional non-focal features that are correlated with \(T\). For instance, a machine learning model designed to detect the facial expression of faces whether a person is happy may also inadvertently capture brightness levels in the image, as brighter images are often associated with happier expressions. This can introduce an extraction error \(\epsilon_{e}\) into the model's predictions, potentially biasing the results.

In our setting, suppose that images used by {\it CNN} are generally brighter. In this case, the model may also capture brightness rather than true happiness. As a result, the predicted variable \(\hat{T}\) reflects not only the intended feature (happiness) but also the unintended influence of brightness. This issue complicates the analysis by conflating the natural effect of emotional expression with confounding factors, making it difficult to separate their impacts. To assess this challenge theoretically, we can divide it into \textit{Well-specified} and \textit{Partially-specified} models.

Consider scenarios where \(T\) represents the actual value of the feature. Achieving this oracle estimate can be done by using human judgment to classify photos as Happy (\(T=1\)) or Unhappy (\(T=0\)), though this method is expensive. Now, assume a linear model where \(Y\) is a binary variable indicating the news outlet (1 for {\it CNN}, 0 for {\it Fox news}). In this case, the parameter of interest, \(\beta\), can be estimated without bias because of the following Proposition.

\begin{proposition} \textbf{Well-specified:}
Let \((Y, X)\) denote observed regressors, where \(Y\) is the variable of interest and \(X \in \mathbb{R}^k\) captures other explanatory variables. Consider the model:

\begin{equation}
\begin{aligned}
&\hat{T}_{oracle} = T + \epsilon_{r}, \quad \text{where} \quad \hat{T}_{oracle} = \text{Human}, \\
&T  = \alpha + \beta Y + \gamma^\top X + e,
\end{aligned}
\end{equation}
where \(e\) and \(\epsilon_r\) are i.i.d. errors, and \(\mathrm{Cov}(T, \epsilon_r) = 0\) (Remark \ref{rem:measurmentbias}). Suppose \(\hat{\beta}\) is the OLS estimator of the parameter of interest (coefficient of \(Y\)). Then \(\hat{\beta}\) is an unbiased estimator of \(\beta\). As \(n \to \infty\):
\[
\sqrt{n}(\hat{\beta} - \beta) \xrightarrow{d} \mathcal{N}\Bigl(0, \Sigma_\beta = \mathbf{e}_2^\top\Bigl(A_0^{-1} \Sigma A_0^{-1}\Bigr)\mathbf{e}_2\Bigr),
\]
where \(\mathbf{Z} = (1, Y, X_1, \ldots, X_k)^\top\), 
\(A_0 = \mathbb{E}[\mathbf{Z}\mathbf{Z}^\top]\),
\(\Sigma = \mathbb{E}[\mathbf{Z}\theta\theta\mathbf{Z}^\top]\),
\(\theta = e + \epsilon_r\),
and \(\mathbf{e}_2 = (0, 1, 0, \ldots, 0)^\top\).
\end{proposition}

\begin{proof}
In scenarios where \(T\) represents an oracle estimation of the population, it is possible to leverage human judgment to categorize photos into Happy (\(T=1\)) or Unhappy (\(T=0\)). Assume a linear model where \(Y\) is a binary variable representing the news outlet (1 for {\it Fox News}, 0 for {\it CNN}).

Let \(\hat{T}_{oracle} = \alpha + \beta Y + \gamma^\top X + \theta\), where \(\theta = e + \epsilon_r\). Define \(\mathbf{Z} = (1, Y, X_1, \ldots, X_k)^\top\) and stack the observations:

\[
\hat{\mathbf{T}}_{oracle} = \mathbf{Z} \boldsymbol{\beta}_{\mathrm{true}} + \boldsymbol{\theta}, \quad \boldsymbol{\beta}_{\mathrm{true}} = (\alpha, \beta, \gamma_1, \ldots, \gamma_k)^\top.
\]

Here, the OLS estimator is:
\[
\hat{\boldsymbol{\beta}} = (\mathbf{Z}^\top \mathbf{Z})^{-1} \mathbf{Z}^\top \hat{\mathbf{T}}_{oracle} = \boldsymbol{\beta}_{\mathrm{true}} + (\mathbf{Z}^\top \mathbf{Z})^{-1} \mathbf{Z}^\top \boldsymbol{\theta}.
\]

Since \(\mathrm{Cov}(T, \epsilon_r) = 0\) (Remark \ref{rem:measurmentbias}), \(\epsilon_r\) is uncorrelated with \((Y, X)\). Therefore:
\[
\mathbb{E}[\theta\, \mathbf{Z}] = \mathbb{E}[(e + \epsilon_r)\, \mathbf{Z}] = \mathbf{0}.
\]

Taking the expectation of \(\hat{\boldsymbol{\beta}}\):
\[
\mathbb{E}[\hat{\boldsymbol{\beta}}] = \boldsymbol{\beta}_{\mathrm{true}} + (\mathbf{Z}^\top \mathbf{Z})^{-1} \mathbb{E}[\mathbf{Z}^\top \boldsymbol{\theta}] = \boldsymbol{\beta}_{\mathrm{true}},
\]
so \(\hat{\boldsymbol{\beta}}\) is unbiased. In particular, \(\hat{\beta}\) is an unbiased estimator of \(\beta\). By the Law of Large Numbers (LLN) and Central Limit Theorem (CLT):
\[
\sqrt{n} (\hat{\boldsymbol{\beta}} - \boldsymbol{\beta}_{\mathrm{true}}) \xrightarrow{d} \mathcal{N}\Bigl(\mathbf{0}, A_0^{-1} \Sigma A_0^{-1}\Bigr),
\]
where:
\[
A_0 = \mathbb{E}[\mathbf{Z}\mathbf{Z}^\top], \quad \Sigma = \mathbb{E}[\mathbf{Z}\theta\theta\mathbf{Z}^\top].
\]
For \(\hat{\beta}\), focusing on the second component:
\[
\sqrt{n} (\hat{\beta} - \beta) \xrightarrow{d} \mathcal{N}\Bigl(0, \mathbf{e}_2^\top A_0^{-1} \Sigma A_0^{-1} \mathbf{e}_2\Bigr),
\]
where \(\mathbf{e}_2 = (0, 1, 0, \ldots, 0)^\top\). This confirms that in a well-specified model, the estimator \(\hat{\beta}\) remains unbiased, and only the variance is affected by the inclusion of the error terms.
\end{proof}

The Proposition above confirms the reliability of using such models in well-specified scenarios.

\begin{proposition}\label{pro:extractionbias} \textbf{Partial-specified:}
Let \((Y, X)\) again denote observed regressors, where \(Y\) is the variable of interest and \(X \in \mathbb{R}^k\) captures other explanatory variables, ensuring no omitted variable bias. Consider the model:

\begin{equation}
\begin{aligned}\label{secondtype}
&\hat{T}_{ml} = T + \epsilon_{r} + \epsilon_{e}, \quad \text{where} \quad \hat{T}_{ml} = f_{1}(Z), \\
&T  = \alpha + \beta Y + \gamma^\top X + e,
\end{aligned}
\end{equation}
where \(e\) and \(\epsilon_r\) are i.i.d.\ errors uncorrelated with \((Y, X)\), and \(\epsilon_e\) is an extraction error that is endogenous (\(\mathrm{Cov}(T, \epsilon_e) \neq 0\); see Remark \ref{rem:extractiobbias}). Suppose \(\hat{\beta}\) is the OLS estimator of the parameter of interest (coefficient of \(Y\)). Then, \(\hat{\beta}\) is a biased estimator due to the endogeneity introduced by \(\epsilon_e\). As \(n \to \infty\):
\[
\sqrt{n}(\hat{\beta} - \beta) \xrightarrow{d} \mathcal{N}\Bigl(\mu = \mathbf{e}_2^\top A_0^{-1} \mathbb{E}[\mathbf{Z}\epsilon_e], \quad \Sigma_\beta = \mathbf{e}_2^\top\Bigl(A_0^{-1} \Sigma A_0^{-1}\Bigr)\mathbf{e}_2\Bigr),
\]
where \(\mathbf{Z} = (1, Y, X_1, \ldots, X_k)^\top\), 
\(A_0 = \mathbb{E}[\mathbf{Z}\mathbf{Z}^\top]\),
\(\Sigma = \mathbb{E}[\mathbf{Z}\theta\theta\mathbf{Z}^\top]\),
\(\theta = e + \epsilon_r + \epsilon_e\),
and \(\mathbf{e}_2 = (0, 1, 0, \ldots, 0)^\top\).
\end{proposition}

\begin{proof}
In real-world scenarios, \(\hat{T}_{ml}\) often includes extraction errors, leading to biases.

Let \(\hat{T}_{ml} = \alpha + \beta Y + \gamma^\top X + \theta\), where \(\theta = e + \epsilon_r + \epsilon_e\). Define \(\mathbf{Z} = (1, Y, X_1, \ldots, X_k)^\top\) and stack the observations:
\[
\hat{\mathbf{T}}_{ml} = \mathbf{Z} \boldsymbol{\beta}_{\mathrm{true}} + \boldsymbol{\theta}, \quad \boldsymbol{\beta}_{\mathrm{true}} = (\alpha, \beta, \gamma_1, \ldots, \gamma_k)^\top.
\]

Here, the OLS estimator is:
\[
\hat{\boldsymbol{\beta}} = (\mathbf{Z}^\top \mathbf{Z})^{-1} \mathbf{Z}^\top \hat{\mathbf{T}}_{ml} = \boldsymbol{\beta}_{\mathrm{true}} + (\mathbf{Z}^\top \mathbf{Z})^{-1} \mathbf{Z}^\top \boldsymbol{\theta}.
\]
\(e\) and \(\epsilon_r\) are assumed to have zero mean and are uncorrelated with \((Y, X)\), but the extraction error \(\epsilon_e\) is endogenous. Consequently, due to \(\mathrm{Cov}(T, \epsilon_e) \neq 0\) (Remark \ref{rem:extractiobbias}), $\mathbb{E}[\mathbf{Z}^\top \boldsymbol{\theta}] = \mathbb{E}[\mathbf{Z}^\top \epsilon_e] \neq 0$, we have:
\[
\mathbb{E}[\hat{\boldsymbol{\beta}}] = \boldsymbol{\beta}_{\mathrm{true}} + (\mathbf{Z}^\top \mathbf{Z})^{-1} \mathbb{E}[\mathbf{Z}^\top \boldsymbol{\theta}] \neq \boldsymbol{\beta}_{\mathrm{true}}.
\]

By the LLN and CLT:
\[
\sqrt{n} (\hat{\boldsymbol{\beta}} - \boldsymbol{\beta}_{\mathrm{true}}) \xrightarrow{d} \mathcal{N}(A_0^{-1} b_0, \, A_0^{-1} \Sigma A_0^{-1}),
\]
where \(A_0 = \mathbb{E}[\mathbf{Z}\mathbf{Z}^\top]\), \(b_0 = \mathbb{E}[\mathbf{Z}\theta]\) and \(\Sigma = \mathbb{E}[\mathbf{Z}\theta\theta\mathbf{Z}^\top]\).

To isolate the second component (\(\beta\)) in \(\hat{\boldsymbol{\beta}}\), let \(\mathbf{e}_2 = (0, 1, 0, \ldots, 0)^\top\). Then:
\[
\sqrt{n} (\hat{\beta} - \beta) \xrightarrow{d} \mathcal{N}\Bigl(\mu = \mathbf{e}_2^\top A_0^{-1} b_0, \quad \Sigma_\beta = \mathbf{e}_2^\top A_0^{-1} \Sigma A_0^{-1} \mathbf{e}_2\Bigr),
\]
where:
\[
A_0^{-1}\,b_0
=
\begin{pmatrix}
a_{11}\,\mathbb{E}[\epsilon_e]
+
a_{12}\,\mathbb{E}[Y\,\epsilon_e]
+
\displaystyle\sum_{j=1}^k a_{1,(j+2)}\,\mathbb{E}[X_j\,\epsilon_e] \\[6pt]
a_{21}\,\mathbb{E}[\epsilon_e]
+
a_{22}\,\mathbb{E}[Y\,\epsilon_e]
+
\displaystyle\sum_{j=1}^k a_{2,(j+2)}\,\mathbb{E}[X_j\,\epsilon_e] \\[6pt]
\vdots \\[6pt]
a_{(2+k)1}\,\mathbb{E}[\epsilon_e]
+
a_{(2+k)2}\,\mathbb{E}[Y\,\epsilon_e]
+
\displaystyle\sum_{j=1}^k a_{(2+k),(j+2)}\,\mathbb{E}[X_j\,\epsilon_e]
\end{pmatrix}.
\]
Since
\(\mathbf{e}_2^\top = (0,\;1,\;0,\ldots,0)\),
it follows that
\[
\mu
=
\mathbf{e}_2^\top\bigl(A_0^{-1}\,b_0\bigr)
=
a_{21}\,\mathbb{E}[\epsilon_e]
+
a_{22}\,\mathbb{E}[Y\,\epsilon_e]
+
\displaystyle\sum_{j=1}^k a_{2,(j+2)}\,\mathbb{E}[X_j\,\epsilon_e].
\]

From Remark \ref{rem:extractiobbias}, we know:
\[
\mathrm{Cov}(T, \epsilon_e) = \beta\, \mathrm{Cov}(Y, \epsilon_e) + \gamma^\top\, \mathrm{Cov}(X, \epsilon_e) \neq 0.
\]
This implies that at least one of \(\mathrm{Cov}(Y, \epsilon_e)\) or one of the elements of \(\mathrm{Cov}(X, \epsilon_e)\) must be nonzero. Consequently, at least one of \(\mathbb{E}[Y\,\epsilon_e]\) or one of the \(\mathbb{E}[X_j\,\epsilon_e]\) must also be nonzero. Therefore,
\[
\mu = a_{21}\,\mathbb{E}[\epsilon_e] + a_{22}\,\mathbb{E}[Y\,\epsilon_e] + \sum_{j=1}^k a_{2,(j+2)}\,\mathbb{E}[X_j\,\epsilon_e] \neq 0.
\]
Thus, \(\hat{\beta}\) is asymptotically biased.
\end{proof}

In summary, in a partially-specified model where \(\hat{T}_{ml}\) includes extraction errors correlated with \(Y\), the OLS estimator \(\hat{\beta}\) becomes biased. This bias stems from the inherent correlation between the predictor \(Y\) and the error term \(\epsilon_e\). Consequently, while the estimator remains consistent in its variance under the CLT, its expectation deviates from the true parameter \(\beta\).

\subsection{Omitted Variable Bias}
\label{sssec:OmittedBias}
This bias can occur when the variable of interest \(Y\) is correlated with the information not captured in $T$. Recall that \(Z^{(-T)}\) represents the components of \(Z\) that are not captured by the extracted feature \(T\). To quantify this bias, we first start by characterizing the true relationship between the variable of interest \( Y \), the extracted feature \( T \), and the omitted components \( Z^{(-T)} \) as $T = f_2(Y, k(Z^{(-T)}); \beta) + \epsilon_0$, where \( k(Z^{(-T)}) \) is a non-parametric function that captures the potentially complex effects of the omitted variables \( Z^{(-T)} \). In contrast, the econometric model that researchers estimate (as described earlier) is given by ${T} = f_{2}(Y; \beta) + \epsilon$. The key challenge is that the error term \( \epsilon \) is composed of both the original noise \( \epsilon_0 \) and the bias due to the exclusion of \( k(Z^{(-T)}) \).

\begin{remark}\label{rem:omittedbias}
Omitted Variable Bias (\(\epsilon_{ov}\)) occurs when the error term includes the effect of omitted variables that are correlated with the independent variable \(Y\). Formally, the error term can be expressed as:
\[
\epsilon = \epsilon_0 + \epsilon_{ov} \quad \text{where}  \quad \text{Cov}(Y, \epsilon_{ov})) \neq 0
\]
\end{remark}

\begin{proposition} \textbf{Omitted Variable Bias:}
Consider the linear model for \( T \):
\begin{equation}
\begin{aligned}
T &= \beta Y + \delta k(Z^{(-T)}) + \epsilon_0 \quad \text{(True Model)}, \\
T &= \tilde{\beta} Y + \epsilon \quad \text{(Observed Model)},
\end{aligned}
\end{equation}
where \( \epsilon = \epsilon_0 + \epsilon_{ov} \) and \( \epsilon_{ov} = \delta k(Z^{(-T)}) \) represents the impact of the omitted variable. Then, the bias in the estimator \( \tilde{\beta} \) is given by:
\begin{equation}
\text{Bias}(\hat{\tilde{\beta}}) = \tilde{\beta} - \beta = \lambda \delta,
\end{equation}
where \( \lambda \) is the coefficient from the projection of \( Y \) on \( k(Z^{(-T)}) \) based on the Frisch-Waugh-Lovell Theorem \citep{hansen2022econometrics}.
\end{proposition}
\begin{proof}
We start by considering the true model:
\[
T = {\beta} Y + {\delta} k(Z^{(-T)}) + \epsilon_0
\]
and the estimated model, which omits the variable \( k(Z^{(-T)}) \):
\[
T = \tilde{\beta} Y + \epsilon
\]
where \( \epsilon = \epsilon_0 + \epsilon_{ov} \) and \( \epsilon_{ov} = {\delta} k(Z^{(-T)}) \). The OLS estimator for \( \tilde{\beta} \) is given by:
\[
\hat{\tilde{\beta}} = \frac{\text{Cov}(Y, T)}{\text{Var}(Y)}
\]

Substituting the true model into this equation, we have:
\[
\text{Cov}(Y, T) = \text{Cov}\left(Y, {\beta} Y + {\delta} k(Z^{(-T)}) + \epsilon_0\right)
\]

Expanding the covariance:
\[
\text{Cov}(Y, T) = {\beta} \text{Cov}(Y, Y) + {\delta} \text{Cov}(Y, k(Z^{(-T)})) + \text{Cov}(Y, \epsilon_0)
\]

Assuming \( \text{Cov}(Y, \epsilon_0) = 0 \), we get:
\[
\text{Cov}(Y, T) = {\beta} \text{Var}(Y) + {\delta} \text{Cov}(Y, k(Z^{(-T)}))
\]

Thus, the OLS estimator becomes:
\[
\hat{\tilde{\beta}} = \frac{{\beta} \text{Var}(Y) + {\delta} \text{Cov}(Y, k(Z^{(-T)}))}{\text{Var}(Y)}
\]

The expected value of \( \hat{\tilde{\beta}} \) is therefore:
\[
\mathbb{E}[\hat{\tilde{\beta}}] = \frac{{\beta} \text{Var}(Y) + {\delta} \text{Cov}(Y, k(Z^{(-T)}))}{\text{Var}(Y)} = {\beta} + \frac{{\delta} \text{Cov}(Y, k(Z^{(-T)}))}{\text{Var}(Y)}
\]

This simplifies to:

\[
\text{Bias}(\hat{\tilde{\beta}}) = \frac{{\delta} \text{Cov}(Y, k(Z^{(-T)}))}{\text{Var}(Y)}
\]

where $\frac{\text{Cov}(Y, k(Z^{(-T)}))}{\text{Var}(Y)} = \lambda$ is the coefficient from the projection of \( Y \) on \( k(Z^{(-T)}) \).
    
\end{proof}

\section{Empirical Evidence} 
\label{appsec:empericalevidence}

In the previous section, we theoretically characterized how the two-step approach based on feature extraction results in information loss that can appear in forms of extraction bias and omitted variable bias. In this section, we empirically examine if the sources of bias discussed in $\S$\ref{ssec:drawbacks} can bias the estimates of the two-step reduced form model outlined in $\S$\ref{ssec:Two_Step_Model}. We focus on visual representations of Hillary Clinton as the focal politician $p$, where \( y_1 \) corresponds to images published by Democratic-leaning news outlets such as \textit{CNN}, \textit{The New York Times}, and the \textit{BBC}, and \( y_2 \) corresponds to those from Republican-leaning outlets like \textit{Fox News}, \textit{Newsmax}, and \textit{The Daily Mail}. Our goal is to estimate the parameter $\beta$ as characterized in Equation \eqref{eq:reducedform}. For illustrative purposes, we focus on the subset of our data containing images of Hillary Clinton (comprising 1,021 articles). 

Following the two-step approach, we first employ the DeepFace model developed by Meta as the machine learning algorithm \( f_1 \) \citep{taigman2014deepface}. We extract the presence (\( Z(\hat{T} = 1) \)) or lack of a smile (\( Z(\hat{T} = 0) \)) in the images of Hilary Clinton using DeepFace.\footnote{DeepFace utilizes a deep neural network with a nine-layer architecture and over 120 million parameters, including convolutional layers and a 3D alignment step. This model has been adapted for tasks such as facial expression and gender recognition in a variety of social science contexts \citep{dvorkin2021sovereign, luca2022scapegoating}.} In the second step, we apply a logistic regression model \( f_2 \) to quantify the relationship between the estimated facial expression \( \hat{T} \) and the independent variable \( Y \) as:
\begin{equation*}
\textrm{log} \left( \frac{\textrm{Pr}(\hat{T} = 1 \mid Y)}{\textrm{Pr}(\hat{T} = 0 \mid Y)} \right) = \alpha + \beta \cdot Y,
\end{equation*}
where \( Y = 1 \) represents Democratic outlets (\textit{CNN}, \textit{The New York Times}, or \textit{BBC}) and \( Y = 0 \) corresponds to Republican outlets ({\it Fox News}, {\it Newsmax}, or \textit{The Daily Mail}). The coefficients \(\alpha\) (intercept) and \(\beta\) (slope) are estimated from the data. For simplicity, this example excludes article-level features \( X \).

However, the above specification ignores all the non-smile-related information in the image (represented by \( Z^{(-T)} \)). As discussed earlier, doing so can lead to biased estimates of $\beta$. To that end, we now consider a modified specification that accounts for a few other image characteristics. For illustrative purposes, we select three arbitrary image features: edge density, brightness, and contrast, which capture a few relevant aspects of the visual context. Thus, the extended specification of the logistic regression model becomes:
\begin{equation*}
\textrm{log} \left( \frac{\textrm{Pr}(\hat{T} = 1 \mid Y, Z^{(-T)})}{\textrm{Pr}(\hat{T} = 0 \mid Y, Z^{(-T)})} \right) = \alpha + \beta \cdot Y + \gamma_1 \cdot \text{EdgeDensity} + \gamma_2 \cdot \text{Brightness} + \gamma_3 \cdot \text{Contrast},
\end{equation*}
where \( Z^{(-T)} \) includes the additional image characteristics, and \( \gamma_1 \), \( \gamma_2 \), and \( \gamma_3 \) are the coefficients that measure the effects of edge density, brightness, and contrast, respectively.

The results from both specifications are shown in Table \ref{tab:logit_results}. Column (1) presents the estimates for the model without other image characteristics. Here, the \textit{Democratic Outlets} variable is statistically significant (\( p < 0.01 \)), with a positive coefficient of \( 0.369 \), indicating that Democratic news outlets are more likely to show Hillary Clinton smiling as a democratic politician. However, when we include other image characteristics in Column (2), the coefficient for \textit{Democratic Outlets} changes to \( 0.278 \) and loses statistical significance (\( p < 0.10 \)). This reduction in both magnitude and significance suggests that part of the association initially attributed to \textit{Democratic Outlets} in Column (1) was actually capturing differences in how images were presented across outlets. In Column (3), we show that the change in the coefficient for Democratic Outlets comes from the fact that the variable Democratic Outlets is correlated with some of the three image features excluded from the model in Column (1). It is important to emphasize that the correlation structure between the Democratic Outlet dummy and three arbitrary images only highlight the potential for the bias in the estimates from the reduced form model, and the correct estimate that can be obtained with including all relevant controls can have any directional relationship with the estimates in Column (1). 

In sum, the empirical evidence confirms the existence of extraction and/or omitted variable bias and highlights the crucial role it plays in inference when social scientists use reduced-form models. Even in this simple illustrative example, we see that failing to account for correlated image characteristics can lead to biased estimates and incorrect conclusions about media bias and political polarization. Finally, it is crucial to emphasize that edge density, brightness, and contrast are just three arbitrary image characteristics selected from the much broader set of potential contextual features \( Z^{(-T)} \). The visual content of an image encompasses a wide range of information, and the choice of these particular features is meant to demonstrate that even a limited selection of characteristics can significantly impact the results.

\begin{table}[t]
\begin{center}
\small{
\begin{tabular}[h]{lccc}
\\[-1.8ex]\hline
\hline \\[-1.8ex]
\multirow{2}{*}{Variable} & \multicolumn{2}{c}{\textit{Dependent Variable: $\hat{T}$ (Happiness) }} & \textit{Dependent Variable: $Y$} \\
\cline{2-4} \\
[-1.8ex] &  (1)   &  (2)  &  (3)   \\
\hline \\[-1.8ex]
Intercept  & -1.170$^{***}$ & -1.809$^{***}$ &  -1.604$^{***}$\\
 & (0.104) & (0.257) &  (0.235)\\[0.1cm]
Democratic Outlets ($Y$) & 0.369$^{**}$ & 0.278$^{\text{†}}$ &  \\
 & (0.141) & (0.145) &  \\[0.1cm]
Brightness Level &  & 0.356  & 0.325\\
 &  & (0.561)  & (0.512)\\[0.1cm]
Edge Density &  & 0.878$^{\text{†}}$ & 2.312$^{***}$\\
 &  & (0.477) &  (0.456)\\[0.1cm]
Contrast &  & 0.537 & 1.450$^{**}$\\
 &  & (0.561) &  (0.513)\\[0.1cm]
\hline \\[-1.8ex]
No. of Obs. & 1021 & 1021  & 1021 \\
\textit{Pseudo-}$R^2$ & 0.005 & 0.013 &  0.040\\
Log-Likelihood & -595.4 & -591.3 &  -679.2\\
Likelihood Ratio Test & 6.87$^{**}$ & 14.97$^{**}$ &  56.94$^{***}$\\
\hline \\[-1.8ex]
\textit{Note:}  & \multicolumn{3}{r}{$^{\text{†}}$p$<$0.10; $^{*}$p$<$0.05; $^{**}$p$<$0.01; $^{***}$p$<$0.001} \\
\hline\hline
\end{tabular}
\caption{Results for logistic regressions. Columns 1-3 show results where $T$ is the binary dependent variable (happiness), while column 4 shows results where $Y$ (news outlet side) is the dependent variable where Republican outlets ({\it Fox News}, {\it Newsmax}, {\it Daily Mail}) are coded as $Y=0$ and Democratic outlets ({\it The New York Times}, {\it BBC}, {\it CNN}) are coded as $Y=1$. Image characteristics include brightness level, edge density, and contrast.}
\label{tab:logit_results}}
\end{center}
\end{table}

\section{Appendix: Implementation Details}\label{appsec:DetailedMLModel}

\subsection{Multi-Modal ML Implementation}\label{appssec:MLImplementation}

First, to fully exploit the presence of similar events, we use the contextual information in articles. The contextual information comprises textual data \( X^{\text{text}} \) from article titles, publication date \( X^{\text{date}} \), politicians' names \( P \), and political affiliations \( P^\text{aff} \), and image data \( Z \). We want our model architecture to fully capture the clustering structure of the contextual information for two reasons. First, identifying clusters with diverse news outlets allows for identification of visual factors related to polarization. Second, identifying clusters with a single or only a few news outlets helps identify stylistic preferences of outlets, preventing those factors to play a role in determining the polarization parameter. 

Figure \ref{fig:ML_Model} in the main text of the paper offers a visual overview of the model. To extract meaningful patterns in textual data \(  X^{\text{text}} \), we utilize Latent Dirichlet Allocation (LDA), which models the text data by identifying underlying topics across the articles. Although methods like BERT \citep{devlin2018bert} are known for their semantic understanding, LDA outperformed these models in our specific prediction task, particularly in identifying and quantifying topic distributions. Each article is thus represented as a 40-dimensional vector \( F_{\text{LDA}} \), encapsulating the relevance of different topics within the text. The LDA processing steps, including pre-processing and model training, are detailed in Web Appendix $\S$\ref{appssec:TextProcessing}. Categorical data, including the publication date \( X^{\text{date}} \), politicians' names \( P \), and political affiliations \( P^\text{aff} \), provides additional context that enhances the model's ability to discern patterns relevant to news outlet classification. These features are represented using \textit{embedding layers}, where each categorical value is mapped to a dense vector representation. Specifically, we use an embedding size of 4 for the publication dates \( X^{\text{date}} \), 8 for the politicians' names \( P \), and 2 for political affiliations \( P^\text{aff} \). For the image input, \textit{ResNet-101} processes the entire image, leveraging its exceptional performance in general image recognition tasks to extract hierarchical and context-rich features from the broader visual content and scene information \citep{he2016deep}. To adapt these architectures to the specific requirements of our task, the last 10 layers of ResNet-101 and the last 5 layers of VGG-Face are fine-tuned, ensuring that both contextual and facial embeddings are optimized for our application.

The second challenge is related to correctly estimating the link between smile and news outlet prediction. We use \textit{MTCNN} architecture for face detection due to its ability to perform joint face detection and alignment with high accuracy, ensuring precise focus on facial regions \citep{zhang2016joint}. Detected faces are then passed through the \textit{VGG-Face} network, which is particularly well-suited for this task because it is pre-trained on facial expression data, making it highly effective at capturing facial attributes such as smiles \citep{parkhi2015deep}. We design this part of the architecture to ensure that the predictive model accounts for the information in the politicians' faces. Later in results, we show how adding this element to the structure allows the model to correctly identify the differences between the counterfactual image versions. 

In summary, the integration of modalities occurs through specialized attention mechanisms:
\squishlist
\item \textit{Chunk attention} is applied to the VGG-Face embeddings, combining them with categorical data (politicians’ names \( P \) and affiliations \( P^\text{aff} \)) to capture correlations between facial features and structured metadata related to the image. This fusion ensures the model can link specific facial attributes to political or identity-related information \citep{liang2024foundations} (see Web Appendix $\S$\ref{appsssec:face_analysis} for details). 
\item \textit{Attention Mechanism} processes the structured data (categorical features, LDA topics), learning to prioritize the most relevant metadata features for the classification task \citep{vaswani2017attention} (see Web Appendix $\S$\ref{appsssec:structured_analysis} for details). 
\item The embeddings from ResNet-101 that are passed through fully connected layers, enriched with contextual image information, are directly incorporated into the final representation (see Web Appendix $\S$\ref{appsssec:image_analysis} for details).
\squishend
The outputs from these attention mechanisms and embeddings are used together in the final classification layer, which predicts the target news outlet. This design ensures that facial details, image context, and structured data interact effectively, enhancing the model’s performance in discerning patterns across modalities. 

After concatenating the feature vectors, the output from this layer is then passed through a final softmax layer to produce the classification prediction. The complete model details are explained in Web Appendix $\S$\ref{appssec:ModelArchitecture}. The model is trained using the AdamW optimizer \citep{loshchilov2017decoupled}, which is particularly suited for large-scale data due to its adaptive learning rate and regularization through weight decay. The loss function used is weighted cross-entropy, which measures the discrepancy between the predicted probabilities and the true labels, optimizing the model to improve classification accuracy. The training process's specifics, including the dataset's division, hyperparameter tuning, and regularization techniques, are discussed in Web Appendix $\S$\ref{appssec:ModelOptimization}.

With the model architecture defined, we can then maximize the following entropy over the training data:
\begin{equation}
\textbf{Entropy:} \quad \max_{\theta} \mathcal{H}(\theta) = \max_{\theta} \left( - \sum_{i=1}^{\mathcal{N}} \sum_{y \in \mathcal{Y}} g(Y_i = y \mid Z_i, X_i, P_i; \theta) \log g(Y_i = y \mid Z_i, X_i, P_i; \theta) \right) 
\end{equation}
The resulting model is what we use to estimate the polarization parameter using Equation \eqref{eq:analog2}.

\subsubsection{Textual Data Processing using LDA}\label{appssec:TextProcessing}
Textual data \( X^{\text{text}} \) is pivotal as it provides contextual and content information from the articles. We evaluated several approaches for modeling this data, including BERT (Bidirectional Encoder Representations from Transformers) \citep{devlin2018bert} and LDA (Latent Dirichlet Allocation) \citep{blei2003latent}. Despite BERT's advanced capabilities in understanding context and semantics, LDA solely demonstrated superior performance in our prediction task, particularly in identifying topics and their distributions across the articles.

LDA is a generative statistical model that explains sets of observations by unobserved groups, revealing why certain parts of the data are similar. The model assumes that each document is a mixture of a small number of topics and that each word in the document is attributable to one of the document's topics. The hyperparameters \( \alpha \) and \( \beta \) are set to 'auto' to allow the model to learn these parameters during training. Mathematically, LDA posits the following generative process for a corpus \( D \) consisting of \( M \) documents, each containing \( N \) words:

\begin{algorithm}[H]
\caption{LDA Generative Process \citep{blei2003latent}}
\begin{algorithmic}[1]
\For{each topic \( k \) in \( \{1, \ldots, K\} \)}
    \State Draw a distribution over words \( \phi_k \sim \text{Dir}(\beta) \)
\EndFor
\For{each document \( d \) in \( \{1, \ldots, M\} \)}
    \State Draw a distribution over topics \( \theta_d \sim \text{Dir}(\alpha) \)
    \For{each word \( n \) in \( \{1, \ldots, N_d\} \)}
        \State Draw a topic \( z_{dn} \sim \text{Multinomial}(\theta_d) \)
        \State Draw a word \( w_{dn} \sim \text{Multinomial}(\phi_{z_{dn}}) \)
    \EndFor
\EndFor
\end{algorithmic}
\end{algorithm}

Here, \( \alpha \) and \( \beta \) are hyperparameters of the Dirichlet distributions, \( \theta_d \) is the topic distribution for document \( d \), \( \phi_k \) is the word distribution for topic \( k \), \( z_{dn} \) is the topic assignment for the \( n \)-th word in document \( d \), and \( w_{dn} \) is the \( n \)-th word in document \( d \). In our model, we set the number of topics \( K \) to 40, and the model is trained with 40 passes over the corpus to ensure topic extraction.

The preprocessing steps for the textual data include cleaning and tokenizing the news titles, followed by lemmatization and stopword removal using NLTK's WordNetLemmatizer and stopwords list \citep{bird2009natural}. The processed tokens are then used to create bigrams and trigrams using Gensim's Phrases model \citep{rehurek_lrec}. These n-grams help capture contextual relationships between words, making the textual data more informative for the LDA model.

The trained LDA model generates topic distributions for each document. For each document \( d \), LDA provides the distribution \( \theta_d \) over topics, which can be interpreted as the document's composition in terms of latent topics. These distributions are then scaled and normalized for input into the neural network. This transformation converts textual data into a structured format encapsulating underlying topics and their relevance to each document. As a result, we have a 40-dimensional vector representation for each news article title.

\subsubsection{Face Information Branch: Framework and Operations}\label{appsssec:face_analysis}
The Face Analysis branch of the proposed architecture is designed to extract, refine, and integrate facial information using a multi-step process that combines advanced detection, feature extraction, dimensionality reduction, and attention mechanisms. This section details each operation, starting from face detection via the MTCNN, progressing to feature extraction using Face-VGG, and culminating with chunked attention for integrating structured metadata. Given an input image \( \mathbf{I} \in \mathbb{R}^{H \times W \times C} \), the MTCNN algorithm detects faces and aligns them for further processing. The detection process is divided into three stages. In the first stage, a proposal network is used to generate candidate bounding boxes:
\[
\mathbf{R}_1 = f_1(\mathbf{I}; \boldsymbol{\theta}_1),
\]
where \( \mathbf{R}_1 \) denotes the set of candidate bounding boxes, and \( \boldsymbol{\theta}_1 \) are the parameters of the proposal network. These candidate regions are refined in the second stage using a refine network:
\[
\mathbf{R}_2 = f_2(\mathbf{R}_1; \boldsymbol{\theta}_2),
\]
where \( \mathbf{R}_2 \) represents the refined bounding boxes. Finally, the output network predicts the bounding boxes and facial landmarks:
\[
(\mathbf{R}_\text{final}, \mathbf{L}) = f_3(\mathbf{R}_2; \boldsymbol{\theta}_3),
\]
where \( \mathbf{L} \in \mathbb{R}^{5 \times 2} \) are the coordinates of the facial landmarks. Using \( \mathbf{R}_\text{final} \) and \( \mathbf{L} \), the aligned face crop \( \mathbf{F} \in \mathbb{R}^{160 \times 160 \times 3} \) is extracted from the original image. Each aligned face crop \( \mathbf{F} \) is passed through a pre-trained Face-VGG network to extract a 512-dimensional feature vector. Let \( g_\text{VGG}(\cdot; \boldsymbol{\Theta}_\text{VGG}) \) denote the Face-VGG model, where \( \boldsymbol{\Theta}_\text{VGG} \) represents its parameters. The extracted feature vector \( \mathbf{v} \in \mathbb{R}^{512} \) is computed as:
\[
\mathbf{v} = g_\text{VGG}(\mathbf{F}; \boldsymbol{\Theta}_\text{VGG}).
\]
The last five layers of the Face-VGG model are fine-tuned to adapt to the specific task. The extracted feature vector \( \mathbf{v} \) is projected into a lower-dimensional space using a fully connected layer. The reduced feature vector \( \mathbf{v}_\text{red} \in \mathbb{R}^{256} \) is given by:
\[
\mathbf{v}_\text{red} = \mathbf{W}_\text{red} \mathbf{v} + \mathbf{b}_\text{red},
\]
where \( \mathbf{W}_\text{red} \in \mathbb{R}^{256 \times 512} \) and \( \mathbf{b}_\text{red} \in \mathbb{R}^{256} \) are learnable parameters of the layer. To improve generalization, batch normalization and dropout are applied to \( \mathbf{v}_\text{red} \). The normalized and regularized feature vector \( \mathbf{v}_\text{norm} \) is computed as:
\[
\mathbf{v}_\text{norm} = \text{BatchNorm}(\mathbf{v}_\text{red}),
\]
and the final output after dropout is:
\[
\mathbf{v}_\text{drop} = \text{Dropout}(\mathbf{v}_\text{norm}, p=0.3).
\]
The reduced feature vector \( \mathbf{v}_\text{drop} \in \mathbb{R}^{256} \) is divided into \( k = 8 \) equally sized chunks, each of dimension \( 32 \):
\[
\mathbf{v}_\text{chunk} = \{\mathbf{v}_i \in \mathbb{R}^{32} \mid i = 1, \dots, k\}, \quad \mathbf{v}_i = \mathbf{v}_\text{drop}[32(i-1):32i].
\]
This chunking operation allows the model to analyze localized regions of the image feature vector \citep{fu2019dual}. To incorporate structured metadata, such as person and party embeddings, the chunked features \( \mathbf{v}_\text{chunk} \) interact with a metadata embedding \( \mathbf{m} \in \mathbb{R}^{32} \). The metadata embedding is computed as:
\[
\mathbf{m} = \mathbf{W}_\text{meta} \mathbf{x} + \mathbf{b}_\text{meta},
\]
where \( \mathbf{x} \) represents the concatenated embeddings of person and party, \( \mathbf{W}_\text{meta} \in \mathbb{R}^{32 \times d_\text{meta}} \), and \( \mathbf{b}_\text{meta} \in \mathbb{R}^{32} \) are learnable parameters. An attention mechanism assigns weights \( \alpha_i \) to each chunk \( \mathbf{v}_i \) based on its similarity to the metadata embedding \( \mathbf{m} \):
\[
\alpha_i = \frac{\exp(\mathbf{v}_i^\top \mathbf{m})}{\sum_{j=1}^k \exp(\mathbf{v}_j^\top \mathbf{m})}.
\]
The attention-weighted feature representation \( \mathbf{v}_\text{att} \in \mathbb{R}^{256} \) is computed as:
\[
\mathbf{v}_\text{att} = \sum_{i=1}^k \alpha_i \mathbf{v}_i.
\]
The output of the Face Analysis branch \( \mathbf{v}_\text{att} \) is concatenated with features from the ResNet branch \( \mathbf{f}_\text{ResNet} \) and the metadata branch \( \mathbf{f}_\text{meta} \) to form the final aggregated feature vector:
\[
\mathbf{z} = [\mathbf{v}_\text{att}; \mathbf{f}_\text{ResNet}; \mathbf{f}_\text{meta}].
\]
This feature vector \( \mathbf{z} \in \mathbb{R}^{d_\text{final}} \) is passed through subsequent layers for classification.

\subsubsection{Structured Metadata Branch: Framework and Operations}\label{appsssec:structured_analysis}

The Structured Data Branch in the proposed architecture processes tabular metadata, including embeddings for categorical features (e.g., person, party, and date) and topic distributions derived from LDA. The key component of this branch is the multi-head attention mechanism, inspired by the Transformer model \citep{vaswani2017attention}, which dynamically highlights relevant interactions among structured features. This section provides the mathematical foundations and operations in a coherent manner.

Let the structured input be represented as \( \mathbf{x} \in \mathbb{R}^{d_\text{struct}} \), where \( d_\text{struct} = 54 \), obtained by concatenating the embeddings of categorical features and topic distributions. The embeddings for person, party, and date are denoted as \( P \in \mathbb{R}^{d_\text{person}} \), \( P^\text{aff} \in \mathbb{R}^{d_\text{party}} \), and \( X^{\text{date}}  \in \mathbb{R}^{d_\text{date}} \), respectively, while the topic vector is \( X^{\text{text}} \in \mathbb{R}^{d_\text{topics}} \), where \( d_\text{text} = 40 \). The structured input vector \( \mathbf{x} \) is thus formed as:

\[
\mathbf{x} = [P; P^\text{aff}; X^{\text{date}}; X^{\text{text}}],
\]
where \( [\cdot] \) denotes vector concatenation. This input vector is then projected into a latent representation space of dimension \( d_\text{latent} = 64 \) via a linear transformation:

\[
\mathbf{h} = \mathbf{W}_\text{linear} \mathbf{x} + \mathbf{b}_\text{linear},
\]
where \( \mathbf{W}_\text{linear} \in \mathbb{R}^{d_\text{latent} \times d_\text{struct}} \) and \( \mathbf{b}_\text{linear} \in \mathbb{R}^{d_\text{latent}} \) are learnable weights and biases. The resulting vector \( \mathbf{h} \in \mathbb{R}^{d_\text{latent}} \) serves as the input to the multi-head attention mechanism. The multi-head attention mechanism begins by projecting \( \mathbf{h} \) into three distinct subspaces corresponding to the query (\( \mathbf{Q} \)), key (\( \mathbf{K} \)), and value (\( \mathbf{V} \)) representations \citep{vaswani2017attention}. These projections are computed as:

\[
\mathbf{Q} = \mathbf{W}_Q \mathbf{h}, \quad 
\mathbf{K} = \mathbf{W}_K \mathbf{h}, \quad 
\mathbf{V} = \mathbf{W}_V \mathbf{h},
\]
where \( \mathbf{W}_Q, \mathbf{W}_K, \mathbf{W}_V \in \mathbb{R}^{d_\text{head} \times d_\text{latent}} \) are learnable weight matrices, and \( d_\text{head} \) represents the dimensionality of each attention head. The scaled dot-product attention mechanism computes the compatibility between queries and keys, resulting in an attention score matrix:

\[
\mathbf{A} = \text{softmax}\left( \frac{\mathbf{Q} \mathbf{K}^\top}{\sqrt{d_\text{head}}} \right),
\]
where \( \text{softmax}(\cdot) \) normalizes the scores across all keys for each query, and the scaling factor \( \sqrt{d_\text{head}} \) stabilizes gradients during training by mitigating the effect of large dot-product magnitudes. The attention scores \( \mathbf{A} \in \mathbb{R}^{d_\text{head} \times d_\text{head}} \) modulate the value vectors \( \mathbf{V} \), yielding the attention output:

\[
\mathbf{O}_\text{single} = \mathbf{A} \mathbf{V}.
\]

To capture diverse interactions among the structured features, multiple attention heads are employed. For each attention head \( i \), the process is repeated independently, producing:

\[
\mathbf{O}_i = \mathbf{A}_i \mathbf{V}_i = \text{softmax}\left( \frac{\mathbf{Q}_i \mathbf{K}_i^\top}{\sqrt{d_\text{head}}} \right) \mathbf{V}_i,
\]
where \( i = 1, \dots, h \) and \( h \) is the number of heads. The outputs from all attention heads are concatenated and projected back into the latent representation space using a linear transformation:

\[
\mathbf{O}_\text{multi} = \mathbf{W}_\text{out} \left[ \mathbf{O}_1; \mathbf{O}_2; \dots; \mathbf{O}_h \right],
\]
where \( \mathbf{W}_\text{out} \in \mathbb{R}^{d_\text{latent} \times (h \cdot d_\text{head})} \) is a learnable projection matrix. The multi-head attention output \( \mathbf{O}_\text{multi} \in \mathbb{R}^{d_\text{latent}} \) encodes a rich and dynamic representation of the structured data. To ensure stability and prevent overfitting, the output \( \mathbf{O}_\text{multi} \) is normalized using batch normalization and regularized with dropout. These operations are defined as:

\[
\mathbf{h}_\text{norm} = \text{BatchNorm}(\mathbf{O}_\text{multi}), \quad 
\mathbf{h}_\text{drop} = \text{Dropout}(\mathbf{h}_\text{norm}, p),
\]
where \( p \) is the dropout probability. The final representation of the structured data branch is \( \mathbf{h}_\text{struct} = \mathbf{h}_\text{drop} \), a 64-dimensional vector ready to be fused with the outputs from other branches in the model. The multi-head attention mechanism allows the model to dynamically prioritize different aspects of the structured data by assigning relevance scores based on the compatibility of queries and keys. The queries (\( \mathbf{Q} \)) represent the model’s current focus, the keys (\( \mathbf{K} \)) encode contextual information about all features, and the values (\( \mathbf{V} \)) provide the corresponding information content. By using multiple attention heads, the model captures diverse patterns and relationships, resulting in a robust and task-specific representation of the structured metadata.

\subsubsection{Image Contextual Information Branch: Framework and Operations}\label{appsssec:image_analysis}

The ResNet101 branch is responsible for extracting high-dimensional feature representations from the full raw image data. Let \( \mathbf{I} \in \mathbb{R}^{H \times W \times C} \) represent an input image, where \( H = 244 \) and \( W = 244 \) are the height and width, and \( C \) is the number of color channels. The input image passes through a pre-trained ResNet101 network, with the last 10 layers trainable. Denote the ResNet101 transformation as \( g_\text{ResNet}(\cdot; \boldsymbol{\Theta}_\text{ResNet}) \), where \( \boldsymbol{\Theta}_\text{ResNet} \) are the trainable weights of the final layers. The output feature map \( \mathbf{F}_\text{ResNet} \in \mathbb{R}^{d_\text{ResNet}} \) is given by:

\[
\mathbf{F}_\text{ResNet} = g_\text{ResNet}(\mathbf{I}; \boldsymbol{\Theta}_\text{ResNet}),
\]
where \( d_\text{ResNet} = 2048 \) is the dimensionality of the extracted feature vector. This high-dimensional feature vector is flattened and passed through a linear layer for dimensionality reduction:

\[
\mathbf{f}_\text{ResNet} = \mathbf{W}_\text{reduce} \mathbf{F}_\text{ResNet} + \mathbf{b}_\text{reduce},
\]
where \( \mathbf{W}_\text{reduce} \in \mathbb{R}^{d_\text{reduce} \times d_\text{ResNet}} \), \( \mathbf{b}_\text{reduce} \in \mathbb{R}^{d_\text{reduce}} \), and \( d_\text{reduce} = 512 \). Batch normalization and dropout are applied to enhance generalization:

\[
\mathbf{f}_\text{norm} = \text{BatchNorm}(\mathbf{f}_\text{ResNet}), \quad
\mathbf{f}_\text{drop} = \text{Dropout}(\mathbf{f}_\text{norm}, p),
\]
where \( p = 0.6 \) is the dropout probability. The final output of the ResNet101 branch is \( \mathbf{f}_\text{drop} \in \mathbb{R}^{512} \), which is fused with outputs from other branches for classification.

\subsubsection{Final Model Architecture}\label{appssec:ModelArchitecture}
The final architecture of the multi-modal model is summarized in Table \ref{tab:improved_multimodal_architecture}. This table provides an overview of each layer in the model.
\begin{table}[t]
\centering
\small
\begin{tabular}{llr}
\hline
\textbf{Layer (type)} & \textbf{Output Shape} & \textbf{Param \#} \\ \hline

\multicolumn{3}{c}{\textbf{Branch 1: Face Information}} \\ \hline
MTCNN (Face Detection) & [-1, 160, 160, 3] & 0 \\
Face-VGG (Last 5 layers trainable) & [-1, 512] & 23,560,896 \\
Linear (Face-VGG to 256) & [-1, 256] & 131,328 \\
BatchNorm1d (Face-VGG normalization) & [-1, 256] & 512 \\
Dropout (Face-VGG, 0.3) & [-1, 256] & 0 \\
Chunked VGG Features (8 chunks of 32) & [-1, 8, 32] & 0 \\
Linear (person + party to 32) & [-1, 1, 32] & 384 \\
Chunked Attention (Chunked VGG attention with person + party) & [-1, 256] & 65,792 \\

\multicolumn{3}{c}{\textbf{Branch 2: Image Contextual Information}} \\ \hline

ResNet101 (Last 10 layers trainable) & [-1, 2048] & 44,549,160 \\
Flatten (ResNet feature vector) & [-1, 2048] & 0 \\
Linear (ResNet to 512) & [-1, 512] & 1,049,088 \\
BatchNorm1d (ResNet normalization) & [-1, 512] & 1,024 \\
Dropout (ResNet, 0.6) & [-1, 512] & 0 \\

\multicolumn{3}{c}{\textbf{Branch 3: Structured Metadata}} \\ \hline

Embedding (person: 8, party: 2, date: 4) & [-1, 14] & 0 \\
Latent Dirichlet Allocation (LDA topics: 40) & [-1, 40] & 0 \\
Concatenation (person, party, date, topics) & [-1, 54] & 0 \\
Linear (Structured data to 64) & [-1, 64] & 3,520 \\
MultiHeadAttention (Structured data attention) & [-1, 64] & 16,512 \\
BatchNorm1d (Structured data normalization) & [-1, 64] & 128 \\
Dropout (Structured data, 0.4) & [-1, 64] & 0 \\ \hline

\multicolumn{3}{c}{\textbf{Final: Concatenation}} \\ \hline
Concatenation (VGG + Structured + ResNet) & [-1, 832] & 0 \\
Dropout (Final dropout, 0.5) & [-1, 832] & 0 \\
Linear (Final classification layer) & [-1, num\_classes] & 8,320 \\ \hline
\multicolumn{2}{l}{Total params} & 69,398,560 \\
\multicolumn{2}{l}{Trainable params} & 10,580,820 \\
\multicolumn{2}{l}{Non-trainable params} & 58,817,740 \\
\hline
\end{tabular}
\caption{Architecture of the Multi-Modal Machine Learning Model}
\label{tab:improved_multimodal_architecture}
\end{table}

The outputs from the three branches—Face Analysis (\( \mathbf{f}_\text{face} \in \mathbb{R}^{d_\text{face}} \)), Structured Data (\( \mathbf{h}_\text{struct} \in \mathbb{R}^{d_\text{struct}} \)), and ResNet101 (\( \mathbf{f}_\text{drop} \in \mathbb{R}^{d_\text{reduce}} \))—are concatenated to form a unified feature vector:
\[
\mathbf{z} = [\mathbf{f}_\text{face}; \mathbf{h}_\text{struct}; \mathbf{f}_\text{drop}] \in \mathbb{R}^{d_\text{final}},
\]
where \( d_\text{final} = d_\text{face} + d_\text{struct} + d_\text{reduce} \). In this case, \( d_\text{face} = 256 \), \( d_\text{struct} = 64 \), and \( d_\text{reduce} = 512 \), resulting in \( d_\text{final} = 832 \). To enhance generalization and mitigate overfitting, dropout is applied to the concatenated representation:

\[
\mathbf{z}_\text{drop} = \text{Dropout}(\mathbf{z}, p),
\]
where \( p = 0.5 \) is the dropout rate. The resulting vector \( \mathbf{z}_\text{drop} \in \mathbb{R}^{d_\text{final}} \) is passed through a fully connected classification layer to compute the logits for the \( n_\text{class} \) target classes:

\[
\mathbf{y} = \mathbf{W}_\text{class} \mathbf{z}_\text{drop} + \mathbf{b}_\text{class},
\]
where \( \mathbf{W}_\text{class} \in \mathbb{R}^{n_\text{class} \times d_\text{final}} \) and \( \mathbf{b}_\text{class} \in \mathbb{R}^{n_\text{class}} \) are the weights and biases of the classification layer. The predicted class probabilities are obtained by applying the softmax function to the logits:

\[
\hat{\mathbf{y}}_i = \frac{\exp(y_i)}{\sum_{j=1}^{n_\text{class}} \exp(y_j)}, \quad i = 1, \dots, n_\text{class},
\]
where \( y_i \) is the \( i \)-th component of \( \mathbf{y} \), and \( \hat{\mathbf{y}}_i \) represents the probability of the \( i \)-th class. The final output \( \hat{\mathbf{y}} \in \mathbb{R}^{n_\text{class}} \) is a normalized probability distribution over the target classes, enabling robust multi-class predictions.

\subsubsection{Model Training and Fine-Tuning}
\label{appssec:ModelOptimization}
To ensure robust model evaluation, the dataset is split into training (85\%) and test (15\%) sets. The splitting process preserves the distribution of key features by stratifying based on both news center and date, ensuring balanced records across all classes. The training set is used to optimize the model parameters, while the test data set is used for assessing model performance and polarization measurement to avoid overfitting bias. 

The training process employs the AdamW optimizer \citep{loshchilov2017decoupled}, configured with a learning rate \( \eta = 0.0001 \) and a weight decay coefficient \( \lambda = 0.01 \). This optimizer is particularly well-suited for large-scale data due to its adaptive learning rates and decoupled weight decay, effectively balancing convergence speed and regularization. The loss function used for training is a weighted cross-entropy loss, adjusted to account for class imbalance and enhanced with label smoothing. Let \( \mathcal{D} = \{(\mathbf{x}_i, \mathbf{y}_i)\}_{i=1}^N \) represent a batch of \( N \) training samples, where \( \mathbf{y}_i \in \{0, 1\}^C \) is the one-hot encoded true label for the \( i \)-th sample, and \( C = 20 \) denotes the number of classes. The predicted class probabilities are given by \( \hat{\mathbf{y}}_i = \text{softmax}(\mathbf{z}_i) \), where \( \mathbf{z}_i \in \mathbb{R}^C \) is the logit vector for the \( i \)-th sample. The weighted cross-entropy loss with label smoothing is defined as:

\[
\mathcal{L} = -\frac{1}{N} \sum_{i=1}^{N} \sum_{c=1}^{C} w_c \, \tilde{y}_{i,c} \log \hat{y}_{i,c},
\]
where \( w_c \) is the weight for class \( c \), computed as the inverse of its relative frequency in the training set to mitigate the impact of class imbalance:

\[
w_c = \frac{1}{\text{freq}(c)} \cdot \frac{\sum_{j=1}^C \text{freq}(j)}{C}.
\]
The smoothed label \( \tilde{y}_{i,c} \) for class \( c \) is defined as:

\[
\tilde{y}_{i,c} = (1 - \alpha) y_{i,c} + \frac{\alpha}{C},
\]
where \( \alpha = 0.05 \) is the label smoothing coefficient. Label smoothing redistributes a small fraction of the ground-truth probability mass to all classes, reducing overconfidence in predictions and improving generalization \citep{szegedy2016rethinking}. The optimizer updates the model parameters \( \boldsymbol{\theta} \) at each iteration by minimizing the total loss \( \mathcal{L} \). The parameter updates follow the AdamW rule:

\[
\boldsymbol{\theta}_{t+1} = \boldsymbol{\theta}_t - \eta \frac{\hat{\mathbf{m}}_t}{\sqrt{\hat{\mathbf{v}}_t} + \epsilon} - \lambda \boldsymbol{\theta}_t,
\]
where \( \hat{\mathbf{m}}_t \) and \( \hat{\mathbf{v}}_t \) are the bias-corrected first and second moments of the gradient, respectively, and \( \epsilon \) is a small constant to ensure numerical stability. The algorithm for parameter updates is formally presented below:

\begin{algorithm}[H]
\caption{AdamW Optimizer}
\begin{algorithmic}[1]
\State \textbf{Input:} Learning rate \( \eta \), decay rates \( \beta_1, \beta_2 \), weight decay coefficient \( \lambda \), small constant \( \epsilon \)
\State \textbf{Initialize:} \( \mathbf{m}_0 = 0 \), \( \mathbf{v}_0 = 0 \), \( t = 0 \)
\For{each iteration}
    \State \( t = t + 1 \)
    \State Compute gradients of the loss: \( \mathbf{g}_t = \nabla_{\boldsymbol{\theta}} \mathcal{L}(\boldsymbol{\theta}_{t-1}) \)
    \State Update biased first moment estimate:
    \[
    \mathbf{m}_t = \beta_1 \mathbf{m}_{t-1} + (1 - \beta_1) \mathbf{g}_t
    \]
    \State Update biased second moment estimate:
    \[
    \mathbf{v}_t = \beta_2 \mathbf{v}_{t-1} + (1 - \beta_2) \mathbf{g}_t^{\odot 2}
    \]
    \State Compute bias-corrected first moment estimate:
    \[
    \hat{\mathbf{m}}_t = \frac{\mathbf{m}_t}{1 - \beta_1^t}
    \]
    \State Compute bias-corrected second moment estimate:
    \[
    \hat{\mathbf{v}}_t = \frac{\mathbf{v}_t}{1 - \beta_2^t}
    \]
    \State Update parameters:
    \[
    \boldsymbol{\theta}_t = \boldsymbol{\theta}_{t-1} - \eta \frac{\hat{\mathbf{m}}_t}{\sqrt{\hat{\mathbf{v}}_t} + \epsilon} - \lambda \boldsymbol{\theta}_{t-1}
    \]
\EndFor
\end{algorithmic}
\end{algorithm}

The optimization process continues over 30 epochs, where each epoch consists of processing mini-batches of data of size 64. At each iteration, the model updates its parameters using the computed gradients and the AdamW update rule. The weight decay term \( \lambda \boldsymbol{\theta}_{t-1} \) ensures regularization by penalizing large weights, helping to improve generalization. Metrics such as validation loss and classification accuracy are monitored at the end of each epoch to assess the model’s performance and ensure generalization to unseen data.

\subsection{Validating Generated Counterfactual Images} \label{appssec:GeneratingCounterfactual}

In this part, we explain generating counterfactual images. First, for each politician \( p \), we select three images \( \tilde{Z}^{0}_{p1}, \tilde{Z}^{0}_{p2}, \tilde{Z}^{0}_{p3} \) that are not used to train our multi-modal ML model, sum up to 84 images. To ensure unbiasedness, these images are chosen to represent the politician in a neutral manner, avoiding any inherent slant. Consequently, this selected image set inherently represents the neutral version \( z^0 \). Additionally, selecting three images helps mitigate any specific biases that might arise from using a single image and ensures that our analysis is robust across different visual contexts.

Using the approach proposed by \cite{mirza2014conditional}, we generate a smiley version of each image while keeping other factors constant. In other words, we apply the transformation \( \pi^{1} \) on images \( \tilde{Z}^{0}_{p1}, \tilde{Z}^{0}_{p2}, \tilde{Z}^{0}_{p3} \) to obtain \( \pi^{1}(\tilde{Z}^{0}_{p1}) = \tilde{Z}^{1}_{p1}, \pi^{1}(\tilde{Z}^{0}_{p2}) = \tilde{Z}^{1}_{p2}, \pi^{1}(\tilde{Z}^{0}_{p3}) = \tilde{Z}^{1}_{p3} \). Each \( \tilde{Z}^{1} \) now embodies the slant notion, specifically a smile, as shown in Figure \ref{AddingSmile} in the main text of the paper. 

Since the isolated change in smile and facial expression is crucial, we further assess the generated images to see if other image characteristics have changed. The generated photos should have no other significant differences. As discussed, one of the challenges with the two-stage approach is that image characteristics such as brightness or colorfulness can cause bias in the model (see $\S$\ref{ssec:drawbacks}). Therefore, we measure the brightness and colorfulness of images before and after applying the smile operator \( \pi^{1} \). The histograms of these characteristics are shown in Figure \ref{fig:image_characteristics}:

\begin{figure}[H]
    \centering
    \includegraphics[width=\linewidth]{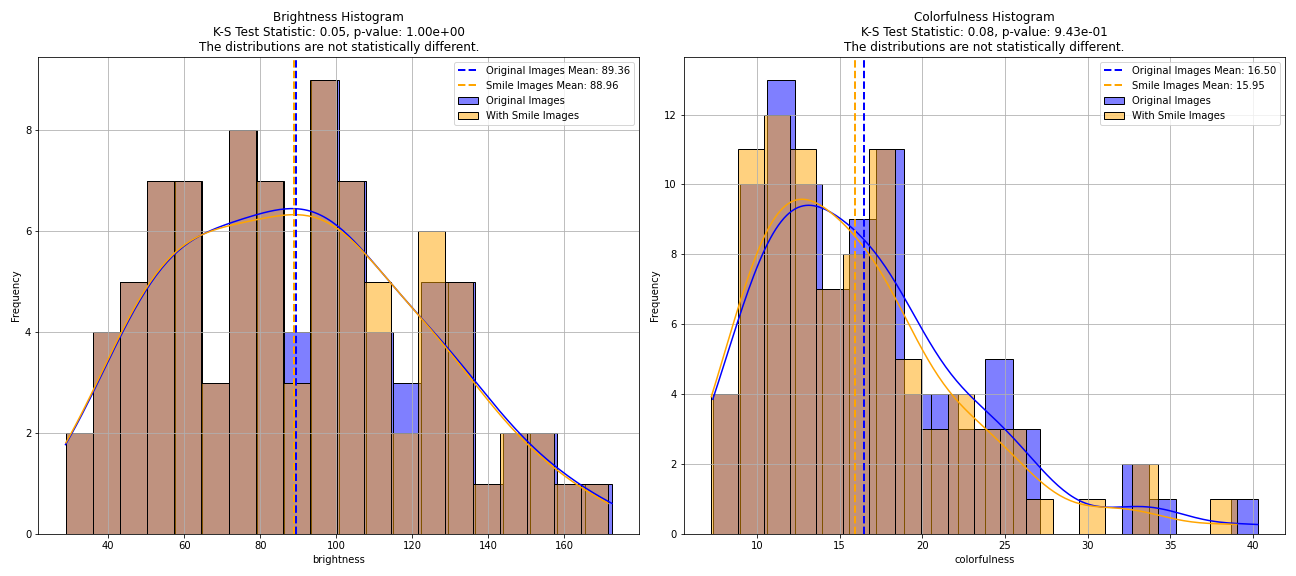}
    \caption{\small Histograms of brightness and colorfulness for original and smiley images.}
    \label{fig:image_characteristics}
\end{figure}

For brightness, the mean value for the original images is 89.36, while for the smiley images, it is 88.96. The Kolmogorov-Smirnov (K-S) test statistic is 0.05 with a p-value of $\approx$1.00, indicating that the distributions are not statistically different. This suggests that the introduction of a smile does not significantly alter the brightness of the images. Similarly, for colorfulness, the mean value for the original images is 16.50, compared to 15.95 for the smiley images. The K-S test statistic here is 0.08 with a p-value of 0.943, again indicating no statistical difference between the distributions.

This analysis confirms that the generated smiley images do not exhibit significant changes in brightness and colorfulness compared to the original images. This supports the validity of using these counterfactual images to isolate the effect of the smile on the parameter of interest, without introducing additional biases from other image characteristics.

\section{Supplementary Results}

\subsection{Detailed Multi-Modal ML Performance}
\label{appssec:result_pred}
In this section, we detail the performance of our multi-modal multi-class classification problem for news outlet prediction. As mentioned earlier, we use an 85\%-15\% split between training and test data. We consider four measures to evaluate the model's predictive performance -- (1) Accuracy, (2) Precision, (3) Recall, and (4) Weighted Cross-Entropy (WCE). Accuracy is simply measured as:
\[
\text{Accuracy} = \frac{\sum_{i=1}^{N} \mathbf{1}(Y_i = \argmax_c \hat{Y}_{ic})}{N},
\]
where $\hat{Y}_{ic}$ is our predicted probability for news outlet $c$ producing article $i$. 

Precision evaluates the reliability of the model’s predictions by measuring the proportion of correctly predicted articles for a given outlet among all articles predicted for that outlet. Formally, it is defined as:
\[
\text{Precision} = \frac{\sum_{i=1}^{N} \mathbf{1}(Y_i = c \land \argmax_c \hat{Y}_{ic} = c)}{\sum_{i=1}^{N} \mathbf{1}(\argmax_c \hat{Y}_{ic} = c)},
\]
where \(\mathbf{1}(Y_i = c \land \argmax_c \hat{Y}_{ic} = c)\) counts the true positives (correct predictions for outlet \(c\)), and \(\mathbf{1}(\argmax_c \hat{Y}_{ic} = c)\) counts all predictions made for outlet \(c\). 

Recall, on the other hand, measures the model’s ability to identify all true articles for a given outlet. It is defined as:
\[
\text{Recall} = \frac{\sum_{i=1}^{N} \mathbf{1}(Y_i = c \land \argmax_c \hat{Y}_{ic} = c)}{\sum_{i=1}^{N} \mathbf{1}(Y_i = c)},
\]
where \(\mathbf{1}(Y_i = c)\) counts all true articles for outlet \(c\). While Precision emphasizes minimizing false positives, Recall ensures that false negatives are minimized, making them complementary metrics to assess the model’s performance comprehensively. 

Finally, WCE Loss is defined as:
\[
\text{WCE Loss} = -\frac{1}{N} \sum_{i=1}^{N} \sum_{c=1}^{C} w_c \, (Y_{ic} (1 - \epsilon) + \epsilon / C) \log(\hat{Y}_{ic}),
\]
where \( C \) is the number of classes, \( w_c \) represents the weight assigned to class \( c \) to handle class imbalance, and \( \epsilon \) is the label smoothing factor applied to soften the target labels, encouraging the model to focus more evenly across all classes.

Table \ref{tab:val_results_sorted_accuracy} provides information on our multi-modal model performance at each class level (news outlet) and presents the test performance of the model across various news sources, sorted by accuracy. The model performs best with \textit{Daily Mail} and shows strong, consistent results across other sources, such as \textit{Newsmax} and \textit{CNN}. The model effectively addresses class imbalances by employing a weighted cross-entropy loss, ensuring fair representation for both frequent and rare classes. This is reflected in the balanced precision, recall, and F1 scores across sources with different sample sizes. The weighted loss function helps prevent overfitting to dominant classes, resulting in reliable performance across various news outlets.

\begin{table}[htp!]
\centering
\small
\begin{tabular}{lccccc}
\hline\hline
\textbf{Source} & \textbf{Accuracy} & \textbf{Precision} & \textbf{Recall} & \textbf{F1-Score} & \textbf{No. of observations} \\
\hline\hline
{\it Daily Mail}      & 0.654 & 0.613 & 0.654 & 0.633 & 987 \\
{\it Newsmax}         & 0.609 & 0.511 & 0.609 & 0.556 & 241 \\
{\it LA Times}        & 0.576 & 0.629 & 0.576 & 0.601 & 151 \\
{\it Fox News}        & 0.493 & 0.509 & 0.493 & 0.501 & 577 \\
{\it ABC News}        & 0.489 & 0.314 & 0.489 & 0.381 & 703 \\
{\it CNN}             & 0.470 & 0.417 & 0.470 & 0.441 & 832 \\
{\it CNBC}            & 0.444 & 0.373 & 0.444 & 0.405 & 410 \\
{\it The New York Times}  & 0.440 & 0.380 & 0.440 & 0.408 & 426 \\
{\it Wall Street Journal} & 0.433 & 0.454 & 0.433 & 0.444 & 593 \\
{\it CS Monitor}      & 0.429 & 0.668 & 0.429 & 0.522 & 203 \\
{\it Time}            & 0.416 & 0.418 & 0.416 & 0.417 & 369 \\
{\it BBC}             & 0.403 & 0.636 & 0.403 & 0.492 & 554 \\
{\it NBC News}        & 0.394 & 0.448 & 0.394 & 0.419 & 937 \\
{\it Washington Post} & 0.373 & 0.316 & 0.373 & 0.342 & 567 \\
{\it CBS News}        & 0.297 & 0.463 & 0.297 & 0.360 & 517 \\
{\it Reuters}         & 0.295 & 0.382 & 0.295 & 0.332 & 233 \\
{\it HuffPost}        & 0.285 & 0.316 & 0.285 & 0.300 & 321 \\
{\it US News}         & 0.249 & 0.411 & 0.249 & 0.309 & 109 \\
{\it USA Today}       & 0.234 & 0.303 & 0.234 & 0.264 & 425 \\
{\it Yahoo News}      & 0.196 & 0.202 & 0.196 & 0.199 & 296 \\
\hline\hline
\end{tabular}
\caption{Predictive accuracy results on the test set for different news outlets (sorted by accuracy)}
\label{tab:val_results_sorted_accuracy}
\end{table}

\subsection{Details of PCA and t-SNE for Reducing Image Dimensionality to Two}
\label{appssec:pcatsne}
PCA is used to reduce the dimensionality of the embeddings from $d$ to $d' \ll d$ by projecting them onto a lower-dimensional subspace that retains most of the variance in the data. Mathematically, PCA identifies a set of orthogonal components $\mathbf{w}_1, \mathbf{w}_2, \ldots, \mathbf{w}_{d'} \in \mathbb{R}^d$, where each component maximizes the variance of the projected data. The reduced embedding for an image is given by:
\[
\mathbf{e}_i^{\text{PCA}} = \mathbf{W}^\top \mathbf{e}_i, \quad \mathbf{W} = [\mathbf{w}_1, \mathbf{w}_2, \ldots, \mathbf{w}_{d'}],
\]
where $\mathbf{W}$ is the matrix of the top $d'$ eigenvectors of the covariance matrix of the embeddings, sorted by their corresponding eigenvalues. This transformation ensures that the majority of the information in the original embeddings is preserved in the reduced representation, making subsequent computations more efficient.

After reducing the embeddings to \( d' \) dimensions using PCA, we project them into a two-dimensional space, \( \mathbf{e}_i^{\text{t-SNE}} \in \mathbb{R}^2 \), using t-SNE for visualization. t-SNE, or t-Distributed Stochastic Neighbor Embedding, is a dimensionality reduction technique designed specifically for high-dimensional data visualization \citep{belkina2019automated}. It aims to preserve the local structure of the data by modeling the similarity between points \( i \) and \( j \) in the original high-dimensional space as probabilities \( P = \{p_{ij}\} \), and then finding a 2D embedding where the similarities, represented by \( Q = \{q_{ij}\} \), approximate \( P \). In this 2D space, the similarity between points \( i \) and \( j \) is modeled using a Student-t distribution, where the pairwise similarity is defined as:
\[
q_{ij} = \frac{\left(1 + \|\mathbf{e}_i^{\text{t-SNE}} - \mathbf{e}_j^{\text{t-SNE}}\|^2\right)^{-1}}{\sum_{k \neq l} \left(1 + \|\mathbf{e}_k^{\text{t-SNE}} - \mathbf{e}_l^{\text{t-SNE}}\|^2\right)^{-1}},
\]
where \(\mathbf{e}_i^{\text{t-SNE}}\) and \(\mathbf{e}_j^{\text{t-SNE}}\) represent the 2D coordinates of points \(i\) and \(j\), \(\|\mathbf{e}_i^{\text{t-SNE}} - \mathbf{e}_j^{\text{t-SNE}}\|\) is the Euclidean distance between points \(i\) and \(j\) in the 2D space, and \(k\) and \(l\) iterate over all points in the dataset to compute the denominator of the similarity measure for proper normalization. t-SNE attempts to arrange the 2D points such that their similarities, denoted as \(Q = \{q_{ij}\}\), approximate the similarities in the high-dimensional space, denoted as \(P = \{p_{ij}\}\). Here, \(P = \{p_{ij}\}\) represents the pairwise similarities computed from the PCA-reduced high-dimensional embeddings, and \(Q = \{q_{ij}\}\) represents the pairwise similarities in the 2D space computed using the formula above. The arrangement of points in 2D is achieved by minimizing the Kullback-Leibler (KL) divergence:
\[
\text{KL}(P \| Q) = \sum_{i \neq j} p_{ij} \log \frac{p_{ij}}{q_{ij}},
\]
where \(p_{ij}\) and \(q_{ij}\) denote the similarity between points \(i\) and \(j\) in the high-dimensional and 2D spaces, respectively. By minimizing this divergence, t-SNE ensures that points with high similarity in the high-dimensional space (large \(p_{ij}\)) remain close in the 2D space (large \(q_{ij}\)), while dissimilar points (small \(p_{ij}\)) are positioned farther apart. The resulting 2D embeddings, \(\mathbf{e}_i^{\text{t-SNE}}\), enable visualization of images' space,

To further analyze the patterns in the embedding space for a specific politician \( p \), we perform clustering on the reduced embeddings associated with their images. For each politician \( p \), let the set of embeddings corresponding to their images be denoted as \(\mathcal{E}^p = \{\mathbf{e}_i^{\text{PCA}} \,|\, i \in \mathcal{I}_p\}\), where \(\mathcal{I}_p\) is the set of image indices for politician \( p \). These embeddings are grouped into \( k = 20 \) clusters using K-Means clustering. The clustering objective for politician \( p \) is to minimize the within-cluster sum of squared distances:

\[
\mathcal{C}_i^p = \arg\min_{\mathcal{C}} \sum_{j=1}^{k} \sum_{\mathbf{e}_i^{\text{PCA}} \in \mathcal{C}_j^p} \|\mathbf{e}_i^{\text{PCA}} - \mathbf{\mu}_j^p\|^2,
\]

where \(\mathcal{C}_j^p\) is the set of embeddings assigned to cluster \( j \) for politician \( p \), and \(\mathbf{\mu}_j^p\) is the centroid of cluster \( j \) for \( p \).  Intuitively, each cluster represents a group of visually similar images, which may correspond to specific events or contextual features, such as similar camera angles and settings. For instance, images of politician \( p \) during a particular speech or event are likely to form a distinct cluster, capturing the shared visual context of those images. Now, for a politician $p$ and a given cluster $k$, we can investigate how different news outlets select images.

\subsection{Hypothesis Tests for the Distribution of Visual Polarization}
\label{appssec:hyptest}

We employ two statistical tests to analyze these differences statistically: the Kolmogorov-Smirnov (K-S) test and one-sample t-tests. The K-S test determines if two samples come from the same distribution. The null hypothesis \( H_0 \) states that the two samples are drawn from the same distribution. On the other hand, the one-sample t-test determines if the mean of a sample is different from a known value (zero in this case). The null hypothesis \( H_0 \) states that the sample mean is equal to the known value. A low p-value indicates that we can reject the null hypothesis. The results of these statistical tests are summarized in Table \ref{tab:statistical-results}.

\begin{table}[H]
    \centering
    \small
    \begin{tabular}{lcccc}
        \toprule
        \textbf{Politicians} & \textbf{Test} & \textbf{Statistic} & \textbf{P-Value} & \textbf{n} \\
        \midrule
        \multirow{3}{*}{Democratic Politicians} 
        & K-S Test & 0.526790 & 0.0000 & 31728 \\
        & Democratic Outlets Mean Test & 96.016984 & 0.0000 & 15864 \\
        & Republican Outlets Mean Test & -82.654572 & 0.0000 & 15864 \\
        \midrule
        \multirow{3}{*}{Republican Politicians} 
        & K-S Test & 0.672592 & 0.0000 & 24978 \\
        & Democratic Outlets Mean Test & -73.504981 & 0.0000 & 12489 \\
        & Republican Outlets Mean Test & 114.545837 & 0.0000 & 12489 \\
        \bottomrule
    \end{tabular}
    \caption{Summary of Statistical Test Results}
    \label{tab:statistical-results}
\end{table}

The results presented in Table \ref{tab:statistical-results} provide strong statistical evidence of visual slant in the portrayal of both Democratic and Republican politicians. For Democratic politicians, the statistical tests confirm a significant and positive visual slant for Democratic-leaning outlets, and a significant and negative visual slant for Republican-leaning outlets. These findings reflect a clear partisan divide, as shown by the significant test results.

For Republican politicians, the table demonstrates an opposing trend. The visual slant is significant and positive for Republican-leaning outlets, whereas it is significantly lower than zero for Democratic-leaning outlets. The statistical significance across all tests underscores the robustness of these patterns. Table \ref{tab:statistical-results} highlights how visual elements, such as smiles, are interpreted differently depending on the political context and outlet alignment, illustrating the media's role in amplifying partisan biases.

\subsection{Individual News Outlet Results}\label{appssec:News_Outlet_Level}

\begin{figure}[H]
    \centering
    \includegraphics[width=1\linewidth]{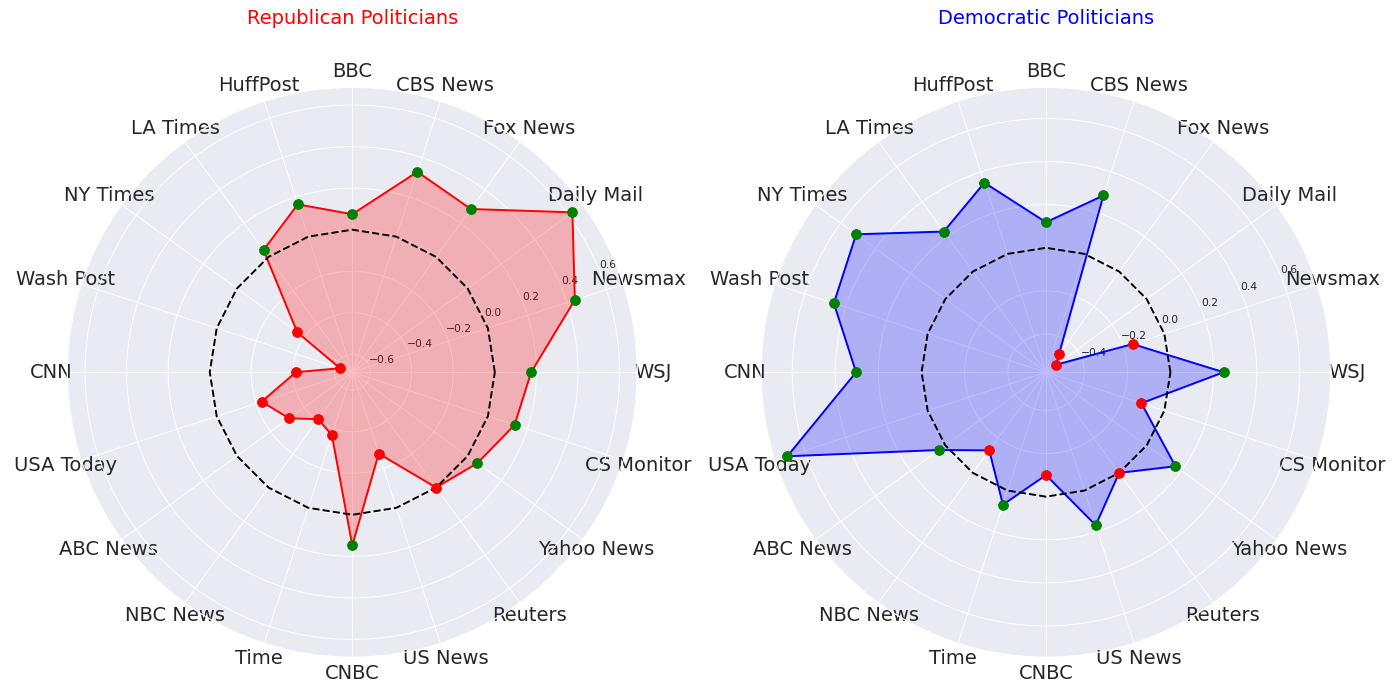}
    \caption{\small Radar plots showing the mean polarization measurement for Republican and Democratic politicians across various news outlets. {\it Reuters} polarization is zero since it serves as the baseline.}
    \label{fig:news_outlets_Radar}
\end{figure}

\begin{figure}[H]
    \centering
    \includegraphics[width=0.95\linewidth]{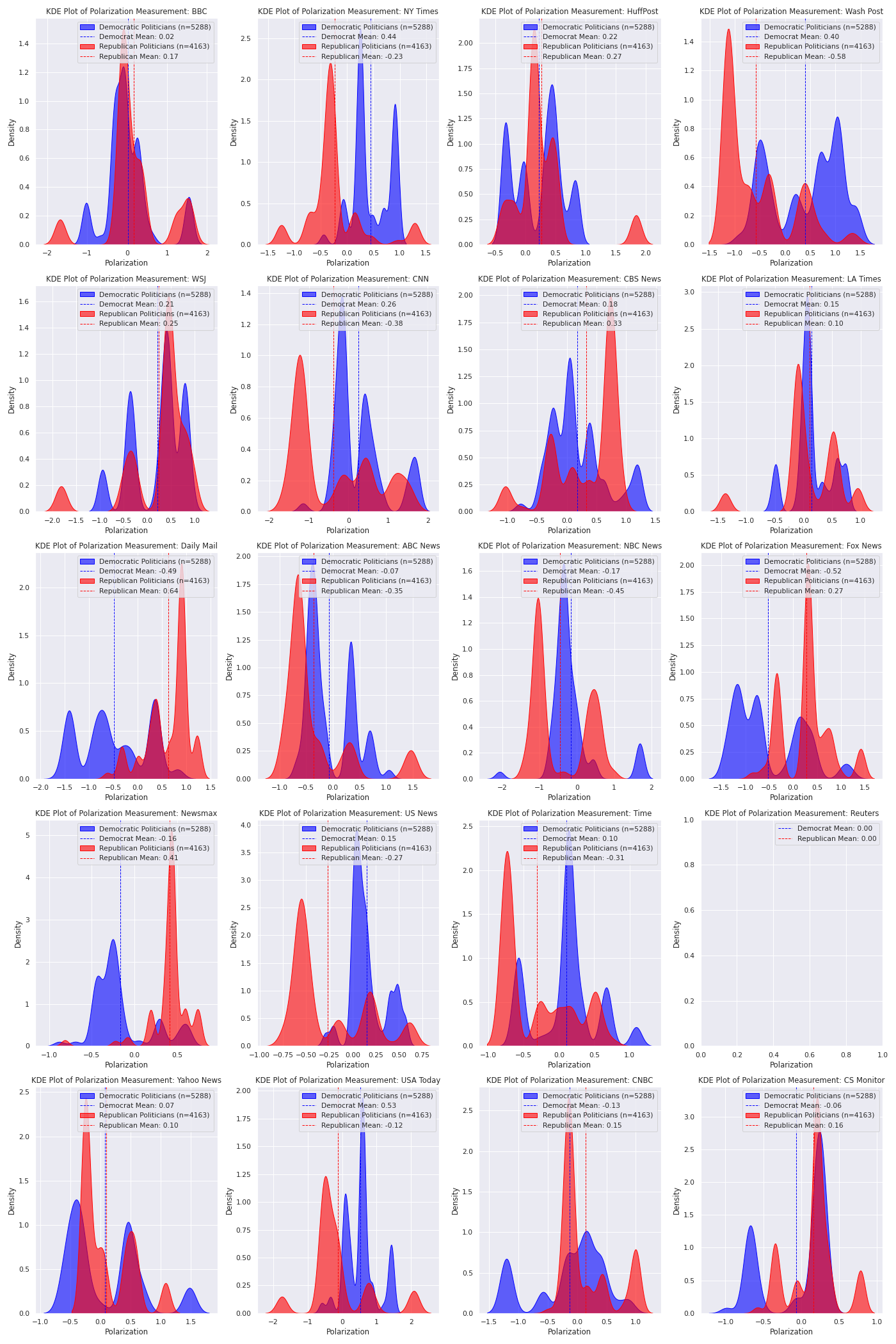}
    \caption{\small Histograms of polarization for Democratic and Republican Politicians in each News Outlet}
    \label{fig:indinidual_Hist}
\end{figure}

\begin{table}[H]
\centering
\begin{tabular}{lcccccccc}
\toprule
News Center & Pol. Side & Mean & Q1 & Q3 & T-Test Stat & T-Test P-Val & Significant? & Support (n) \\
\midrule
BBC & Dem & 0.02 & -0.32 & 0.63 & 2.89 & $3.90 \times 10^{-3}$ & True & 5288 \\
BBC & Rep & 0.17 & -0.11 & 1.53 & 14.87 & $8.47 \times 10^{-49}$ & True & 4163 \\
NY Times & Dem & 0.44 & -0.01 & 0.93 & 90.19 & $0.00$ & True & 5288 \\
NY Times & Rep & -0.23 & -0.71 & 0.44 & -27.63 & $1.97 \times 10^{-154}$ & True & 4163 \\
HuffPost & Dem & 0.22 & -0.34 & 0.77 & 41.33 & $0.00$ & True & 5288 \\
HuffPost & Rep & 0.27 & -0.19 & 0.49 & 36.33 & $2.87 \times 10^{-251}$ & True & 4163 \\
Wash Post & Dem & 0.40 & -0.52 & 1.05 & 42.34 & $0.00$ & True & 5288 \\
Wash Post & Rep & -0.58 & -1.13 & 0.49 & -55.35 & $0.00$ & True & 4163 \\
WSJ & Dem & 0.21 & -0.42 & 0.78 & 28.15 & $1.49 \times 10^{-162}$ & True & 5288 \\
WSJ & Rep & 0.25 & -0.48 & 0.77 & 24.34 & $1.93 \times 10^{-122}$ & True & 4163 \\
CNN & Dem & 0.26 & -0.36 & 1.51 & 29.14 & $2.81 \times 10^{-173}$ & True & 5288 \\
CNN & Rep & -0.38 & -1.22 & 1.15 & -24.19 & $4.70 \times 10^{-121}$ & True & 4163 \\
CBS News & Dem & 0.18 & -0.25 & 1.01 & 28.14 & $1.80 \times 10^{-162}$ & True & 5288 \\
CBS News & Rep & 0.33 & -0.29 & 0.73 & 39.12 & $2.11 \times 10^{-285}$ & True & 4163 \\
LA Times & Dem & 0.15 & -0.48 & 0.63 & 31.82 & $1.77 \times 10^{-203}$ & True & 5288 \\
LA Times & Rep & 0.10 & -0.10 & 0.56 & 13.27 & $2.21 \times 10^{-39}$ & True & 4163 \\
Daily Mail & Dem & -0.49 & -1.41 & 0.36 & -52.10 & $0.00$ & True & 5288 \\
Daily Mail & Rep & 0.64 & 0.00 & 0.90 & 94.11 & $0.00$ & True & 4163 \\
ABC News & Dem & -0.07 & -0.46 & 0.66 & -11.31 & $2.40 \times 10^{-29}$ & True & 5288 \\
ABC News & Rep & -0.35 & -0.81 & 0.34 & -36.04 & $8.04 \times 10^{-248}$ & True & 4163 \\
NBC News & Dem & -0.17 & -0.67 & 0.44 & -19.09 & $1.37 \times 10^{-78}$ & True & 5288 \\
NBC News & Rep & -0.45 & -1.05 & 0.61 & -38.75 & $1.06 \times 10^{-280}$ & True & 4163 \\
Fox News & Dem & -0.52 & -1.26 & 0.37 & -53.96 & $0.00$ & True & 5288 \\
Fox News & Rep & 0.27 & -0.35 & 0.78 & 35.00 & $1.74 \times 10^{-235}$ & True & 4163 \\
Newsmax & Dem & -0.16 & -0.45 & 0.51 & -35.51 & $5.97 \times 10^{-248}$ & True & 5288 \\
Newsmax & Rep & 0.41 & 0.19 & 0.69 & 115.30 & $0.00$ & True & 4163 \\
US News & Dem & 0.15 & 0.02 & 0.48 & 58.18 & $0.00$ & True & 5288 \\
US News & Rep & -0.27 & -0.56 & 0.20 & -42.57 & $0.00$ & True & 4163 \\
Time & Dem & 0.10 & -0.57 & 0.66 & 17.50 & $1.00 \times 10^{-66}$ & True & 5288 \\
Time & Rep & -0.31 & -0.74 & 0.49 & -40.21 & $5.77 \times 10^{-299}$ & True & 4163 \\
Yahoo News & Dem & 0.07 & -0.48 & 0.73 & 8.95 & $4.75 \times 10^{-19}$ & True & 5288 \\
Yahoo News & Rep & 0.10 & -0.24 & 0.58 & 15.71 & $4.78 \times 10^{-54}$ & True & 4163 \\
USA Today & Dem & 0.53 & 0.07 & 1.42 & 82.28 & $0.00$ & True & 5288 \\
USA Today & Rep & -0.12 & -0.51 & 0.93 & -9.03 & $2.64 \times 10^{-19}$ & True & 4163 \\
CNBC & Dem & -0.13 & -1.21 & 0.42 & -15.40 & $2.13 \times 10^{-52}$ & True & 5288 \\
CNBC & Rep & 0.15 & -0.23 & 0.99 & 21.14 & $2.73 \times 10^{-94}$ & True & 4163 \\
CS Monitor & Dem & -0.06 & -0.67 & 0.27 & -10.63 & $4.03 \times 10^{-26}$ & True & 5288 \\
CS Monitor & Rep & 0.16 & -0.34 & 0.73 & 33.06 & $4.96 \times 10^{-213}$ & True & 4163 \\
\bottomrule
\end{tabular}
\caption{Statistical analysis of polarization by news outlet and politician affiliation}
\label{tab:smile_effects}
\end{table}

\subsection{Correlation Analysis for Existing Partisanship Scores }
\label{appssec:MoreResults}
We extend our analysis to all 19 outlets, \( y^k \in \mathcal{Y} \), to evaluate whether polarization patterns observed in $\S$\ref{sssec:pol_dist} hold more broadly. Specifically, we calculate outlet-level {\it CVS} measurement for all outlets following Equation \eqref{equ:CVS}. To validate our measure, we compare {\it CVS} with existing benchmarks, focusing on the {\it conservative share score} by \citet{faris2017partisanship}, \cite{allsides2024}, \cite{flaxman2016filter}. Figures \ref{fig:reg_harvard}, \ref{fig:reg_allsides}, and \ref{fig:reg_flax} visualize these relationships alongside correlation analysis using Pearson and Spearman.
 
\begin{figure}[H]
    \centering
    \includegraphics[width=0.9\linewidth]{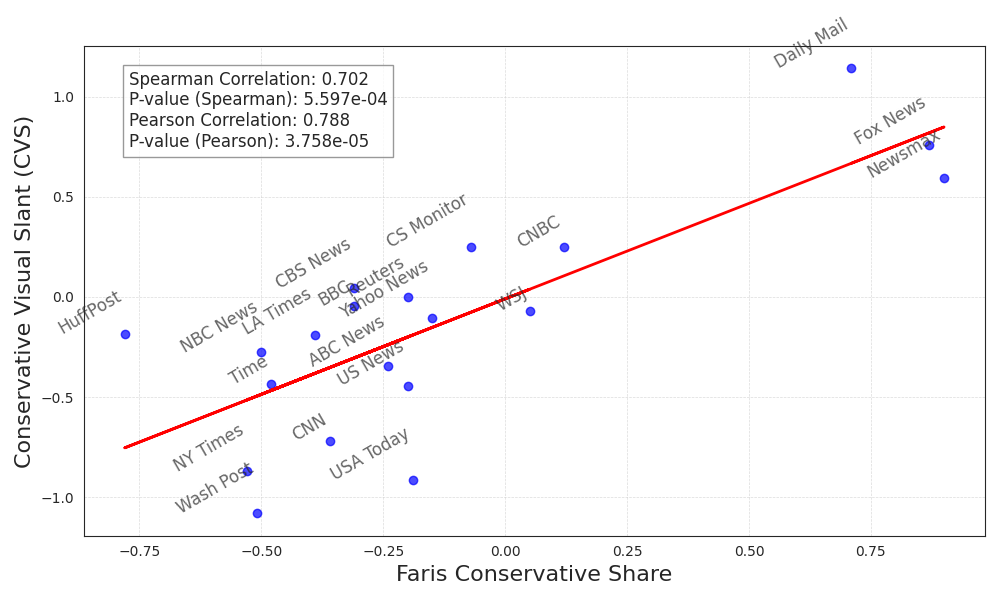}
    \caption{\small Correlation between {\it CVS} and the conservative share score of news outlets (based on \cite{faris2017partisanship}'s Conservative Score).}
    \label{fig:reg_harvard}
\end{figure}

\begin{figure}[H]
    \centering
    \includegraphics[width=0.9\linewidth]{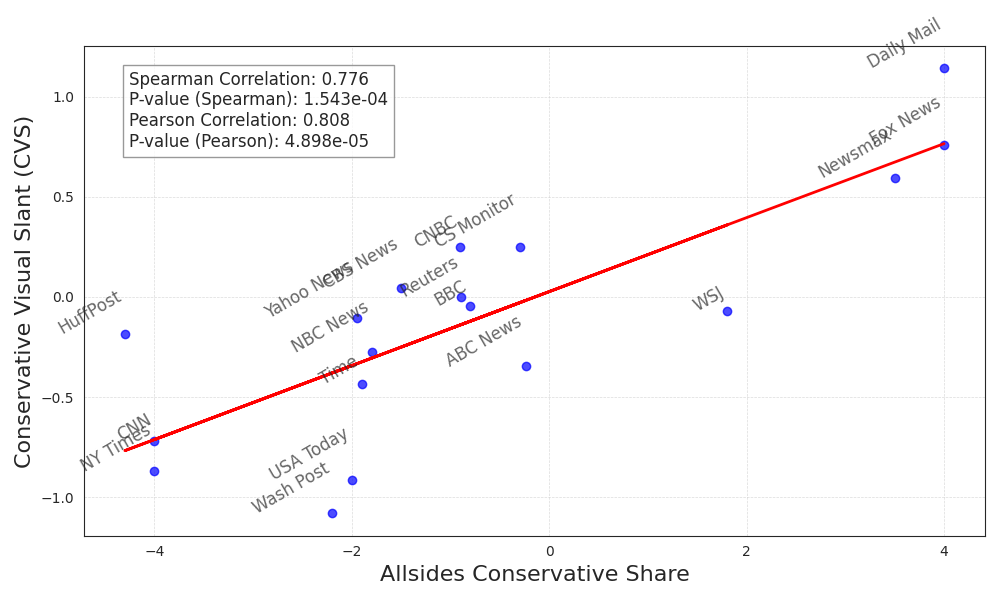}
    \caption{\small Correlation between {\it CVS} and the conservative share score of news outlets (based on \cite{allsides2024} Conservative Score).}
    \label{fig:reg_allsides}
\end{figure}

\begin{figure}[H]
    \centering
    \includegraphics[width=0.9\linewidth]{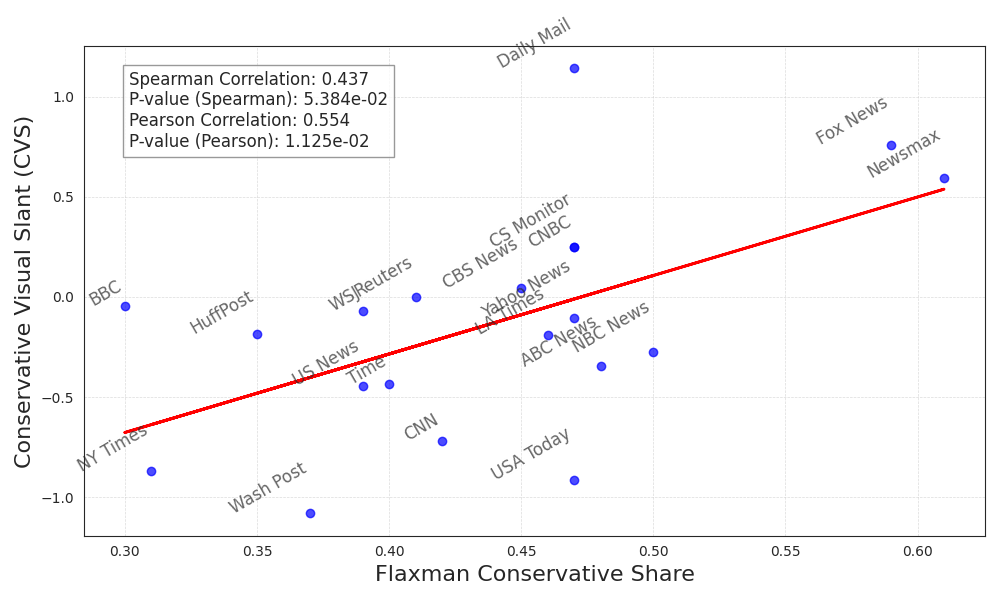}
    \caption{\small Correlation between {\it CVS} and the conservative share score of news outlets. (Based on \cite{flaxman2016filter} Conservative Score)}
    \label{fig:reg_flax}
\end{figure}

Across all three benchmarks, we observe a significant positive correlation between the conservative share scores and our {\it CVS} measure. The strength of correlation varies across sources, with \cite{allsides2024} showing the highest Spearman (\(\rho = 0.776\), \(p < 0.001\)) and Pearson (\(r = 0.808\), \(p < 0.001\)) correlations, followed by \cite{flaxman2016filter} (\(\rho = 0.702\), \(r = 0.788\)), and \cite{faris2017partisanship} showing the weakest but still statistically significant association (\(\rho = 0.437\), \(r = 0.554\)). These results support the validity of {\it CVS} as an independent measure of conservative visual slant, aligning well with established textual and audience-based conservative share scores.

\subsection{Comparison of Slant Visual from PMCIG and Two-stage Approach}
\label{appssec:comparison}

In this section, we present a comparison of our proposed method, PMCIG, with the established two-stage approach for analyzing slant in visual content. For this purpose, we collect the overall visual slant scores from \cite{boxell2021slanted}, which are calculated as the difference between the relative favorability towards Republicans minus Democrats (similar to {\it CVS} we present in Equation \eqref{equ:CVS}, but using the two-stage approach).

Our dataset and \cite{boxell2021slanted} share 11 common news outlets; therefore, the comparison focuses on these shared outlets. To evaluate the performance, we use both Pearson correlation analysis and Spearman rank correlation analysis. Each of these analyses is conducted by comparing the respective methods (PMCIG and Boxell) against three existing conservative share measures, providing an assessment of the alignment and ranking accuracy across approaches. Figures \ref{fig:Comparison_Harvard_vs_Boxell}, \ref{fig:Comparison_Allsides_vs_Boxell}, and \ref{fig:Comparison_OtherPaper_vs_Boxell} summarize our findings, providing statistical insights into the relative effectiveness of each method.

\begin{figure}[H]
    \centering
    \includegraphics[width=0.9\linewidth]{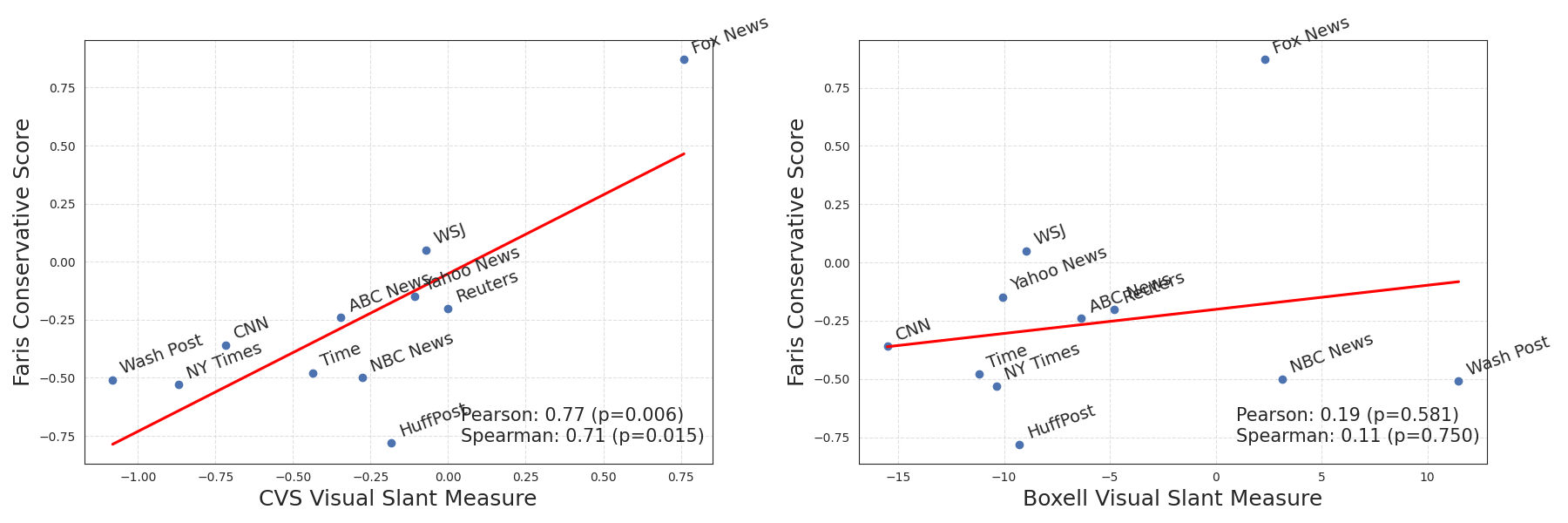}
    \caption{Comparison of visual slant from our PMCIG method (left) and \cite{boxell2021slanted} approach (right) against the conservative score from \cite{faris2017partisanship}.}
    \label{fig:Comparison_Harvard_vs_Boxell}
\end{figure}

\begin{figure}[H]
    \centering
    \includegraphics[width=0.9\linewidth]{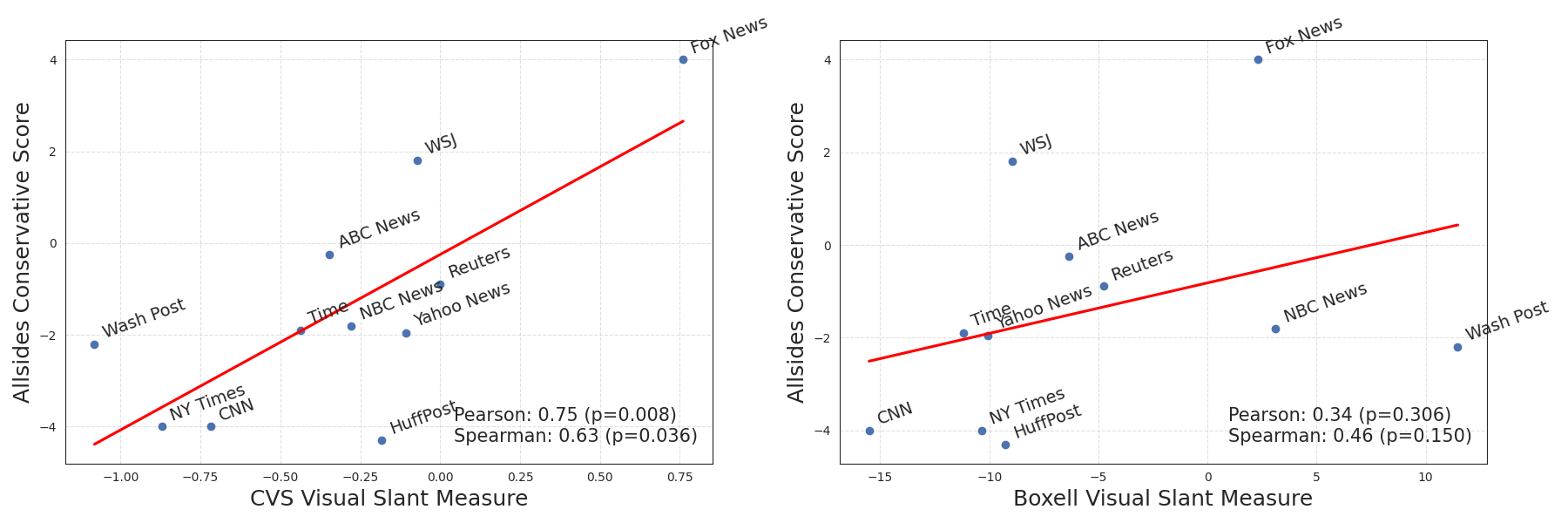}
    \caption{Comparison of visual slant from our PMCIG method (left) and \cite{boxell2021slanted} approach (right) against the conservative score from \cite{allsides2024}.}
    \label{fig:Comparison_Allsides_vs_Boxell}
\end{figure}

\begin{figure}[H]
    \centering
    \includegraphics[width=0.9\linewidth]{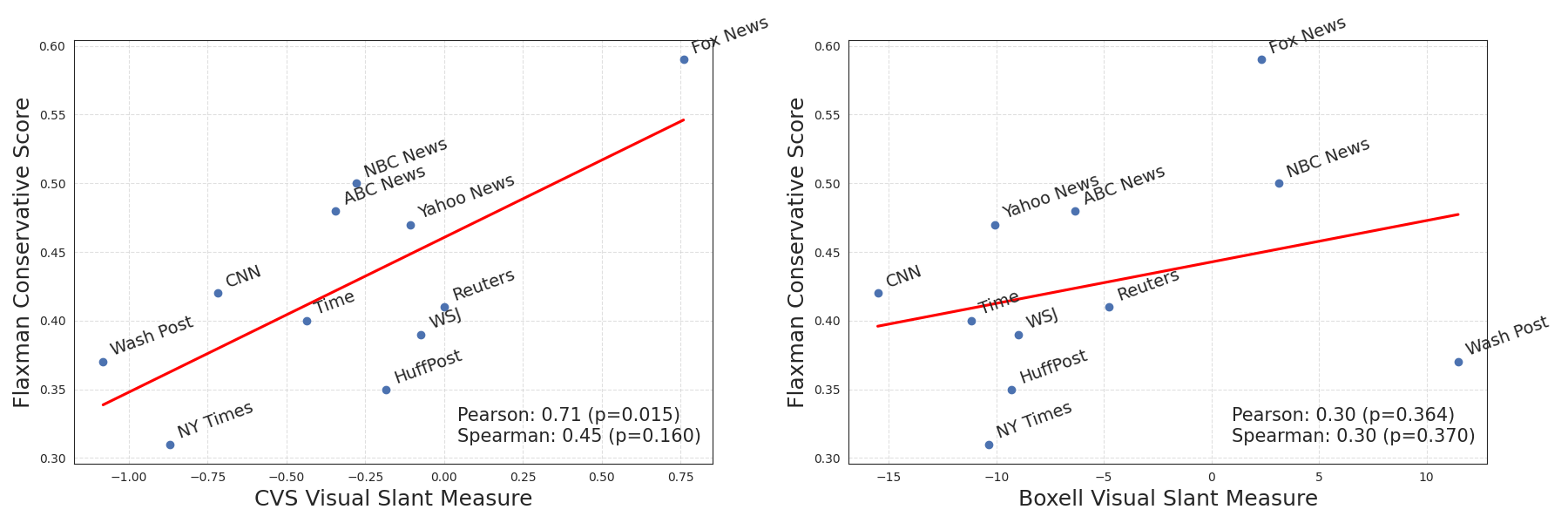}
    \caption{Comparison of visual slant from our PMCIG method (left) and \cite{boxell2021slanted} approach (right) against the conservative score from \cite{flaxman2016filter}.}
    \label{fig:Comparison_OtherPaper_vs_Boxell}
\end{figure}

The results clearly show that PMCIG outperforms \cite{boxell2021slanted} two-stage approach in measuring visual slant. Across all three benchmarks, PMCIG achieves higher Pearson and Spearman correlations, demonstrating stronger alignment with established conservative share measures. Additionally, PMCIG’s correlations are statistically significant, with lower p-values, whereas \cite{boxell2021slanted} approach often produces weaker and less reliable correlations. PMCIG also excels in ranking accuracy, as reflected in its consistently higher Spearman rank correlations, confirming that it orders news outlets' visual slant more accurately.

\subsection{Individual Politicians Results}\label{appssec:Politicians_Level}
\begin{figure}[H]
    \centering
    \includegraphics[width=1\linewidth]{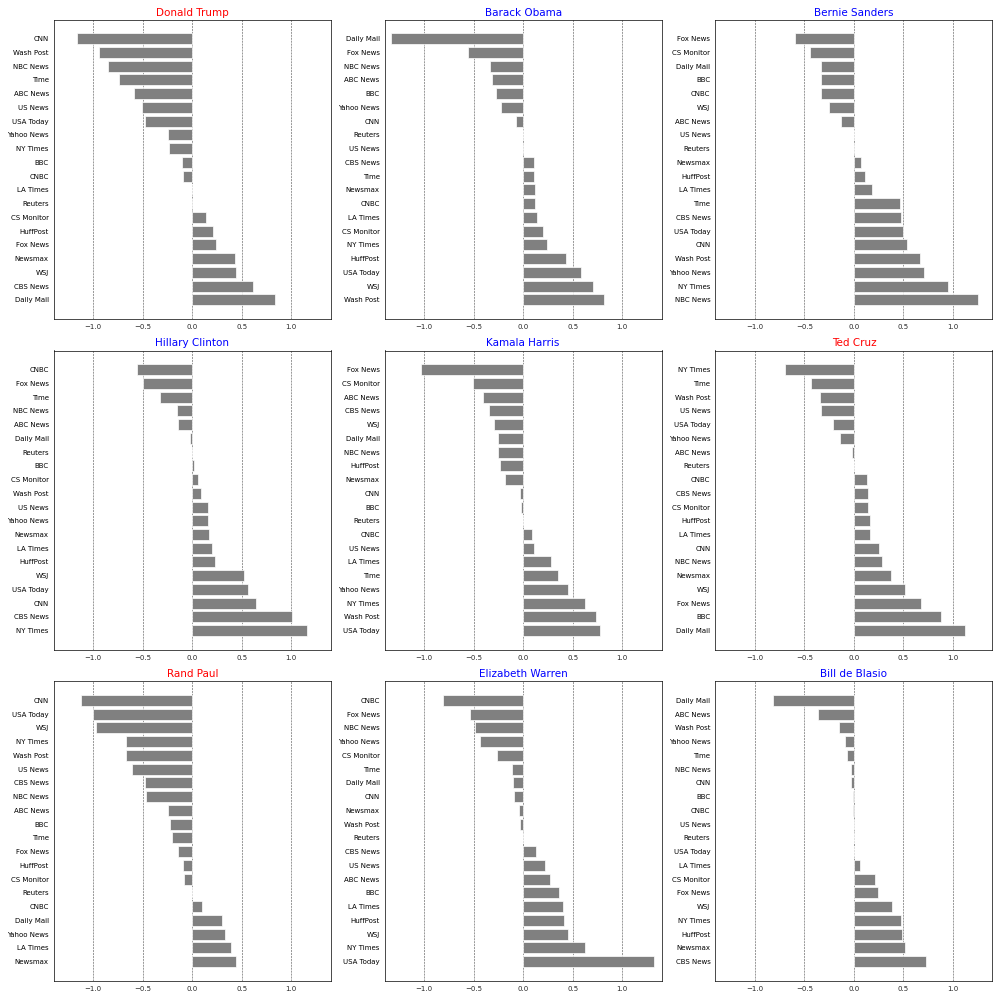}
    \caption{\small Bar plot showcasing the polarization across various Politicians, segmented by different News Outlets, Part I}
    \label{fig:enter-label}
\end{figure}

\begin{figure}[H]
    \centering
    \includegraphics[width=1\linewidth]{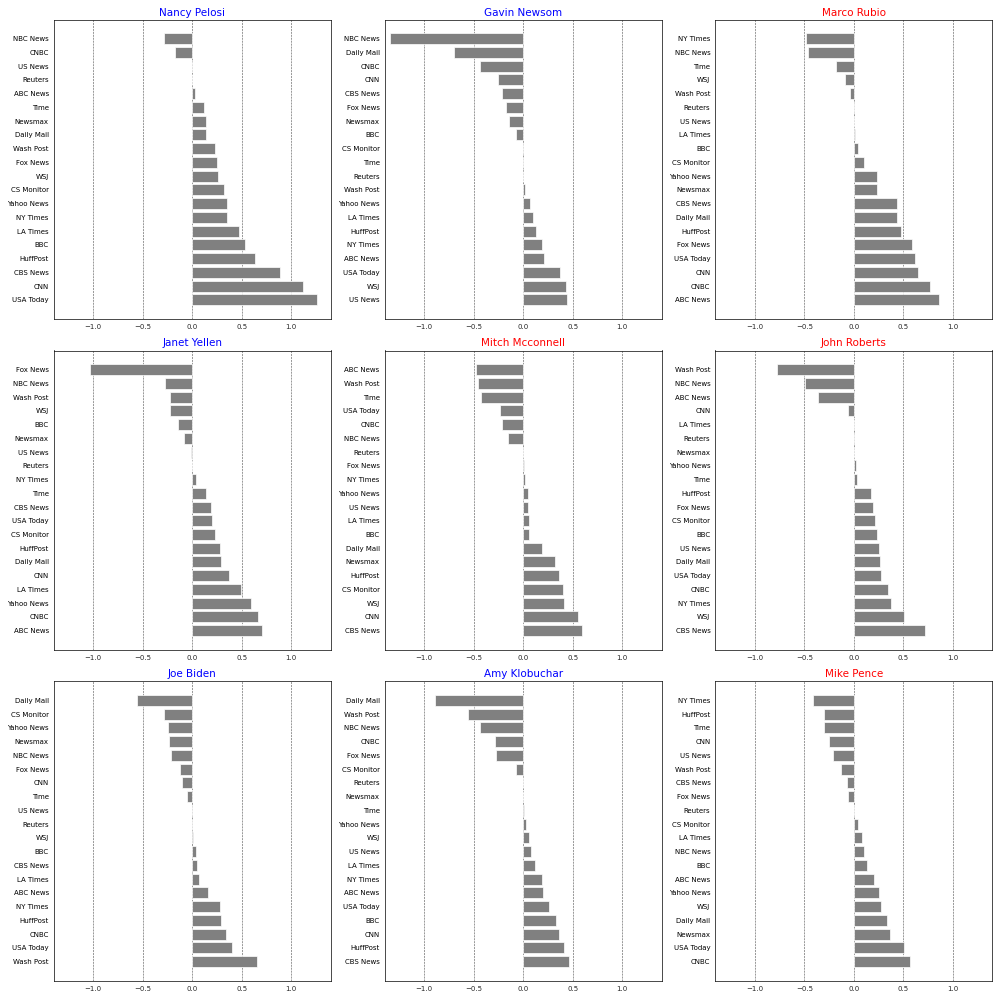}
    \caption{\small Bar plot showcasing the polarization across various Politicians, segmented by different News Outlets, Part II}
    \label{fig:enter-label}
\end{figure}

\begin{figure}[H]
    \centering
    \includegraphics[width=1\linewidth]{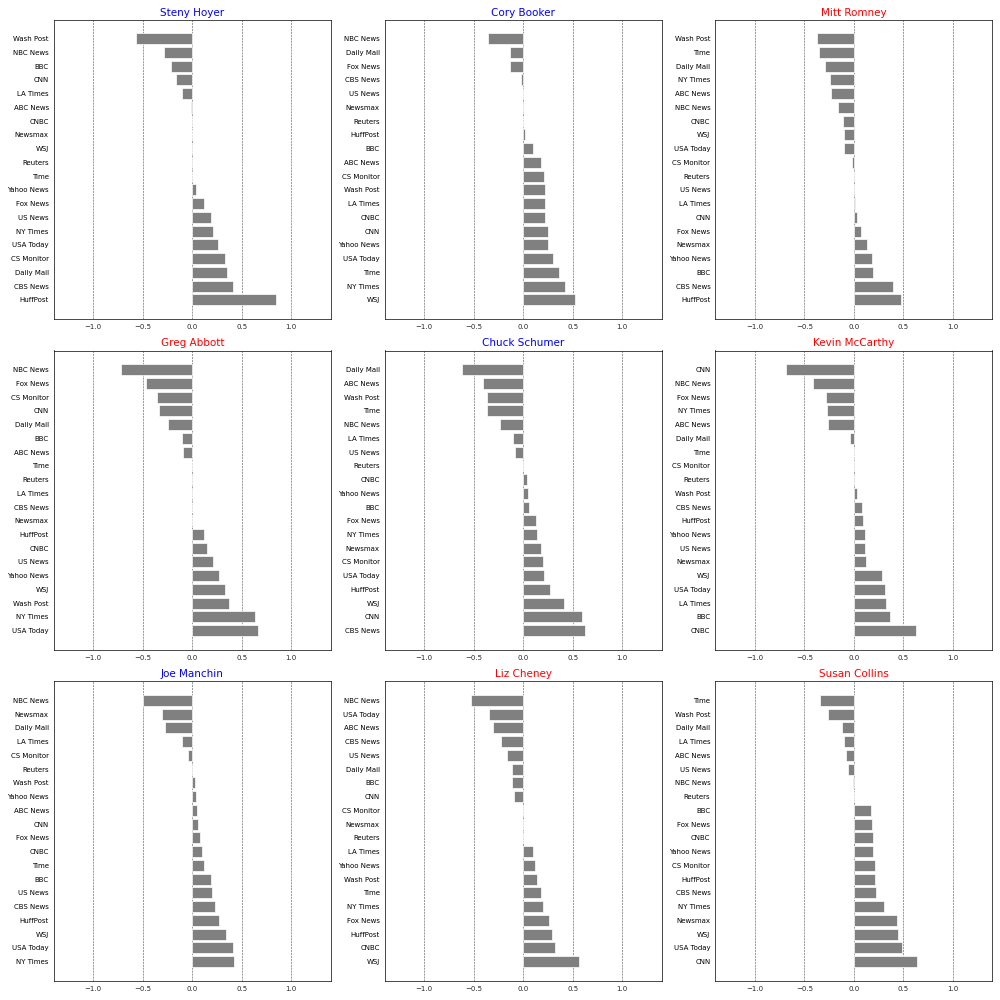}
    \caption{Bar plot showcasing the polarization across various Politicians, segmented by different News Outlets, Part III}
    \label{fig:enter-label}
\end{figure}

\subsection{Ideological Partisanship in States and the Polarization of Politicians}\label{appssec:polarized_politician}

In this section, we examine the relationship between state-level ideological partisanship and the visual polarization of politicians. Our goal is to understand whether politicians from states with strong partisan leanings experience greater or lesser media polarization compared to those from more ideologically mixed states. To measure ideological partisanship at the state level, we rely on the presidential vote margin in the 2016 election—specifically, the difference between the percentage of votes received by the politician’s party and the opposing party. This vote difference serves as a proxy for how politically secure or competitive a politician’s home state is. We then regress each politician’s OVP from Equation \eqref{equ:OVP} on this vote difference to assess whether electoral environment influences how divisive a politician’s visual representation is across media outlets. Figure \ref{fig:OVP_Vote} presents the results of this analysis, revealing a statistically significant positive relationship.

\begin{figure}[H]
    \centering
    \includegraphics[width=1\linewidth]{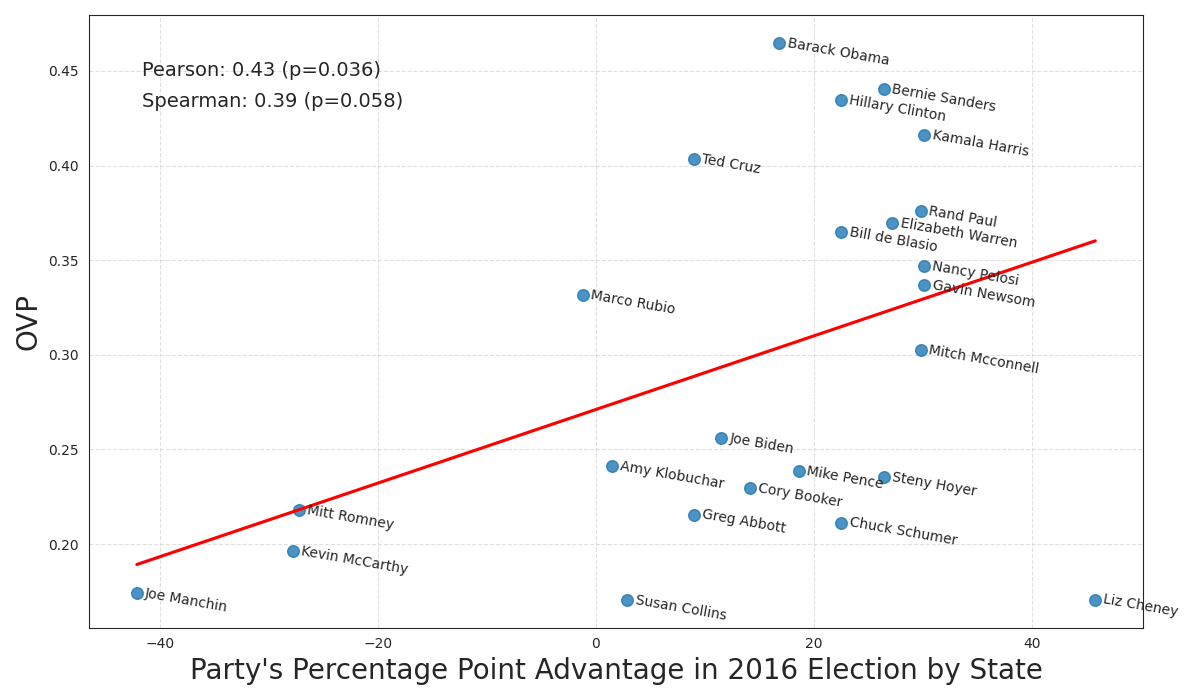}
    \caption{\small Relationship between state partisanship and politician's \textit{Overall Visual Polarization}}
    \label{fig:OVP_Vote1}
\end{figure}

This finding aligns with the idea that politicians from safe states enjoy greater ideological freedom, enabling them to take stronger partisan stances without the need to appeal to a broad electorate. Consequently, they may attract more partisan media portrayals, reinforcing their image as highly divisive figures. For instance, Bernie Sanders (Vermont) and Ted Cruz (Texas), both from states with strong partisan identities, exhibit high OVP values, likely reflecting their strong ideological positions and the way they are framed by different media outlets. In contrast, Joe Manchin (West Virginia) and Susan Collins (Maine), who represent states where their party is in the minority, display lower OVP values, suggesting that politicians from ideologically mixed states receive more moderate media portrayals.

Overall, these results suggest that state-level ideological partisanship influences how politicians are visually portrayed in the media, with those from politically homogeneous states being more polarizing figures.

\end{appendices}

\end{document}